\documentclass[nonatbib]{elsarticle}
\usepackage[OT4]{fontenc}
\usepackage[square,numbers]{natbib}
\usepackage{graphicx}
\usepackage{verbatim}
\usepackage{todonotes}
\usepackage{amsmath}
\usepackage{amssymb}
\usepackage[ruled]{algorithm}
\usepackage{algorithmic}
\usepackage{amsthm}
 \usepackage{lineno}
\usepackage[utf8]{inputenc}
  \usetikzlibrary{snakes}
\usepackage{subcaption}
\usepackage{appendix}
  \usepackage{mathtools}
  \usepackage{multirow}

  \usepackage{thmtools}
  \usepackage{
    hyperref,        cleveref,      }
  \declaretheorem[name=Theorem,
  refname={Theorem,Theorems},
Refname={Theorem,Theorems}]{theorem}
  \declaretheorem[name=Lemma,
  refname={Lemma,Lemmas},
Refname={Lemma,Lemmas}]{lemma}
  \declaretheorem[name=Auxiliary Lemma,
  refname={Aux. Lemma,Aux. Lemmas},
Refname={Aux. Lemma,Aux. Lemmas}]{auxlemma}
  \declaretheorem[name=Example,
  refname={Example,Examples},
Refname={Example,Examples}]{example}

  \declaretheorem[name=Definition,
  refname={Definition,Definitions},
Refname={Definition,Definitions}]{definition}

\newcommand{\tabitem}{~~\llap{\textbullet\,}~~}

\newcommand{\vvarprod}{\ensuremath{\bigtimes}}

\definecolor{hanpurple}{rgb}{0.32, 0.09, 0.98}

\renewcommand{\algorithmiccomment}[1]{\textit{/* #1 */}}

\newcommand{\graph}[2]{\ensuremath{\mathcal{G}_{#2}({#1})}}
\newcommand{\context}[1]{\ensuremath{#1}}

\newcommand{\secuencia}{\ensuremath{S_{ij}}}
\newcommand{\z}{\ensuremath{x_Z}}
\newcommand{\x}{\ensuremath{\mathbf{x}}}
\newcommand{\uu}{\ensuremath{x_U}}

\newcommand{\PV}{\ensuremath{ij}}
\newcommand{\p}[1]{\ensuremath{{#1}^{\PV}}}
\newcommand{\DPij}[1]{\ensuremath{\mathcal{D}^{ij}(#1)}}

\newcommand{\DE}{\ensuremath{D^E_{fg}}}
\newcommand{\DEf}{\ensuremath{D^E_{f}}}

\newcommand{\DB}{\ensuremath{D_{fg}}}

\newcommand{\Xijof}[1]{\ensuremath{\mathcal{X}^{\PV}({#1})}}
\newcommand{\bigXijof}[1]{\ensuremath{\mathcal{X}^{\PV}\left({#1}\right)}}
\newcommand{\Xzoof}[1]{\ensuremath{\mathcal{X}^{01}({#1})}}
\newcommand{\Xij}{\ensuremath{\mathcal{X}^{\PV}}}
\newcommand{\Xkl}{\ensuremath{\mathcal{X}^{kl}}}
\newcommand{\val}[1]{\ensuremath{val(#1)}}

\newcommand{\scope}[1]{\ensuremath{S_{#1}}}

\newcommand{\MP}{\ensuremath{F}}
\newcommand{\MQ}{\ensuremath{G}}

\newcommand{\MA}{A}
\newcommand{\MB}{B}
\newcommand{\MC}{C}

\newcommand{\modelB}{{M}_2}
\newcommand{\modelA}{{M}_1}

\newcommand{\crossproductDE}[1]{\ensuremath{\Delta(#1)}}

\newcommand{\distcomplete}[2]{\ensuremath{d_C({#1},{#2})}}
\newcommand{\dist}[2]{\ensuremath{d({#1},{#2})}}
\newcommand{\disti}[2]{\ensuremath{\tilde{d}_\Ind({#1},{#2})}}

\usepackage{enumerate}

\newcommand{\PERP}{\mbox{\ensuremath{\perp\!\!\!\!\perp}}}
\newcommand{\NPERP}{\mbox{\ensuremath{\perp\!\!\!\!\perp\!\!\!\!\!\not~~~}}}
\newcommand{\indep}[2]{\ensuremath{#1 \PERP #2}}

\newcommand{\dep}[2]{\ensuremath{#1 \NPERP #2}}
\newcommand{\ci}[3]{\ensuremath{(\indep{#1}{#2} \mid #3)}}
\newcommand{\mi}[2]{\ensuremath{(\indep{#1}{#2})}}
\newcommand{\md}[2]{\ensuremath{(\dep{#1}{#2})}}

\newcommand{\cd}[3]{\ensuremath{(\dep{#1}{#2} \mid #3)}}
\newcommand{\assertion}[3]{\ensuremath{I({#1},{#2} \mid #3)}}
\newcommand{\assertionM}[2]{\ensuremath{I({#1},{#2})}}

\newcommand{\X}{\ensuremath{\mathcal{X}}}

\newcommand{\Fone}{\ensuremath{F^{1}}}
\newcommand{\Ftwo}{\ensuremath{F^{2}}}
\newcommand{\Gone}{\ensuremath{G^{1}(f)}}
\newcommand{\Gtwo}{\ensuremath{G^{2}(f)}}
\newcommand{\Gcomp}{\ensuremath{\Gone \cup \Gtwo}}

\newcommand{\tri}{\ensuremath{ \assertion{X_i}{X_j}{x_U,X_W} }}
\newcommand{\trip}{\ensuremath{ \assertion{X_i}{X_j}{x_Z} }}

\newcommand{\Tripletes}{\ensuremath{\mathcal{T}}}

\newcommand{\TripletesP}{\ensuremath{\mathcal{T}_{FC}}}
\newcommand{\IP}{\TripletesP}

\newcommand{\Ind}{\ensuremath{t}}

\newcommand{\DP}[1]{\ensuremath{\mathcal{D}(#1)}}
\newcommand{\DC}[1]{\ensuremath{\mathcal{D}_C(#1)}}

\newcommand{\TripletesUGM}{\ensuremath{\mathcal{T}^{UGM}}}

\newcommand{\cupij}{\cup^{ij}}
\newcommand{\bigcupij}{\bigcup^{ij}}

\DeclareRobustCommand{\thicklines}[1]{\tikz[baseline]{\draw[thick,double] (0,.5ex)--++(.5,0) ;}}
\DeclareRobustCommand{\wavylines}[1]{\tikz[baseline]{\draw[snake,decoration={snake,amplitude=.4mm,segment length=2mm,
post length=0mm,pre length=0mm}] (0,.5ex)--++(.5,0) ;}}

\usepackage{todonotes}

\begin{document}

%\clearpage
%\pagestyle{plain}
%\setcounter{page}{1}
\begin{frontmatter}

\title{Efficient comparison of independence structures of log-linear models}

%\author{  \name Jan Strappa\thanks{Corresponding author.} \and Facundo Bromberg \\
%  \addr  Laboratorio DHARMa \\
%  Dept.\ de Sistemas de Informaci\'on\\
%  Universidad Tecnol\'ogica Nacional\\
%  Rodríguez 273, Ciudad de Mendoza, \\
%    Mendoza, Argentina\\
%  \email \{jstrappa,fbromberg\}@frm.utn.edu.ar \\
%\\
%  \addr Consejo Nacional de Investigaciones \\
%  Científicas y Técnicas (CONICET), Argentina }
%\maketitle
 
\author[1,2]{Jan Strappa\corref{cor1}}
%\fnref{fn1,fn2}}
\ead{jstrappa@frm.utn.edu.ar}

\author[3,2]{Facundo Bromberg}
\ead{fbromberg@frm.utn.edu.ar}

\cortext[cor1]{Corresponding author}

\affiliation[1]{organization={Laboratorio de Investigación en Cómputo Paralelo/Distribuido (LICPaD) -- Universidad Tecnol\'ogica Nacional, Facultad Regional Mendoza},
addressline={Rodríguez 273},
postcode={CP 5500},%M5502AJE
city={Ciudad de Mendoza, Mendoza},
country={Argentina}}

\affiliation[2]{organization={Consejo Nacional de Investigaciones Científicas y Técnicas (CONICET)},
addressline={Av. Ruiz Leal s/n Parque General San Martín},
postcode={CP 5500},
city={Ciudad de Mendoza, Mendoza},
country={Argentina}}

\affiliation[3]{organization={Laboratorio DHARMa -- Universidad Tecnol\'ogica Nacional, Facultad Regional Mendoza},
addressline={Rodríguez 273},
postcode={CP 5500},%M5502AJE
city={Ciudad de Mendoza, Mendoza},
country={Argentina}}

%\begin{abstract}
%Log-linear models are a family of probability distributions which capture a variety of relationships between variables, including context-specific independencies.
%  There are a number of approaches for automatic learning of their independence structures from data, although to date, no efficient method exists for evaluating these approaches directly in terms of their independence structure.
%  The only known methods evaluate them indirectly through the complete density,
%  which requires further learning of the numerical parameters for the structures produced by each approach.
%  These indirect methods introduce potential distortions when used for the comparison of the structures.    This work addresses this issue by presenting a measure for the direct and efficient comparison of the independence
%  structures of log-linear models, inspired by the efficient Hamming distance comparison method used in undirected graphical models.
%  The measure presented not only can be efficiently computed in terms of the number of variables of the domain, but is also proven to be a metric.
%Efficiency with respect to the number of features is not guaranteed for a large number of features in the models and will be the subject of future work.
%\end{abstract}

\begin{abstract}
Log-linear models are a family of probability distributions which capture relationships between variables. 
They have been proven useful in a wide variety of fields such as epidemiology, economics and sociology.
The interest in using these models is that they are able to capture context-specific independencies, relationships that provide richer structure to the model.
Many approaches exist for automatic learning of the independence structure of log-linear models from data.
The methods for evaluating these approaches, however, are limited, and are mostly based on indirect measures of the complete density of the probability distribution.
Such computation requires additional learning of the numerical parameters of the distribution, which introduces distortions when used for comparing structures.
This work addresses this issue by presenting the first measure for the direct and efficient comparison of independence structures of log-linear models.
Our method relies only on the independence structure of the models, which is useful when the interest lies in obtaining knowledge from said structure, or when comparing the performance of structure learning algorithms, among other possible uses. 
We present proof that the measure is a metric, and a method for its computation that is efficient in the number of variables of the domain.
%However, efficiency in the number of features in the models is not guaranteed and will be the subject of future work.
\end{abstract}


%\maketitle
%Research highlights
%\begin{highlights}
%\item
%\item
%\end{highlights}

\end{frontmatter}

\section{Introduction and Motivation}

% TODO
%%% citations to add
% epidem Shah_2018,Yuan_2021,Panagiotakos_2021
% econ   LUNDTOFTE20134256,Zio_o_2022 
% socio  Raftery_2001,Schwartz_2016,Bucca_2019 

This paper presents a metric for efficiently comparing the independence structures of two \emph{log-linear models}, a well-known representation of probability distributions over the assignments of discrete domains \cite{christensen2006log,AGRESTI02,haberman1973log,koller09,LAURITZEN96}.
These models are widely used for representing real-world distributions in many different disciplines, such as epidemiology \cite{Shah_2018,Yuan_2021,Panagiotakos_2021}, economics \cite{LUNDTOFTE20134256,Zio_o_2022} and sociology \cite{Raftery_2001,Schwartz_2016,Bucca_2019}.
Due to the complexity of data analysis and the rapidly growing availability of large quantities of data, interest in the automatic learning of these models from data has increased, becoming a subfield of machine learning \cite{DELLAPIETRA97,MCCALLUM03,LeealNIPS06,davis2010bottom,lowd14a,van2012markov}.
In this area, the problem of automatically learning a structure from data is often divided into structure learning and parameter learning.
The former involves finding a set of dependencies that best represent the relationships among the variables in the domain, while the latter consists in estimating the parameters that quantify the structure.
 The contribution of this work is inspired by the structure learning problem, which, in general, has two main objectives.
 Structures are either used as an intermediate step toward the construction of complete density models for inference tasks, such as the estimation of marginal and conditional probabilities (which is known as \emph{density estimation}) \cite{lowd14a,van2012markov,davis2010bottom}; or they are used as an interpretable model that shows the most significant interactions of a domain (known as \emph{knowledge discovery}) \cite{LeealNIPS06,haaren2013,claeskens2015,nyman2014context,pensar2017marginal}.
To date, there are no direct methods for the evaluation of the quality of structures of log-linear models, despite their importance for obtaining accurate predictions and, perhaps more importantly, their crucial role when trying to understand patterns present in data and to draw reliable conclusions from those patterns.
Our proposal offers a metric that can be used for several purposes, including the assessment of learning algorithms, knowledge discovery, and the design of new algorithms.
To the best of our knowledge, our proposal is the first of its kind, with no other structural distance metric in the literature for evaluating log-linear models in terms of their context-specific independencies.
%The most significant might be to allow for the comparison and selection of learning methods based on the quality of the structures, which can lead to improvements in both density estimation tasks and knowledge discovery.
%Other uses could involve applying this metric in order to produce a diverse sample of synthetic structures for use in experiments, or as a search algorithm mechanism that may help find nearest neighbors for guiding the search, or keep diversity in the search space ---by informing the search of the similarity among candidate structures--- to encourage exploration and avoid local optima.

In this work we make references to and draw inspiration from Markov networks, an interesting subset of the family of log-linear models, whose independence structure is an undirected graph, with the nodes representing the random variables of the domain, and the edges encoding direct probabilistic influence between the variables.
There are several methods for learning Markov network structures from data \cite{brombergmargaritis09b,schluter2014ibmap,pensar2017marginal,schluter2018blankets}. In undirected graphical models (UGMs), the absence of an edge indicates that the dependence could be mediated by some other subset of variables, corresponding to conditional independence between these variables.
The use of graphs as representation, however, has an important disadvantage: since it simply uses the basic concept of conditional and marginal independence, this representation may hide the occurrence of fine-grain structure such as \emph{context-specific independencies} \cite{boutilier1996context,hojsgaard2004statistical,koller09}, which are independencies that hold only in a subspace of the configurations of the conditioning set.
Log-linear models are more flexible than graphical models, since they are capable of encoding not only conditional independencies, but also context-specific independencies.
The log-linear representation of the structure is defined as a set of feature functions, each consisting of an assignment to some subset of the variables in the domain.
Given a set of features, the joint probability distribution is completely specified by the feature weights, one real number per feature, which are the numerical parameters of the log-linear model.

For the structure learning problem, there has been a surge of interest towards methods that construct a log-linear model by selecting features from a dataset, usually by performing a local search that incrementally adds or deletes features \cite{DELLAPIETRA97,LeealNIPS06,davis2010bottom,van2012markov,lowd2013learning,lowd14a}.
This approach defines structure learning as a feature selection problem, where the features represent dependencies between subsets of random variables.
All these contributions have only been assessed for the density estimation goal of learning, that is, the selection of models for inference tasks, by measuring the quality of learned models in terms of prediction performance.
The reason for this is that, historically, this family of models has rarely been used with the goal of knowledge discovery in mind, given that the interpretability of conventional log-linear models is burdensome, and reading independencies from them is not trivial.
However, this has been changing lately, mainly due to the fact that the theory of log-linear models for contingency tables~\cite{darroch1980markov} has been augmented by the introduction of a variety of representations that generalize graph-based undirected graphical models: Context-Specific Interaction models \cite{eriksen1999context,hojsgaard2004statistical}, Stratified-graphical models \cite{nyman2014stratified,nyman2014context}, and Canonical models \cite{ederaSchluterBromberg14}. % \todo{Mention any new representation from latest works}
%TODO relacionar
Such contributions have shifted attention towards methods that focus on learning only the structure of these models, while the parameter learning step may be performed afterwards with existing techniques, or not performed at all, depending on the use case.
Although the focus of this work is on log-linear models for categorical data, it is interesting to mention other models that are able to represent context-specific interactions, for example, in domains based on ordinal data, such as Hierarchical Marginal Models \cite{nicolussiContextspecificIndependenciesHierarchical2019}; or causal models, such as CPT-tress \cite{boutilier1996context} and Labeled Directed Acyclic Graphs (LDAGs) \cite{PENSAR201691,coranderLogicalApproachContextspecific2019}.
In these representations, the semantics for reading conditional independence from graphs serves as inspiration for graphically expressing context-specific independencies, with the aim of representing a much wider class of models while maintaining their interpretability.

Due to all these relatively recent developments that focus on the structures of the models, there is an incresing need for tools that assess the quality of structure learning methods.
At present, this assessment is often carried out with the \emph{Kullback--Leibler divergence (KL-divergence)} \cite{kullback1951information,cover2012elements}.
In this context, the KL-divergence measures the similarity of the complete distributions encoded by each structure together with their parameters~\cite{ederaSchluterBromberg14,edera2014grow}.
%In order to compare structure learning algorithms, one must perform an additional step of obtaining the parameters for the structures, usually with approximate methods.
%One of its advantages is that it satisfies the properties of nonnegativity and discrimination.
%It has some advantages, such as satisfying the properties of nonnegativity and discrimination, i.e., $D_{KL}(p||q) > 0$ and $D_{KL}(p||q) = 0 \iff p = q$ for any two probability distributions $p$ and $q$.
Initially, divergences were the only means of computing statistical distance.
For model comparisons they have been used mainly in the process of disproving the null-hypothesis, in which one model differs from the other when the divergence equals zero. 
However, they present some limitations when the null-hypothesis holds, i.e., when the models are different, as they provide no sense of scale for their difference.
%Initially, divergences were the only means of computing statistical distance.
%They were formally introduced as a statistical distance by \cite{BHATTA43}, and popularized by the Kullback–Leibler divergence in \cite{kullback1951information}.
%Since then, they are still in use, proven useful for statistical comparisons of probabilistic models \cite{gardner2018,venturini2015statistical}.
%For model comparisons they have been used mainly in the process of disproving the null-hypothesis that one model differs from the other when the divergence equals zero.
%However, they present some limitations when the null-hypothesis holds, i.e., when the models are different, as they provide no sense of scale for their difference.
%TODO mover parte de lo anterior a sección KL
In that case, the statistical community uses \emph{distance functions} or  \emph{measures}\cite{DODGE2006}, a.k.a. \emph{metrics}, a notion stronger than divergence that satisfies not only nonnegativity and discrimination, but also \emph{symmetry} and the \emph{triangle inequality}.
Since the KL-divergence is a measure of divergence between distributions, it is an indirect procedure that requires learning the parameters in addition to the structure; therefore, the quality of structures is analyzed by evaluating the quality of the resulting full distribution.
This introduces some shortcomings.
The first disadvantage is that false positives and false negatives have a different impact on the quality of the distribution.
False negatives cannot be mitigated by the numerical parameters, because they add incorrect independence assumptions to the distribution that can invalidate statistical inference, leading to faulty conclusions.
Instead, false positives may be mitigated when learning the parameters, by setting some weights to zero to encode the independencies that were not found by structure learning.
Thus, KL is unable to accurately measure false positives in the structure, since these can be obscured by the parameters.
As a second disadvantage, it is important to note that the parameter learning process is sensitive to data scarceness; therefore, the KL measure might not be accurate when data is insufficient.
Both shortcomings are illustrated by a toy example in Section~\ref{app:kl}.
Since our method is computed directly over the structures, it addresses both problems: it allows for a separate analysis of false positives and false negatives, and is not influenced by data scarceness.

When learning structures for high-dimensional domains, the computation of
the KL-divergence becomes infeasible and some works report instead the Conditional Marginal Log-likelihood (CMLL)  \cite{LeealNIPS06,davis2010bottom,van2012markov,lowd2013learning,lowd14a}, which uses marginal probabilities in order to avoid the computation of the partition function that normalizes the distribution.
Although useful in practice, CMLL is an approximate method, and it also presents the first and second shortcomings mentioned above, because it also requires the task of learning the numerical parameters of the structure.
Lastly, as a means of understanding structural qualities without taking into account the parameters, a few works have used the number of features and average feature length \cite{van2012markov,ederaSchluterBromberg14,edera2014grow}.
Both are aggregated and indirect indicators and as such not very informative; moreover, they do not allow for trustworthy comparison between different structures.
A summary of the characteristics of all these methods is provided in Table~\ref{tab:measures}.

\begin{table}
  \footnotesize
  \centering
  \begin{tabular}{|p{2.4cm}|p{4cm}|p{4cm}|}
\hline
Measure & Advantages & Disadvantages \\\hline
KL-divergence
&
{\raggedright
\tabitem ease of implementation

\tabitem satisfies nonnegativity and discrimination
}
&
{\raggedright
\tabitem not a metric (symmetry, triangle inequality)

\tabitem unable to measure FPs in the structure

\tabitem sensitive to data scarceness

\tabitem infeasible in high dimensions
}
\\\hline
CMLL
&
\tabitem scalability
&
{\raggedright
\tabitem unable to measure FPs in the structure

\tabitem sensitive to data scarceness

\tabitem uses an approximation
}
\\\hline
Number of features
&
{\raggedright
\tabitem parameter-independent

\tabitem correlates to \# of dependencies
}
&
{\raggedright
\tabitem indirect measure
}
\\\hline
Average feature length
&
{\raggedright
\tabitem parameter-independent

\tabitem provides an idea of the density of the structure
}
&

\tabitem indirect measure
\\\hline
  \end{tabular}
  \caption{Characteristics of measures for the comparison of log-linear models}
  \label{tab:measures}
\end{table}

Our method works by measuring the number of structural differences that appear between two log-linear models, efficiently producing a confusion matrix that counts the true positives, false positives, true negatives and false negatives that appear in the second model, relative to the first one.
It is inspired on the structure comparison method of Markov networks: the Hamming distance of their graphs~\cite{brombergmargaritis09b,schluter2014ibmap,pensar2017marginal,schluter2018blankets}, i.e., the sum of false positives plus false negatives in terms of edges.
Although there is no unequivocal graph representation for log-linear models and therefore no straightforward generalization of the Hamming distance for this case, we will show that both measures take advantage of different properties of their respective independence representations in order to reduce the complexity of the comparison.
As will be discussed in detail later in Section~\ref{sec:structureComparisonLL}, a straightforward counting of dependencies and independencies for producing the confusion matrix for log-linear models presents an exponential computational cost due to the much larger space of possible structures when compared to graphical models.
The main advantage of our method is that it can efficiently compute the counts in the confusion matrix with respect to the number of variables.
Nevertheless, the efficiency w.r.t. the number of features is not guaranteed for a large number of features in the models, and it will be the subject of future work.

%Our motivation is to provide a computationally efficient technique for measuring the similarity of two log-linear models based solely on the set of independencies that they encode, that is, their structure.
The contribution of this work has several potential applications. 
Most notably, this technique can be used to improve the quality of structures learned from data by providing the means for comparing different structure learning techniques over synthetic data produced by a known underlying distribution.
This is achieved by comparing the quality of the structures obtained from data by any given algorithm, measured as the distance from the learned structure to the structure of the underlying distribution.
Better structures have important advantages as they can improve both the quality and efficiency of parameter learning, leading to better density models for inference tasks.
In addition, although the complexity of the structures of log-linear models has been a limiting factor in the past, recent contributions have allowed for knowledge discovery tasks, by improving the interpretability of the models with some of the representations mentioned above.
Therefore, our method can also contribute to this goal, by providing a tool that can help select among different algorithms, or algorithm configurations, in order to find the best learning strategies specifically based on the quality of structures, which is not achievable with the state-of-the-art methods.
While these are the main benefits we identify for our contribution, there might be many other potential use cases.
For instance, another aspect of experimentation that could be explored is the possibility of generating random synthetic structures from the space of possible structures, using their structural distance as a guarantee that the sample is not biased.
At present, these synthetic experiments are usually comprised of a small number of handmade structures, designed in order to highlight the advantages of a particular method.
In addition, we see an interesting possibility of application in the incorporation of this measure as a means of assessing similarities among structures in search algorithms, where each solution in the search space would be equivalent to a complete log-linear model structure.
A similar example in the space of features is in \cite{davis2010bottom}, where a measure of similarity among features is used to generate the nearest candidate features w.r.t. a given feature from the current structure, in order to guide the search.
Lastly, another use case can be found in the design of new log-linear structure learning algorithms. If an algorithm poses the structure learning problem as a search in the space of possible structures (feature sets), then our metric could be used in different ways to evaluate similarity between structures, e.g. to find similar structures in a proximity search, or to maintain diversity by encouraging the generation of structures that differ from each other.
Lastly, it should be noted that our method is efficient in two ways: on the one hand, it is proven to be efficient in the number of variables of the models, when compared to a brute force approach; on the other hand, it avoids the complexity of parameter learning when used as a substitute for methods that compare the complete distributions. 
%TODO cerrar párrafo

This paper is organized as follows: In the next section we present the notation and main concepts required for our analysis.
Section~\ref{sec:structureComparisonLL} establishes and justifies the basis for our approach, by introducing a brute-force method for the structural comparison of log-linear models, highlighting its sources of exponentiality, and providing a roadmap for tackling them.
Section~\ref{sec:approach} presents our approach for the efficient computation of the confusion matrix, together with a proof of correctness.
Section~\ref{sec:distance} introduces a distance measure directly computable from the confusion matrix, and provides proof it is indeed a distance metric by proving all four properties: \emph{nonnegativity}, \emph{discrimination}, \emph{symmetry}, and \emph{triangle inequality}.
Section~\ref{sec:summary} summarizes the main steps for the development of our contribution and for its computation. 
Section~\ref{app:kl} describes the example comparison of our metric against KL-divergence, the most common measure used in recent works concerning log-linear models structure learning.
Section~\ref{sec:discussion} presents some conclusions, open questions and some ideas to extend this work.
To simplify the presentation, the proofs of some lemmas have been removed from the main text and are presented in detail in~\ref{app:lemmas}.
Similarly, all auxiliary lemmas are described and proven in~\ref{app:auxiliary}.

\section{Background knowledge and Notation}\label{sec:notation}

This section introduces key concepts of probabilistic models and the notation used to denote them throughout the manuscript.
The first two parts, Sections~\ref{sec:randomVariables} and \ref{sec:loglinearmodels}, present basic definitions concerning random variables and log-linear models, together with some notations specific to this work.
The remaining three sections are more involved and present crucial aspects of our contribution.
Firstly, in Section~\ref{sec:independence}, we define different kinds of probabilistic independencies and reproduce important equivalences. 
Secondly, Section~\ref{sec:dependencyModels} explains the structure representation on which our contribution is conceptually based.
Finally, Section~\ref{sec:ugm} provides an overview of an analogous strategy used for the comparison of Undirected Graphical Models (UGMs), that serves as a partial inspiration for our method.

\subsection{Random Variables} \label{sec:randomVariables}

Let V be a finite set of indices for a
set of discrete random variables $X_V$. Lowercase subscripts denote single indices
(e.g., $X_i,X_j \in X_V$ where $i,j \in V$), while uppercase subscripts denote
subsets of indices (e.g., $X_A \subseteq X_V$ where $A \subseteq V$).
A variable $X_k$ can take a value from a finite set of configurations, denoted
by \val{X_k}. For example, for a binary variable $X_0$, $\val{X_0}=\{0,1\}$.
An arbitrary configuration in $\val{X_k}$ will be denoted in lowercase, e.g., $x_k$.

A set of variables $X_A$, $A\subseteq V$,  can take values from the cross-product of $\val{X_k}$, over all $k\in A$;
with individual configurations denoted by $x_A$.
The set of variables assigned in some configuration $x$ is called the
\emph{scope} of $x$, denoted  $S_x$; e.g., for $x=x_A$, $S_{x}=X_A$.

The space of all configurations for $X_V$ is denoted as $\X$.
A \emph{canonical context} $\context{x}$ is a complete assignment in a domain, i.e.,
$\context{x}\in \mathcal{X}$, and $\scope{\context{x}}=X_V$.
Even though canonical contexts are not used in this work in this manner,
we make extensive use of a similar concept: \emph{fully-contextualized (FC)
contexts} (or simply \emph{contexts}  when the meaning is clear), which are
configurations defined for a given pair of indices $(i,j)$ and consist of
assignments to all variables in $X_V \setminus \{X_i,X_j\}$. The set of all FC
contexts for one $(i,j)$ is denoted as \Xij. Its name stems from the fact that it is used in the sense of a completely contextualized conditioning set.

\subsection{Log-linear models}\label{sec:loglinearmodels}

 For a distribution to be considered an element of the log-linear family it must be structured through a set of feature functions $F = \{ f_i (X_{D_i} ) \}$,
 and specify a numerical value $\theta_i$ for each assignment $x_{D_i}$ of the subset of variables
 $X_{D_i}$, where $D_i \subseteq V$,
  resulting in the following generic functional form of distributions in the log-linear family:

 \begin{equation}\label{eq:jointloglinear}
 p (x) =  \frac{1}{Z(\theta)}\, exp \left( \sum_{f_i \in F} \theta_i f_i (x_{D_i})
 \right) ,
 \end{equation}
 where $Z(\theta)$  is the partition function that
 ensures  that  the  distribution  is normalized (i.e., all entries sum to $1$).

  In what follows, features are denoted by lowercase letters, such as $f$, $g$ or $h$. The value that variable $X_k \in X_V$ takes in feature $f$ is denoted by $X_k(f)$. For example, if $f=<X_0=1,X_2=0,X_3=1>$, then $X_2(f)=0$.
  Also, overloading the naming used for variable configurations, the set of variables that are assigned in a feature $f$ is called the $scope$ of $f$, and it is denoted by $\scope{f}$. For the last example, $\scope{f}=\{X_0,X_2,X_3\}$.

  Finally, we introduce the notation $\p{h}$ to denote a feature composed of the
  same assignments as feature $h$, except for those of $X_i$ and $X_j$, when
  $i,j$ is a pair of distinct indices such that $X_i,X_j \in \scope{h}$. For
  instance, if $V=\{0,\dots,5\}$, and if

  $$h = < X_0=2, X_1=1, X_2=1, X_5=0 >,$$

  then the same feature without the pair of assignments to $(X_0,X_2)$ will be

  $$h^{02} = < X_1=1, X_5=0 >.$$

\subsection{Independence} \label{sec:independence}
We use the notation $\ci{X_A}{X_B}{X_C}_p$ to denote that in the distribution $p$,
variables in set $X_A$ are (jointly) independent of those in $X_B$,
  conditioned on the values of the variables in $X_C$, for
 disjoint sets of indices $A$, $B$, and $C$.
This occurs if and only if the conditional distribution of $X_A$ conditioned on the values of variable $X_B$ and $X_C$ only depends on the values of $X_C$. Formally,
\begin{equation*}
        \ci{X_A}{X_B}{X_C}_p \iff p(x_A|x_B,x_C) = p(x_A|x_C),
\end{equation*}
for all $x_A \in val(X_A)$, $x_B \in val(X_B)$ and
$x_C \in val(X_C)$.
  The negation is $\cd{X_A}{X_B}{X_C}_p$, which denotes conditional dependence.
 $\assertion{X_A}{X_B}{X_C}_p$ denotes a query of conditional independence,
i.e., a question of whether the
independence \ci{X_A}{X_B}{X_C} holds or not; symbolically:
\begin{equation*}
        \assertion{X_A}{X_B}{X_C}_p \text{ is true } \iff \ci{X_A}{X_B}{X_C}_p.
\end{equation*}

A context-specific independence~\cite{boutilier1996context,hojsgaard2004statistical,koller09} between variables $X_A$ and $X_B$ given variables
$X_C$ and a set of configurations (context) $X_D=x_D$, where $D \cap A \cap B \cap C = \emptyset$, is defined as     \begin{equation*}
  \ci{X_A}{X_B}{X_C,x_D}_p \iff p(x_A|x_B,x_C,x_D) = p(x_A|x_C,x_D),
    \end{equation*}
    for all assignments $x_A \in val(X_A)$, $x_B \in val(X_B)$ and
$x_C \in val(X_C)$, whenever $p(X_B,X_C,x_D)>0$.
A context-specific dependence is denoted by $\cd{X_A}{X_B}{X_C,x_D}_p$;
and $\assertion{X_A}{X_B}{X_C,x_D}_p$ is a context-specific independence query.

From the above, it is easy to prove the following
    equivalence of context-specific independencies:

    \begin{equation}
      \ci{X_i}{X_j}{\uu,X_W} \equiv \forall x_W \in \val{X_W},
      \ci{X_i}{X_j}{\uu,x_W},
      \label{eq:implicationindependence}
    \end{equation}

    or, equivalently,

    \begin{equation}
      \cd{X_i}{X_j}{\uu,X_W} \equiv \exists x_W \in \val{X_W},
      \cd{X_i}{X_j}{\uu,x_W},
      \label{eq:implicationdependence}
    \end{equation}

    for all $X_i\neq X_j,\, U \cap W = \emptyset,\, U \cup W \subseteq
    {V}\setminus \{i,j\},\, \uu \in val(X_U)$.

One key result of probabilistic models consists of the separation of the independence semantics of the distribution into an explicit structure.
Interestingly, for log-linear distributions, this structure is completely encoded by its set of features $F$. 
In other words, the set of features $F$ is sufficient for determining dependence or independence, which is formalized by replacing $p$ as the subscript in the notation of independencies, e.g., $\ci{X_A}{X_B}{X_C,x_D}_F$, and dependencies, e.g.,  $\ci{X_A}{X_B}{X_C,x_D}_F$.

We shall start with some intuitions, to then proceed with the formalization of these concepts. For that, we first note that the numerical parameters $\theta_i$ in the logarithmic representation of Eq.~\ref{eq:jointloglinear} can take any real value.
Non-null terms, i.e., $\theta_i \neq 0$, indicate the presence of probabilistic interactions among the variables that appear together in the scope of a feature.
In contrast, when $\theta_i=0$ for some $i$, the corresponding feature ``disappears" from the model and, as a consequence, the interactions between the variables in its scope also vanish.
Thus, the notion of independence is related to setting certain parameters to $0$.
The set of all features $F$ in a log-linear model, allows for any marginal, conditional or context-specific independence query to be verified.

First, we will formalize this idea for the (strictly) context-specific case, and show how the other types of (in)dependencies can be deduced from it.

Given a context $x_U$, a (strictly) context-specific independence of the form $\ci{X_i}{X_j}{x_U}_F$ is verified, firstly, by considering all features in $F$ that are ``compatible'' with $x_U$, i.e.,

\begin{equation} \label{eq:FIndep}
F' = \left\{f\in F \,\middle\vert \, \forall u \in U, X_u \in S_f \implies X_u(x_U) = X_u(f) \right\};
\end{equation}
that is, for every variable $X_u$ in the context $x_U$ that is also in the context of $f$ (represented by $S_f$), it is the case that their assigned values in the context $x_U$ and in feature $f$ are equal.

Secondly, for the independence to hold, it must be verified that no  feature in $F'$ contains both $X_i$ and $X_j$ in its scope:

\begin{equation} \label{eq:fIndep}
\ci{X_i}{X_j}{x_U}_F \iff \forall f \in F',\, X_i \not\in
\scope{f} \vee X_j \not\in
\scope{f}.
\end{equation}

In order to read the most general context-specific independencies of the form
$\ci{X_i}{X_j}{x_U,X_W}_F$, the following
equivalence~\cite{hojsgaard2004statistical,ederaSchluterBromberg14,PENSAR201691}
can be used:

\begin{equation}\label{eq:ci-csi}
    \ci{X_i}{X_j}{X_{U\cup W}}_F \equiv
    \forall x_U \in val(X_U),\, \ci{X_i}{X_j}{x_U,X_W}_F.
\end{equation}

From the feature set representation, we begin by verifying a set of independencies where the
whole conditioning set is contextualized ($X_W=\emptyset$) and can later
aggregate any subset of variables in $U$ for which the equivalence in
Eq.~\ref{eq:ci-csi} holds,
which allows us to obtain the truth value for the most general type of queries
$\ci{X_i}{X_j}{x_U,X_W}_F$, which also includes all queries where $x_U=\emptyset$
(conditional independencies). Therefore, any conditional (in)dependence can be
read from the set of features of a log-linear model.
Lastly, marginal independencies $\mi{X_i}{X_j}$ are simply conditional
independencies among $X_i$ and $X_j$ that
hold for all conditioning sets, and can thus be deduced by verifying a set of
conditional independencies.

\subsection{Dependency models} \label{sec:dependencyModels}

Given the set of features $F$, it is then straightforward to read
(in)dependencies of any given pair of variables conditioned on any partially
contextualized conditioning set.
Alternatively, \cite{PEARL88} proposes an explicit representation of the (in)dependencies of the distribution: the \emph{dependency model}, an exhaustive listing of  all dependencies in the distribution.
For the case of the well-known \emph{undirected graphical models} (UGMs), for instance, the dependency model reports, for each pair of variables, whether they are dependent given each possible conditioning set of variables.
To formalize it, it is convenient to first define the set of all possible triplets of variable pairs and conditioning set for some given set $V$ of random variables,

\begin{equation}
  \TripletesUGM = \left\{\ \assertion{X_i}{X_j}{X_U} \ \ \middle\vert\
  \begin{array}{ll} i \neq j \in V,\\
      U \subseteq V\setminus\{i,j\}
  \end{array}\right\},\label{eq:tripletesUGM}
\end{equation}

to then define the dependency model of some undirected model $H$ as

\begin{equation}
    \mathcal{D}^{UGM}_C(H)=\left\{\ \assertion{X_i}{X_j}{X_U} \in \TripletesUGM \ \ \middle\vert\  \cd{X_i}{X_j}{X_U}_H \ \right\}.
       \label{eq:DUGMs}
\end{equation}
We use the superscript $UGM$ to clarify that this set can encode the (in)dependencies of UGMs, which excludes context-specific structure.
The subscript $C$ (for ``complete") is used as part of our notation of the dependency model to contrast the exhaustive definitions from an approximate version defined in the following section.

Log-linear distributions are more complex in that they can encode not only marginal and conditional dependencies, but also context-specific dependence assertions. This requires a generalization from the idea of a dependency model to a \emph{context-specific dependency model}. As for the undirected case, we formalize it in two steps, starting by the set of \emph{contextualized triplets}:
  {\label{eq:alltriplets}
\begin{equation*}
    \Tripletes=\left\{\ \tri \ \ \middle\vert \ \ \begin{array}{ll} i \neq j\in V; \\
    U, W \subseteq V\setminus\{i,j\}; \\
    U \cap W = \emptyset; \\
    x_U \in val(X_U) \end{array}\right\},
\end{equation*}}

to then define the context-specific dependency model (of a log-linear structure $F$) as
 \begin{equation}
    \mathcal{D}_C(F)=\left\{\ \tri \in \Tripletes \ \ \middle\vert\ \
      \cd{X_i}{X_j}{x_U,X_W}_F \ \right\}.\label{eq:dmcomplete}
\end{equation}
where, again, the subscript $C$ denotes the completeness of this model, in contrast to its approximate version defined in Section~\ref{sec:ugm}.

Given that these dependency models are exhaustive, it is straightforward to determine the (in)dependence in model $F$ of any given  assertion $\Ind = \assertion{X_i}{X_j}{x_U,X_W}$  by a simple verification of inclusion in a set; in this way, $\Ind\in \mathcal{D}_C(F)$ indicates dependence $\cd{X_i}{X_j}{x_U,X_W}_F$ is true for model $F$, while $\Ind\notin \mathcal{D}_C(F)$ indicates that the independence $\ci{X_i}{X_j}{x_U,X_W}_F$ holds in that model.

\subsection{Comparison of undirected graphical models}\label{sec:ugm}

In the following section, we will show that log-linear models present several exponential complexities in their structure, the first of which is analogous to an exponentiality present in the space of structures of UGMs.
Because of this, it may be helpful for the reader to understand such exponentiality in the case of UGMs and how it is overcome by the most widely used measure for comparing these models.
In what follows, we provide a brief explanation of this common approach and its advantages.
The underlying idea is that, by taking advantage of the properties of UGMs, their structures can be compared correctly and completely, avoiding an exhaustive comparison.
These ideas have inspired one aspect of our own approximation, by using similar concepts that apply to the much larger class of log-linear models, and also our proof in Section~\ref{sec:distance}, in which we used more general properties and equivalences that apply to the structures of log-linear models to obtain guarantees that this class of models are correctly and completely compared by our method.

As it can be seen in Eqs.~\ref{eq:tripletesUGM} and ~\ref{eq:DUGMs}, UGMs suffer from an exponentiality in the number of subsets  $U \subseteq V \setminus \{i,j\}$.
The approach for undirected graphical models compares them over the polynomial-size dependency model $\mathcal{D}^{UGM}$, a subset  of $\mathcal{D}^{UGM}_C$ containing only \emph{fully conditional dependencies} defined as the dependencies in $\mathcal{D}^{UGM}_C$ with a maximum-size conditioning set $U$, i.e.,  $U = V \setminus \{i,j\}$.
Formally,
    \begin{equation*}
      \mathcal{D}^{UGM}(H) \equiv\left\{\ \assertion{X_i}{X_j}{X_U} \in \TripletesUGM \ \ \middle\vert\  \cd{X_i}{X_j}{X_U}_H,\,\,\, U=V\setminus \{i,j\}  \ \right\}.
      \label{eq:DPUGMs}
    \end{equation*}

The comparison based on fully conditional dependencies corresponds to the known approach for the comparison of two undirected graphical models: the \emph{Hamming distance} of their graphs, as according  to the \emph{pairwise Markov property}~\cite{PEARL88} this reduced set is nothing more than the edges of the undirected graph, i.e.,

  \[
    \cd{X_i}{X_j}{X_{V\setminus \{i,j\}}}_H \equiv (X_i,X_j) \in E,
  \]
where $E$ is the set of edges of the graph representation of model $H$.

Despite only being conducted on the subset of the fully conditional dependencies, the Hamming distance comparison satisfies the properties of a \emph{metric}: \emph{nonnegativity}, \emph{symmetry}, \emph{discrimination}, and \emph{triangle inequality}\cite{DODGE2006,aliprantis2006}.
The first guarantees that the measure is greater than zero for every possible input.
The second property guarantees that the distance from one model to the other is the same as the distance from the second to the first.
The third property guarantees that for any two undirected models $H_1$ and $H_2$, their distance is zero if and only if they are identical.
And finally, the fourth property guarantees that given three models $H_1$, $H_2$ and $H_3$, the distance from $H_1$ to $H_3$ is always smaller than the sum of the distances between the other two, i.e., the distance from $H_1$ to $H_2$ plus the distance from $H_2$ to $H_3$.

%For the Hamming distance of undirected graphs, the first two properties are trivially satisfied since the Hamming distance is nonnegative and commutative.
For the Hamming distance of undirected graphs, the first two and the last properties are trivially satisfied.
The second property is also easily verified on graphs; nevertheless, it is useful to also inquire whether the satisfaction of this property for graphs implies that the complete dependency models are also identical when the distance between graphs is zero.
In other words, we would like to know if, when the reduced dependency models $\mathcal{D}^{UGM}(H_1)$ and $\mathcal{D}^{UGM}(H_2)$ are equal (Hamming distance of zero), then the complete dependency models $\mathcal{D}^{UGM}_C(H_1)$ and $\mathcal{D}^{UGM}_C(H_2)$ are also equal.
To the best of our knowledge, this statement has no formal proof in the literature, yet we believe it is not difficult to prove.
As an intuitive justification, let us note, first, that the equality over fully conditional dependencies is equivalent to the equality of the undirected graphs.
By the Markov properties~\cite{LAURITZEN96,koller09}, any (general) conditional dependence in $\mathcal{D}^{UGM}$ can be read from a graph, thus determining the complete model.
Then, the equality of these subsets of dependencies implies the equality of the complete dependency models:

\begin{equation}
  \mathcal{D}^{UGM}(H_1) = \mathcal{D}^{UGM}(H_2) \iff  \mathcal{D}^{UGM}_C(H_1) = \mathcal{D}^{UGM}_C(H_2).
  \label{eq:discriminationDMUGMs}
\end{equation}
%Lastly, triangle inequality is satisfied by the Hamming distance of the graphs, and then\dots
%is it true that this inequality implies the inequality for complete dependency models?

%\resaltar{TODO Tendría que estar justificada cada propiedad pero la última parece compleja y jamás vi demostración de esto en la literatura.}

\section{Structure comparison between log-linear models}
  \label{sec:structureComparisonLL}

In this section we will show how to arrive at a formal definition of the sets in a confusion matrix for directly and thoroughly comparing the structures of log-linear models. 
%In this section we aim to set the ground for an efficient and valid approach for computing a thorough comparison of log-linear structures.
The section begins by describing an exhaustive brute-force approach for this comparison, while highlighting its main sources of exponential computational complexities.
Then, it motivates and formalizes some required approximations, and proves that, despite these approximations, the resulting comparison is \emph{valid}.
% and that despite the (apparent) remaining exponentialities there is an alternative method for \emph{efficiently} computing the comparison.
With this result, we can continue to address the remaining source of complexity in Section~\ref{sec:approach}. 
%Its validity is proven in Section~\ref{sec:distance}, by introducing a measure based on the confusion matrix, and then proving that this measure is a \emph{distance measure} or \emph{metric}.

Comparing the structures of two log-linear models $F$ and $G$ implies comparing all the independencies and dependencies encoded in each of them.
An exhaustive, straightforward approach for this comparison should  examine each possible triplet from the set $\Tripletes$, testing
its membership in both $\mathcal{D}_C(F)$ and $\mathcal{D}_C(G)$. 
%Throughout this work we will use the following convention: true positives correspond to interactions that are present in both models, true negatives correspond to interactions that are absent in both models, false positives correspond to interactions present in the second structure that are absent in the first one, and false negatives correspond to absent interactions in the second structure that are present in the first one.
Throughout this work we will use the convention that a positive case corresponds to a dependence or interaction, whereas the absence of an interaction is a negative case, in accordance with the comparison of UGMs. 
Then, the dependency model comparison results in a \emph{confusion matrix} for  $F$ and $G$, with  two correct cases and two incorrect cases:
if the triplet belongs to both, one count is added to \emph{true positives}
($TP_C$); if the triplet is missing in both, it is counted as a \emph{true
negative} ($TN_C$); if $F$ does not contain the triplet but $G$ does, it is
counted as a \emph{false positive} ($FP_C$); and if the triplet belongs to $F$ but
not to $G$, it counts as a \emph{false negative} ($FN_C$).
Formally,

\begin{align}
  TP_C &= \left\vert \ \left\{ \   \Ind \in \Tripletes  \mid \Ind \in \DC{F} \land \Ind \in
\DC{G} \  \right\} \ \right\vert, \label{eq:TPcomplete}\\
FN_C &= \left\vert \  \left\{ \   \Ind \in \Tripletes  \mid \Ind \in \DC{F} \land \Ind
\not\in \DC{G} \  \right\} \ \right\vert,\label{eq:FNcomplete}\\
FP_C &= \left\vert \  \left\{ \   \Ind \in \Tripletes  \mid \Ind\not\in \DC{F} \land
\Ind \in \DC{G}  \ \right\} \ \right\vert,\label{eq:FPcomplete}\\
TN_C &= \left\vert  \ \left\{ \   \Ind \in \Tripletes  \mid \Ind \not\in \DC{F} \land \Ind
\not\in \DC{G}  \ \right\} \ \right\vert, \label{eq:TNcomplete}
\end{align}
where, again, the subscript $C$ denotes the fact that this confusion matrix is
computed over the \emph{complete} dependency models.

Unfortunately, the complexity of these evaluations depends directly on the cardinality of $\Tripletes$, which is
exponential in three possible ways:

\begin{enumerate}
    \item There is an exponential number of subsets of $V\setminus \{i,j\}$.
    \item For each subset of $V\setminus \{i,j\}$, there is an exponential number of disjoint sets $U$ and $W$. In other words, let $S \subseteq V\setminus \{i,j\}$; then, for each possible $S$, we have that $U$ and $W$ can be all partitions of $S$ into two sets, plus the cases where $U=\emptyset, W=S$ and $U=S,W=\emptyset$.
    \item For each possible $U$ and $W$ where $U$ is not empty, there is a number of contexts $x_U$ that is exponential in the size of $U$.
\end{enumerate}

In order to give an intuition of the context-specific dependency model and its complexity, we provide a simple example.

\begin{example}\label{ex:dmodel}
  Let $V=\{0,\dots,3\}$ be the index set of binary variables $X_V$.
  Any two log-linear models $\modelA$ and $\modelB$ over this domain can be represented by their context-specific dependency models $\mathcal{D}_C(\modelA)$ and $\mathcal{D}_C(\modelB)$, where each contains some subset of all possible marginal, conditional and context-specific dependency assertions, as summarized in Table~\ref{table:complexity}.

  \begin{table}
    \centering
    \begin{tabular}[ht]{p{3cm} p{3cm} l}
      Case & \# of assertions & Examples \\\hline
      $U=W=\emptyset$ & 6 & $\assertionM{X_0}{X_2}$ \\\hline
      \multirow{2}{3cm}{$|W|=1, U=\emptyset$} & \multirow{2}{4em}{12} & $\assertion{X_0}{X_1}{X_2}$ \\
      &                 & $\assertion{X_0}{X_1}{X_3}$ \\\hline
      \multirow{2}{3cm}{$|W|=2, U=\emptyset$} & \multirow{2}{4em}{6} & $\assertion{X_0}{X_1}{X_2,X_3}$ \\
      && $\assertion{X_1}{X_2}{X_0,X_3}$ \\\hline
      \multirow{2}{3cm}{$W=\emptyset, |U|=1$} & \multirow{2}{4em}{24} & $\assertion{X_0}{X_1}{X_2=0}$ \\
      && $\assertion{X_0}{X_1}{X_2=1}$ \\\hline
      \multirow{2}{3cm}{$W=\emptyset, |U|=2$} & \multirow{2}{4em}{24} & $\assertion{X_0}{X_1}{X_2=0,X_3=0}$ \\
      &&  $\assertion{X_0}{X_1}{X_2=0,X_3=1}$ \\\hline
      \multirow{3}{3cm}{$|W|=1, |U|=1$} & \multirow{2}{4em}{24} & $\assertion{X_0}{X_1}{X_2=0,X_3}$ \\
      && $\assertion{X_0}{X_1}{X_2=1,X_3}$ \\
      && $\assertion{X_0}{X_1}{X_2,X_3=0}$
    \end{tabular}
    \caption{Each row shows a subset of Eq.~\ref{eq:dmcomplete} for a domain with 4 binary variables. The first column determines the subset according to the cardinality of the conditioning sets, the second column indicates the number of assertions present in each subset, and the third column exemplifies a few of those assertions.}
    \label{table:complexity}
  \end{table}

  In total, with this representation, we would need to test 96 unique assertions per model in order to produce the counts of the confusion matrix (Eqs.~\ref{eq:TPcomplete}~to~\ref{eq:TNcomplete}) for $\modelA$ and $\modelB$.
\end{example}

Our proposal addresses the three exponentialities.
The first one is addressed by adapting an approximation that is widely used for the subclass of UGMs, described in Section~\ref{sec:ugm}.
Although the approach for these models is based on unassigned conditioning sets, it can be applied similarly to context-specific dependency models by considering only assertions in which $S = U \cup W = V\setminus\{i,j\}$.
In this way, we reduce the number of possible sets $S$ from the power set of $V\setminus\{i,j\}$ to only one set per pair $(i,j)$.
%In the case of context-specific dependency models, in order to guarantee that the models are compared correctly and completely, it must be proven that the subset of triplets used by our method satisfies the same properties.

%In this work we propose to follow this same approach for the case of log-linear models, requiring that only cases with $U \cup W = V\setminus\{i,j\}$  should be considered.

Unfortunately, this has no impact on the second nor the third exponentialities, which are specific to this class of models.
With the first reduction we have one choice of set $S$ per each pair of variables, but from each pair this set has an exponential number of assignments to $U$ and $W$; that is, there is an exponential number of ways of splitting the conditioning set into an assigned set and unassigned set.
This second exponentiality is addressed through further reductions in the number of comparisons, by considering only assertions where $W=\emptyset$ and $U=V\setminus\{i,j\}$.
The conditioning sets in these assertions thus correspond to $\Xij$, the set of \emph{fully contextualized contexts} as defined in Section \ref{sec:randomVariables}.
In what follows, we rename $U$ as $Z$ when referring to this case of fully-contextualized conditioning sets.
This reduction is justified by the ability of context-specific structures to represent more general dependencies and independencies, based on the equivalences in Section~\ref{sec:independence} (in particular, Eq.~\ref{eq:ci-csi} and its negation).

Both reductions are then formalized by defining $\DP{F}$, a reduced version of the complete dependency model $\mathcal{D}_C(F)$ (for an arbitrary model $F$),  that contains one assertion of the form $\assertion{X_i}{X_j}{\z}$ for every $X_i \neq X_j \in X_V$ and for every fully-contextualized context $\z \in \Xij$. We define $\DP{F}$ by formalizing the reduced set of triplets $\TripletesP$ as

      \begin{equation}
           \TripletesP = \left\{\trip \ \ \middle\vert\  i \neq j \in V,\,\,\, \z \in \Xij  \right\},\label{eq:tripletesP}
        \end{equation}

to then define the reduced, fully-contextualized dependency model of a model $F$ as
 \begin{equation}
   \DP{F} \equiv\left\{\ \trip \in \TripletesP \ \ \middle\vert\  \cd{X_i}{X_j}{\z}_F \  \right\},
      \label{eq:DP}
    \end{equation}

which reduces the comparison of two log-linear models $F$ and $G$ of Eqs.~\ref{eq:TPcomplete}-\ref{eq:TNcomplete} to the computation of the fully-contextualized confusion matrix:

\begin{align}  TP &= \left\vert \ \left\{\  \Ind \in \TripletesP  \mid \Ind \in \DP{F} \land \Ind \in
\DP{G} \ \right\}\ \right\vert,\label{eq:TPdepmodel}\\
FN &= \left\vert\  \left\{\   \Ind \in \TripletesP  \mid \Ind \in \DP{F} \land \Ind
\not\in \DP{G} \ \right\}\ \right\vert,\label{eq:FNdepmodel}\\
FP &= \left\vert\  \left\{\   \Ind \in \TripletesP  \mid \Ind\not\in \DP{F} \land
\Ind \in \DP{G}\  \right\}\ \right\vert,\label{eq:FPdepmodel}\\
TN &= \left\vert\  \left\{\   \Ind \in \TripletesP  \mid \Ind \not\in \DP{F} \land \Ind \not\in \DP{G}\  \right\}\ \right\vert.\label{eq:TNdepmodel}
\end{align}

    \begin{example}\label{ex:fcdmodel}
      By using Eq.~\ref{eq:DP} for defining $\modelA$ and $\modelB$ from \autoref{ex:dmodel}, the only assertions needed correspond to the fifth row in Table~\ref{table:complexity} (the case where $W=\emptyset$ and $|U|=2$).
      To compare $\mathcal{D}(\modelA)$ and $\mathcal{D}(\modelB)$, we would have to test these 24 assertions on each of them in order to produce the values of the confusion matrix, instead of the 96 per model of the exhaustive dependency model.
    \end{example}

To validate this reduced comparison we will prove that the errors $(FP+FN)$ computed over the fully-contextualized confusion matrix is a metric, which means that it satisfies the properties of \emph{non-negativity},  \emph{discrimination}, \emph{symmetry} and \emph{triangle inequality}.
The proof that the fully-contextualized accuracy is a distance is rather long, and thus has been postponed to ~\autoref{T:distance} in Section~\ref{sec:distance}.
To give an intuition of how FC conditioning sets can represent arbitrary structures, we introduce the following example:

\begin{example}\label{ex:DP}
  Let ${M}$ be a model over 3 binary variables $X_V$, $V=\{0,1,2\}$. Suppose $M$ is saturated except for the context-specific independence $\ci{X_1}{X_2}{X_0=1}$.
Using the representation proposed by Eq.~\ref{eq:DP}, its structure can be written as

  \begin{align*}
    \DP{M} = \{ &
     \assertion{X_1}{X_2}{X_0=0}, \\
&    \assertion{X_0}{X_2}{X_1=0}, \\
&    \assertion{X_0}{X_2}{X_1=1}, \\
&    \assertion{X_0}{X_1}{X_2=0}, \\
&    \assertion{X_0}{X_1}{X_2=1} \}.
    \label{eq:exampleDP}
  \end{align*}

  The context-specific structure is simply represented by the absence of $\assertion{X_1}{X_2}{X_0=1}$, while $\assertion{X_1}{X_2}{X_0=0}$ is present.
  But $M$ should also include conditional and marginal dependencies among the other variables.
  By Eq.~\ref{eq:implicationindependence},

  $$\cd{X_0}{X_2}{X_1=0} \land \cd{X_0}{X_2}{X_1=1} \implies \cd{X_0}{X_2}{X_1},$$

  and by the contrapositive of the Strong Union axiom,

  $$\cd{X_0}{X_2}{X_1}\implies \md{X_0}{X_2},$$

  and likewise for $X_0$ and $X_1$.

  In this way, all dependencies in an exhaustive dependency model can be also encoded by this set.
\end{example}

At this point then, the outcome is a fully-contextualized confusion matrix, where the first two exponentialities are addressed by relying on the properties of the model.
%Moreover, we can already define a metric over this confusion matrix.
However, this definition still presents the third exponentiality.
In the following section, we propose a method that overcomes such exponentiality by using an equivalent representation and applying an efficient algorithm.

\section{Approach for an efficient comparison}\label{sec:approach}

This section presents an efficient alternative to the brute-force algorithm for computing the fully-contextualized confusion matrix of Eqs.~\ref{eq:TPdepmodel} - \ref{eq:TNdepmodel}.
The approach is presented in three parts. First, in Section~\ref{sec:setform}, we re-arrange the confusion matrix into a simpler form based on sets of contexts. 
Then, in Section \ref{sec:preliminaries}, we present some preliminary definitions and concepts that will allow us to operate with these sets.
We conclude with Section \ref{sec:theorem1}, which relies on all of these foundations to achieve an efficient method for computing the confusion matrix.

\subsection{Confusion matrix in set form}\label{sec:setform}

First, we note that for any arbitrary set of features $F$, $\DP{F}$  can be partitioned over  mutually exclusive dependency sets $\DPij{F}$, i.e.,

\begin{equation}
  \DP{F} =\bigcup_{i \neq j \in V} \DPij{F}. \label{eq:decomposition}
\end{equation}

This follows by first noticing that, from its definition in
Eq.~\ref{eq:tripletesP}, the triplet set $\TripletesP$ can be easily partitioned over pairs $(i,j)$, i.e.,

\begin{equation*}
      \TripletesP^{ij} = \left\{\ \trip \ \ \middle\vert\ \z \in \Xij \ \right\},\label{eq:tripletesPij}
\end{equation*}

and that from its definition in Eq.~\ref{eq:DP}, $\DP{F}$ is partitioned
accordingly, resulting in

\begin{equation*}
      \DPij{F} \equiv \left\{\ \trip \in \TripletesP^{ij}\ \ \middle\vert\ \cd{X_i}{X_j}{\z}_F \ \right\}.
        \label{eq:DPij}
\end{equation*}
or simply

\begin{equation*}
      \DPij{F} \equiv \left\{\ \trip\ \   \middle\vert\ \ \z \in \Xij, \,\,\, \cd{X_i}{X_j}{\z}_F \ \right\}.
        \label{eq:DPij}
\end{equation*}

A further simplification is achieved by noticing that all elements in
$\TripletesP$ differ from each other solely by the FC conditioning set $\z$, resulting in an alternative way of writing the dependency model $\DPij{F}$ that simply specifies those FC conditioning sets $\z$ for which the triplet is a dependency according to model $F$, i.e.,

\begin{equation}
    \Xij(F) \equiv \{\ \z \in \Xij ~\mid ~ \cd{X_i}{X_j}{\z}_F \ \}.
    \label{eq:XijF}
\end{equation}

From the latter and the decomposition of \DP{F} in Eq.~\ref{eq:decomposition},  we can break down the confusion matrix of Eqs.~\ref{eq:TPdepmodel} - \ref{eq:TNdepmodel} over the configuration sets $\Xij(F)$ and $\Xij(G)$ as follows

  \begin{equation}
    TP = \sum_{i\neq j\in V}{TP_{ij}} ;\quad TP_{ij} = | \{ \z \in \Xij ~|~  \z \in \Xij(F) \land  \z \in \Xij(G) \}|
    \label{eq:TPij}
  \end{equation}

  \begin{equation}
    FN = \sum_{i\neq j\in V}{FN_{ij}} ;\quad FN_{ij} = | \{ \z \in \Xij |  \z \in \Xij(F) \land  \z \notin \Xij(G) \}|
    \label{eq:FNij}
  \end{equation}

  \begin{equation}
    FP = \sum_{i\neq j\in V}{FP_{ij}} ;\quad FP_{ij} = | \{ \z \in \Xij |  \z \notin \Xij(F) \land  \z \in \Xij(G) \}|
    \label{eq:FPij}
  \end{equation}

  \begin{equation}
    TN = \sum_{i\neq j\in V}{TN_{ij}} ;\quad TN_{ij} = | \{ \z \in \Xij |  \z \notin \Xij(F) \land  \z \notin \Xij(G) \}|.
    \label{eq:TNij}
  \end{equation}

  Then, from basic set equivalences, one can observe that the conjunction in the definition of $TP_{ij}$ makes it equivalent to the intersection of two sets, one for each term in the conjunction, namely,

\begin{align}
    TP_{ij} &=  \Big|
             \Big\{ \z \in \Xij ~~|~~ \z \in \Xij(F) \Big\} \bigcap
            \Big\{ \z \in \Xij ~~|~~ \z \in \Xij(G)  \Big\}
           \Big|,  \nonumber \\
           &=  |\Xijof{F} \cap \Xijof{G} | .\label{eq:TPa}
 \end{align}

Similarly, by set equivalences, the conjunctions of set inclusion and exclusion of $FN_{ij}$ and $FP_{ij}$ can be re-expressed as the difference of two sets, to obtain

 \begin{align}
    FN_{ij} &=  |\Xijof{F} \setminus \Xijof{G} |, \label{eq:FNa}\\
    FP_{ij} &= |\Xijof{G} \setminus \Xijof{F} |  \label{eq:FPa}.
  \end{align}
Finally, to simplify the expression for $TN_{ij}$ we first extract the negation to obtain $\lnot( \z \in \Xij(F) \lor \z \in \Xij(G)) $, and rewrite the negation as set complement and the disjunction as set union, to obtain

 \begin{equation}
    TN_{ij} =  |\overline{ \Xijof{F} \cup \Xijof{G}  }|. \label{eq:TNa}
 \end{equation}

\subsection{Preliminary definitions}\label{sec:preliminaries}
The following are a number of definitions, naming conventions,  and equivalences
that will be useful throughout the remainder of this section.

\begin{definition}[Union of features]\label{def:union}

Given two features $f$ and $g$, if for all $X_k \in \scope{f} \cap \scope{g},\ X_k(f) = X_k(g)$ (i.e., $f$ and $g$ have no incompatible assignments), then a \emph{union feature} $f \cup g$ can be defined
as a new feature $h$ such that, for all $X_k \in \scope{f} \cup \scope{g},~ X_k \in \scope{h}$ and

          \begin{equation*}
          X_k(h)  \equiv \left\{
               \begin{array}{ll}
               X_k(f) & \quad  if~ X_k \in \scope{f} \\     	     X_k(g) & \quad if~ X_k \in \scope{g}.             \end{array}
    	    \right.
    \label{eq:def:union}
          \end{equation*}

\end{definition}

\begin{example}\label{ex:union}
  Given $f$ and $g$ defined over a domain $X_V$ where $V=\{0,\dots,4\}$:

  \begin{align*}
  f&=<X_0=1,X_2=0,X_3=1>\text{, and} \\
  g&=<X_0=1,X_1=0,X_4=0>\text{, then their union is}\\
  f \cup g &= <X_0=1,X_1=0,X_2=0,X_3=1,X_4=0>.
 \end{align*}
 \end{example}

 \begin{definition}[Union of features over $(i,j)$]\label{def:unionij}

   Given a set of features $H$ over variables $X_V$, if for all $h\in H$ it is satisfied that $X_i,X_j \in S_h$, and that $\forall X_k\in \bigcup_{h\in H} S_{\p{h}}$ there are no incompatible assignments, i.e., $\forall h, h' \in H, \forall X_k \in \scope{h} \cap \scope{h'} \setminus \{X_i,X_j\},\ X_k(h) = X_k(h')$, then the \emph{union of features $h\in H$ over indices $(i,j)$}
   is denoted as $\cupij_{h\in H} h$ and is defined as any of the possible unions
   %, denoted by the symbol $\cupij$, is defined as any of the possible unions

   \begin{equation}\label{eq:unionij}
	    \bigcupij_{h\in H} h  \equiv \bigcup_{h \in H} \p{h} \cup x_i \cup x_j,
          \end{equation}

	  for any $x_i\in val(X_i)$ and $x_j\in val(X_j)$.
 \end{definition}

 \begin{example}\label{ex:unionij}
  Following Example~\ref{ex:union}, now replace $f$ and $g$ by

  \begin{align*}
  f&=<X_0=1,X_2=0,X_3=1> \text{, and} \\
  g&=<X_0=0,X_1=0,X_2=1,X_4=0> \text{, then their union over }(0,2)\text{ is} \\
   f \cup^{02} g &= <X_0=\cdot,X_1=0,X_2=\cdot,X_3=1,X_4=0>,
 \end{align*}

 where the dots may be replaced by any arbitrary assignment in $val(X_0)$ and $val(X_2)$, respectively.
 \end{example}

\begin{definition}[Fully-contextualized context set of a
  feature]\label{def:FCset}
A \emph{FC context set} for a feature $h$ w.r.t. a pair of distinct variables $X_i,X_j \in X_V$ is the  subset of all FC contexts $\z$ in $\Xij$ over which $X_i$ and $X_j$ are dependent according to feature $h$, that is,

\begin{equation*}
\Xijof{h} \equiv \{\ \z \in \Xij ~\mid ~ \cd{X_i}{X_j}{\z}_h \ \}. \label{eq:contextsetf}
\end{equation*}
which, according to  Eqs.~\ref{eq:FIndep} and \ref{eq:fIndep}, occurs for every
$\z$ such that the assigned variables in $h$ have matching values with $\z$,
provided that $X_i,X_j \in S_h$.
\end{definition}

A  straightforward result from this definition is that the FC context set
$\Xijof{h}$ of a feature $h$ (with assignments for $X_i$ and $X_j$) contains one element for each configuration of all
remaining variables outside its scope; formally,

\begin{equation}
    \left|\Xijof{h} \right| = \left\{
    \begin{array}{ll}
        1 & \quad if~\left|\scope{h}\right|=\left| X_V\right| \\
        \prod_{X_k \in X_V\setminus \scope{h}}{\left|val(X_k)\right|} & \quad otherwise.
  \end{array}\right.\label{eq:cardfeature}
\end{equation}

This cardinality is an important result that will allow us to efficiently compute partial counts in our proposed method. 
Let us illustrate the advantage of this definition with some examples:

 \begin{example}\label{ex:cardfeature}
 Given feature $h=<X_0=0,X_1=0,X_3=1,X_4=0,X_6=1>$, with $V=\{0,\dots,6\}$, $val(X_2)=\{0,1\}$ and $val(X_5)=\{0,1,2\}$, its FC contexts for pair $(0,1)$ are
     \begin{align*}
\Xzoof{h} =  \{ & <X_2=0,X_3=1,X_4=0,X_5=0,X_6=1>,  \\
& <X_2=0,X_3=1,X_4=0,X_5=1,X_6=1>, \\
& <X_2=0,X_3=1,X_4=0,X_5=2,X_6=1>, \\
& <X_2=1,X_3=1,X_4=0,X_5=0,X_6=1>, \\
& <X_2=1,X_3=1,X_4=0,X_5=1,X_6=1>, \\
& <X_2=1,X_3=1,X_4=0,X_5=2,X_6=1> \ \},
\label{eq:examplecardfeature}
\end{align*}
and its cardinality is clearly $6$. We have that $X_V \setminus \scope{h} = \{X_2,X_5\}$; therefore, by Eq.~\ref{eq:cardfeature}, the cardinality is $|\Xzoof{h}| = |val(X_2)| \times |val(X_5)| = 2 \times 3 = 6$.
If we now consider a greater number of unassigned variables, say, 10 binary variables (instead of two as in the previous example), then the exhaustive computation would have to \emph{generate} 1024 contexts and count them, while Eq.~\ref{eq:cardfeature} would simply compute $2^{10}$. 
 \end{example}

The following is a definition for an essential notion in our method, since its efficiency relies on working with sets of features that are a mutually exclusive (but equivalent) version of arbitrary feature sets.
We call these sets \emph{partition models}.
In the following section, we describe Algorithm~\ref{alg:partition}, which computes partition models from any given feature set.

% We define now a partition model of some feature set, postponing to Section~\ref{sec:partitionAlgs} the description of  Algorithm~\ref{alg:partition} that computes partition models for any given feature set.

 \begin{definition}[Partition model]\label{def:partition}
  A set of features $P$ is a \emph{partition model} for the set of features $H$ if and only if

\begin{equation*}
  \Xijof{H} = \Xijof{P}  \text{\quad   and \quad  }
\forall p \neq p' \in P,\, \Xijof{p} \cap \Xijof{p'} = \emptyset ,
\end{equation*}
  \end{definition}

  that is, the matching FC context set of $P$ and $H$ is partitioned by the FC
  context sets $\Xijof{p}$ of features $p \in P$.

  \begin{example}\label{ex:partition}
  Given $V=\{0,\dots,4\}$ and the set of features

    \begin{align*}
H =  \{ & <X_0=0,X_1=0,X_4=0>,  \\
& <X_0=0,X_1=0,X_3=0> \},
\label{eq:examplepartitionH}
\end{align*}
    on the one hand, we note that
\begin{align*}
\Xzoof{H} =  \{ & <X_3=0,X_4=0>,  \\
& <X_3=1,X_4=0>,  \\
& <X_3=0,X_4=1> \},
\end{align*}

and $\Xzoof{<X_0=0,X_1=0,X_4=0>}\ \cap\ \Xzoof{<X_0=0,X_1=0,X_3=0>} = \{<X_3=0,X_4=0>\}$,
so the features in $H$ are not a partition of $\Xzoof{H}$.
  On the other hand, the set
    \begin{align*}
  P =  \{ & <X_0=0,X_1=0,X_4=0>,  \\
& <X_0=0,X_1=0,X_3=0,X_4=1> \}, \end{align*}

  is a partition of $\Xzoof{H}$ because

\begin{align*}
    \Xzoof{P} =  \{ & <X_3=0,X_4=0>,  \\
             & <X_3=1,X_4=0>,  \\
            & <X_3=0,X_4=1> \}, \end{align*}

that is, $\Xzoof{H}=\Xzoof{P}$, and

$$\Xzoof{<X_0=0,X_1=0,X_4=0>} \cap \Xzoof{<X_0=0,X_1=0,X_3=0,X_4=1>} = \emptyset .$$

\end{example}

  A straightforward consequence of the definition of partition models is the
  possibility of efficiently computing the cardinality of the FC context set of
  some partition model $P$. This follows, first, by noticing that the FC
  context set of $P$ can be decomposed into the FC context set of its features
  $p$ as follows,

\begin{equation*}
     \Xijof{P}  = \bigcup_{p\in P}{ \Xijof{p} }, \label{eq:partitionunion}
\end{equation*}

and then, by the fact that the contexts for all $p$ are mutually exclusive,
the cardinality can be expressed as a sum:

\begin{equation*}
    \left\vert \Xijof{P} \right\vert = \sum_{p\in P}{\left\vert \Xijof{p}
  \right\vert },\label{eq:partitioncard}
\end{equation*}

  where its cardinality can be computed efficiently according to Eq.~\ref{eq:cardfeature}.

\subsubsection{Partitioning Algorithm} \label{sec:partitionAlgs}

  \begin{algorithm}[tb]
    \caption{\label{alg:partition} $partition(H)$.}
    \begin{algorithmic}[1]
      \STATE \algorithmiccomment{Given a set of features $H$, it returns its partition model $P$  (see
  Definition~\ref{def:partition}). The notation $(i,j)$ is omitted for
      clarity.}

        \STATE $h' \longleftarrow $ some arbitrary feature $h \in H$
        \STATE $P \longleftarrow \{h'\} $\label{alg:P}
        \FOR {$h\in H\setminus \{h'\}$ }  \label{alg:h}
          \STATE $D_h \longleftarrow \{h\}$\label{alg:Dh}

    \FOR {$p\in P$ }\label{alg:p}
            \STATE $D_{hp} \longleftarrow \emptyset$ \label{alg:DhpInit}
            \FOR {$h'\in D_h$ } \label{alg:hprime}
                \STATE $D_{hp} \longleftarrow D_{hp} \cup D_{h'\setminus p}$\label{alg:Dhp}

                \ENDFOR \label{alg:end:p}
            \STATE $D_h \longleftarrow D_{hp}$\label{alg:DhDhp}

            \ENDFOR \label{alg:end:hprime}
            \STATE $P \longleftarrow P \cup D_h $\label{alg:PDh}
        \ENDFOR  \label{alg:end:h}

        \RETURN $P$
     \end{algorithmic}
  \end{algorithm}

  We now introduce an algorithm whose main purpose  is to produce the partition model $P$ for the input set  of features $H$ over the FC context set $\Xij$.
  According to ~\autoref{def:partition}, this partition model $P$ is equivalent
  to $H$ in that both represent the same FC context set, i.e., $\Xij(H) = \Xij(P)$, but $P$ contains features with no overlapping FC contexts, i.e., $\forall p, p' \in P, \Xij(p) \cap \Xij(p') = \emptyset$.
  The partitioning algorithm is shown in Algorithm~\ref{alg:partition}.

  Producing a partition requires avoiding the double counting of every possible
  FC context in $H$. For a pair of features, say $h,h' \in H$, this is achieved by keeping one of them intact, say $h'$, and subtracting its FC contexts from the other, i.e., producing some set of features $D$ satisfying $\Xij(D) = \Xij(h) \setminus \Xij(h')$.
  For that, we use the operation  \emph{feature difference} defined in~\autoref{lemma:singlediff} of Section~\ref{subsubsec:singlefeatures}. This operation takes the two features $h$ and $h'$ and produces a new set of features $D_{h\setminus h'}$ whose dependency model $\Xij(D_{h\setminus h'}) = \cup_{d \in D_{h\setminus h'}} \Xij(d)$ equals $\Xij(h) \setminus \Xij(h')$. These operations are specific for some given pair $(i,j)$, but its explicit mention is omitted for brevity.

  When $H$ contains more than two features, the basic feature difference
  operation must be conducted over every pair.
  If these were simple sets, one subtraction per element would suffice.   However, when the operations are conducted over features, the difference feature is instead a set of features.
  This makes the procedure more complex.
  First, the algorithm keeps track of the partitioned (subtracted) features in $P$, initialized by a single, arbitrary feature $h'\in H$ in line~\ref{alg:P}.
    The algorithm then conducts two nested loops, one over all remaining features $h  \in H \setminus \{h'\} $ (lines \ref{alg:h}-\ref{alg:end:h}), and the other over each  feature $p \in P$  (lines \ref{alg:p}-\ref{alg:end:p}), with the main idea of subtracting from $h$ every feature $p$,  to produce a new state of $P$ in line~\ref{alg:PDh} that is guaranteed to be a partitioned model for the subset of $H$ that has already been visited.
    The core of the second loop contains initially a subtraction of some $h
  \in H$ minus $p$. However, after the first iteration over $P$, the difference
  of $h$ minus $p$ produces not one, but a set $D_{h \setminus p}$ of features, requiring
  several subtractions in the second iteration, one per $d \in D_{h \setminus
  p}$. This is solved by storing all subtractions in $D_{h}$, initialized with $h$ in line~\ref{alg:Dh}, and updated in line~\ref{alg:DhDhp} with the difference features $D_{hp}$ produced for $p$.
  There is one final difficulty to address: the subtraction of a single $p$ from every feature $h'$  in $D_h$. The only real complication is to collect the resulting features. For that, the loop over $p$ maintains the set $D_{hp}$, initialized empty in line~\ref{alg:DhpInit}, and updated with the set $D_{h'\setminus p}$ resulting from the subtraction of $p$ from $h'$. Only after collecting all difference features for every $h' \in D_h$, $D_h$ is updated again with the set $D_{hp}$ of new differences.

 These procedures could be illustrated with the following example. In the first  iteration of the loop over $H$ (lines
 \ref{alg:h}-\ref{alg:end:h}) $h$ is the second feature from $H$, and P contains the first one, i.e. $p=h'$ (line \ref{alg:p}).
 In the innermost loop of line \ref{alg:hprime}, we have then $h$ as the only element in $D_{h}$; resulting in the subtraction of $p=h'$ from $h$, with the set of features $D_{hp}$ becoming $D_{h \setminus h'}$.
 In line \ref{alg:PDh}, this new set is added to $P$.
   In the third iteration over $H$, the algorithm takes the third element of $H$, say $h''$, which must be
 subtracted from $P$, that at this point equals $P=\{h' \cup D_{h \setminus h'}\}$.
 The interesting aspect of this third iteration is that the loop over features $p \in P$ (lines \ref{alg:p}-\ref{alg:end:p}) now runs over more than one feature.
 For clarity of exposition let us rename these as $P=\{p^1,p^2,\dots,p^{|P|}\}$.
 In the first iteration of this second loop, we obtain the difference set  $D_{h'' \setminus p^1}$.
 In the next iteration, we have to subtract $p^2$ from each  of the features in $D_{h'' \setminus p^1}$.
 At this point it becomes clear why we need the additional third loop over $D_h$ in
 lines~\ref{alg:hprime}-\ref{alg:end:hprime}: for $p^2$, we need to produce
 the difference set $d \setminus p^2$  for each $d$ in $D_h=D_{h' \setminus p^1}$, which must be conducted incrementally.
 At the end of this loop, the resulting difference features are in $D_h$, and we use this  set for $p^3$, the next element in $P$.
 In other words, each iteration over $P$  produces a set $D_h$, which becomes
 gradually smaller w.r.t. the number of FC contexts represented.
 After subtracting all $p \in P$, we have a final set $D_h$ which does not
 overlap with any $p$.
 The union of these difference sets (line \ref{alg:Dhp}) is added to the partition set $P$, and the
 algorithm proceeds with the next feature in $H$.

\subsection{Efficient computation of the confusion matrix} \label{sec:theorem1}

In this section we present the approach for efficiently computing the comparison
of two log-linear models $F$ and $G$, as expressed by the set form of the
confusion matrix (Eqs.~\ref{eq:TPa}-\ref{eq:TNa}). The approach is expressed
in the following theorem, which includes a proof of correctness, i.e., a proof
that the confusion matrix computed by the efficient method is guaranteed to
produce the same counts as the FC confusion matrix of Eqs.~\ref{eq:TPa}-\ref{eq:TNa}.

  \begin{theorem} \label{T:method}
   Let $F$ and $G$ be two log-linear model structures over $X_V$.
   The fully contextualized confusion matrix $TP$, $FP$, $FN$, and $TN$ of $G$
   w.r.t. $F$ can be computed efficiently in terms of $|V|$ as follows:

      \begin{align}
  TP &= \sum_{i\neq j\in V}{TP_{ij}} ;   &      TP_{ij} &=  |\Xijof{F} \cap \Xijof{G} |  \equiv   \sum_{p\in P^{TP}}\left\vert{ \Xijof{p}} \right\vert, \label{eq:T:method:TP} \\[5pt]
    FN &= \sum_{i\neq j\in V}{FN_{ij}} ;     &   FN_{ij} &=  |\Xijof{F} \setminus \Xijof{G} |  \equiv  \sum_{p\in P^{FN}}\left\vert{ \Xijof{p}} \right\vert, \label{eq:T:method:FN} \\[5pt]
   FP &= \sum_{i\neq j\in V}{FP_{ij}} ;     &   FP_{ij} &=  |\Xijof{G} \setminus \Xijof{F} |  \equiv  \sum_{p\in P^{FP}}\left\vert{ \Xijof{p}} \right\vert, \label{eq:T:method:FP}
    \end{align}
      \begin{equation}
    TN =     \left(\sum_{i\neq j\in V}{\,\, \prod_{k \in V\setminus \{i,j\}
        }{\left|val(X_k)\right|} } \right)- TP - FN - FP .
          \label{eq:T:method:TN}
      \end{equation}

     where the cardinalities of
          the FC contexts $\Xijof{p}$ over the individual features $p$ can be computed efficiently by Eq.~\ref{eq:cardfeature},
          and feature sets
     $P^{TP}$, $P^{FN}$, and $P^{FP}$ are partition models of the sets of features $H^{TP}$, $H^{FN}$, and $H^{FP}$, respectively, defined as follows:

     The equivalent set for $TP_{ij}$, named $H^{TP}(F,G)$ (where we omit the dependence over $(i,j)$ for clarity), is defined as
    \begin{equation}
        H^{TP}(F,G) \equiv \left\{ {f} \cupij {g}  \middle\vert  f\in F,\, g\in G,\, {C}_1(f),\, {C}_1(g),\, {C}_2(f,g) \right\},\label{eq:HTP}
            \end{equation}
    that is, it contains one \emph{union feature} $ {f} \cupij {g} $ (as defined by ~\autoref{def:unionij} in Section~\ref{sec:preliminaries}) for each pair of features $f \in F$ and $g \in G$ that satisfy the following conditions for non-empty union features:

     \begin{align}
         {C}_1(f)   &\equiv X_i \in \scope{f} \land X_j \in \scope{f}, \label{eq:C1} \\[5pt]
         {C}_1(g)   &\equiv X_i \in \scope{g} \land X_j \in \scope{g}, \label{eq:C1g} \\[5pt]
         {C}_2(f,g) &\equiv \forall X_k \in \scope{f} \cap \scope{g} \setminus \{X_i,X_j\},\quad  X_{k}(f) = X_{k}(g).  \label{eq:C2}
     \end{align}

     Condition $C_1$ requires feature $f$ (resp. $g$)  to contain both $X_i$ and $X_j$ in its
     scope. Condition $C_2$ requires features $f$ and $g$ to contain no incompatible values on variables other than $X_i$ and $X_j$; and it corresponds to the requirement for the existence of the union of features over $(i,j)$, as specified in its definition.

     The equivalent set for $FN_{ij}$ is defined as

   \begin{equation}
   \begin{aligned}
               H^{FN}(F,G) \equiv& \phantom{\bigcup\,} \left\{ {f} \, \cupij \, \bigcupij_{d\in \DEf} d  \, \middle\vert \,
                  f \in \Ftwo \setminus \Fone,  \,\DEf \in \vvarprod_{g \notin
      \Gcomp}\, \DE \right\}  \\[8pt]
                                    & \bigcup  \,  \left\{ \quad\quad\,\,\,\, \bigcupij_{d\in \DEf} d  \,  \middle\vert \,
                  f \notin (\Ftwo\cup \Fone),  \,\DEf \in \vvarprod_{g \notin
      \Gcomp}\, \DE \right\},
    \end{aligned}
    \label{eq:HFN}
   \end{equation}

   where

   \begin{equation*}
     {f} \bigcup\limits^{ij}  \bigcup\limits^{ij}_{d\in \DEf} d
   \end{equation*}

   is computed with Eq.~\ref{eq:unionij} as only one union operation over $(i,j)$ over the set $\{f\}\cup \DEf$, and
        with $\Fone$, $\Ftwo$, $\Gone$, and $\Gtwo$ defined as

    \begin{align}
      \Fone &\equiv \left\{ f\in F \mid \Gone \neq \emptyset \right\};  \label{eq:Fone} \\[5pt]
      \Ftwo &\equiv \left\{ f\in F \mid \Gtwo \neq \emptyset \right\};  \label{eq:Ftwo} \\[5pt]
                        \Gone~\equiv~&\left\{ g\in G \mid \left[ C_1(g) \land
			\p{g} \subseteq \p{f} \right] \lor \lnot C_1(f) \right\} \label{eq:Gone}; \\[5pt]
                        \Gtwo~\equiv~&\left\{ g\in G \mid \lnot C_2(f,g) \lor \lnot C_1(g) \right\};  \label{eq:Gtwo}
                 \end{align}

    and set  $\DE$ for features $f$ and $g$ defined as

                                       \begin{equation}
					 \DE = \bigcup_{k \in S_g\setminus S_f} {\DE}_{(k)} ,  \tag{\ref{eq:DE}}
                \end{equation}

                with
                \begin{equation}
                    {\DE}_{(k)} = \left\{ \text{\ features }d \  \middle\vert \begin{array}{ll}
		      S_d = S_f\cup S_g^{\leq k}; \\[6pt]
                    \forall X_m\in S_{f}, X_m(d)=X_m(f);  \\[6pt]
                    \forall X_m\in S_g^{\leq k}~s.t.~m < k\setminus S_f, X_{m}(d) = X_{m}(g); \\[6pt]
		    X_k(d) \neq X_k(g) \end{array}\right\},\tag{\ref{eq:DEk}}
                \end{equation}

    where the notation $X_m(d)$ refers to the assignment to variable $X_m$ in feature $d$ (likewise for $X_m(f)$ and $X_k(g)$), and $S_g^{\leq k} = \{m\in S_{\p{g}} | m \leq k\}$.

    We will now aim to give some intuitions regarding Eqs.~\ref{eq:HFN}-\ref{eq:Gtwo}, and Eqs. \ref{eq:DE} and \ref{eq:DEk}; however, the full rationale behind these definitions will become clear in the proof.

Set $H^{FN}(F,G)$ is the union of two sets of features.
The first contains one feature union per $f$ in $\Ftwo$ and not in $\Fone$,  obtained by computing the feature union of the feature $f$ without $X_i$ and $X_j$, i.e., $\p{f}$, and each feature $d$ in the set of features  $\DEf$ corresponding to $f$, computed as the cross product over all feature sets $\DE$, one for each $g$ that is neither  in   $\Gone$ nor $\Gtwo$.   The second set of features differs in two aspects: it contains  feature unions over features $d$ only; and features $d$ belong to the $\DEf$ over the features $f$ that are neither part of $\Fone$ nor $\Ftwo$.
Although $D_{fg}^E$ is an exponential set by definition, our interest does not lie in the computation of this set, but in the cardinality of an equivalent set, i.e., its partition model $P^{FN}$. This model will be obtained by the use of Algorithm~\ref{alg:partition} and with syntactic operations over the features of the output set, as will be shown in Section~\ref{subsubsec:singlefeatures}.

As regards $\Fone$, we can observe that the features in this set will be excluded from the construction of the set $H^{FN}(F,G)$.
Specifically, set $\Fone$ contains all features in $F$ that do not satisfy $C_1(f)$, i.e., at least one of $X_i$ or $X_j$ is not in its scope; plus features for which there is at least one $g\in G$ that satisfies both $\p{g} \subseteq \p{f}$ and $C_1(g)$, i.e., it contains both $X_i$ and $X_j$ in its scope, and its remaining assignments are a subset of the assignments in $\p{f}$.
The first case are features in $F$ which do not encode dependencies among $X_i$ and $X_j$, and the second case are features which do encode dependencies, but these particular dependencies are also encoded by some feature(s) in $G$.
For counting false negatives, all these features in $\Fone$ must be excluded, as can be seen in the definition of both parts of the union in $H^{FN}(F,G)$.

Set $\Ftwo$, instead, contains all features $f\in F$ for which there exists at least one $g$ that does not satisfy $C_2(f,g)$, or where $g$ does not have the pair of variables $X_i$ and $X_j$ in its scope.  $\lnot C_2(f,g)$ occurs when the intersection of the scopes $S_f$ and $S_g$ is not empty, and for at least one variable in this intersection, its assignment differs in $f$ and in $g$.
In this case, these features in $F$ are encoding dependencies which are not encoded in $G$.
When this happens, all FC contexts corresponding to $f$ must be included in the resulting set, and this is why inclusion in the set $\Ftwo \setminus \Fone$ determines the presence of $\p{f}$ in the first part of the union of $H^{FN}(F,G)$.

As for the set $\DE$, it contains the features in $F$ that do not correspond to the previous two cases, i.e., features that do not belong in $\Fone$ nor $\Ftwo$.
While its definition may seem complex at first, it may be intuitively understood as a set of features that correspond to the non-trivial difference of two FC context sets between two features.
     Because of the way in which this set is constructed by using an operation between single features (which will be described later in Section~\ref{subsubsec:singlefeatures}),
     it is partitioned over multiple subsets ${\DE}_{(k)}$ where each $k$ is related to a variable that is in the scope of $g$ but not in the scope of $f$.
     All features in all subsets contain the same assignments as $f$, which guarantees that the FC contexts represented by the set will belong to $\Xijof{f}$.
     In addition, since the difference set must contain contexts that are not in $g$, each feature in ${\DE}_{(k)}$ contains, for a particular variable $X_k$ (which is assigned in $g$ but not in $f$), an assignment that differs from $X_k(g)$.
     This includes contexts that should be in $f$ (due to the variable being unassigned in that feature) but excludes configurations that are in $g$ (which should be subtracted).
     The complete reasoning behind the difference set will become clear in \autoref{lemma:singlediff}.

     We conclude with the definition of the equivalent set for $FP_{ij}$, which is
     simply the reversed version of the respective set for the $FN_{ij}$:
    \begin{equation*}
      H^{FP}(F,G) \equiv H^{FN}(G,F).\label{eq:HFP}
    \end{equation*}

\end{theorem}

%\todo{poner lemma:HTP y lemma:HFN}

\vspace*{0.5cm}

  \begin{proof}
  The decomposition of $TP$, $FP$, and $FN$ into $TP_{ij}$, $FP_{ij}$, and $FN_{ij}$ follows from Eqs.~\ref{eq:TPij}-\ref{eq:TNij}. Then, the proof of
  the equivalences of these three cases with the computationally efficient expressions of the r.h.s.
  proceeds by demonstrating their equivalence with sets $H^{TP}$, $H^{FP}$, and $H^{FN}$ through~\cref{lemma:HTP,lemma:HFN}, presented and proven in the subsections immediately following this proof.
 % For clarity of exposition, the proofs of these lemmas have been moved to their own section in Appendix~\ref{app:lemmas}.

  We proceed now to discuss the details of these proofs, together with the case of $TN$ that follows a different structure.

    \begin{enumerate}

            \item \emph{True positives}:

        \begin{align*}
          TP_{ij} &=  \left|\Xijof{F} \cap \Xijof{G}\right| & \text{by Eq.~\ref{eq:TPa}}\\[0.7ex]
    &=  \left|\Xijof{H^{TP}}\right| & \text{by~\autoref{lemma:HTP}}\\[0.7ex]
                                    &= \left|\Xijof{P^{TP}}\right| & \text{by equivalence (\autoref{def:partition})}\\[0.7ex]
                  &= \left|\bigcup_{p\in P^{TP}} \Xijof{p}\right|   &\text{by~\autoref{al:Funion} (see~\ref{app:auxiliary})}\\[0.7ex]
                  &= \sum_{p\in P^{TP}}\Xijof{p} & \text{by the fact that $P^{TP}$ is a partition.}
        \end{align*}

            \item \emph{False negatives:}

        \begin{align*}
          FN_{ij} &= \left|\Xijof{F} \setminus \Xijof{G} \right|  & \text{By Eq.~\ref{eq:FNa}}\\[0.7ex]
    &=  \left|\Xijof{H^{FN}}\right| & \text{by~\autoref{lemma:HFN}}\\[0.7ex]
                                    &= \left|\Xijof{P^{FN}}\right| & \text{by equivalence (\autoref{def:partition})}\\[0.7ex]
                  &= \left|\bigcup_{p\in P^{FN}} \Xijof{p}\right|  &
      \text{by~\autoref{al:Funion},~\ref{app:auxiliary}}\\[0.7ex]
                  &= \sum_{p\in P^{FN}} \Xijof{p} & \text{by the fact that $P^{FN}$ is a partition.}
        \end{align*}

          \item  \emph{False positives:} For $FP_{ij}$, the proof follows from the fact that it is the same computation as for $FN_{ij}$ but exchanging the operands.

      \item  \emph{True negatives:}  The count $TN$ is computed as the remainder of counts, that is, by discounting the sum of counts $TP$, $FP$, and $FN$ from the total number of fully contextualized configurations. In Eq.~\ref{eq:T:method:TN}, the latter is represented by the first term, a simple operation involving
      a sum over each pair $(i,j)$ where for each, we compute the product
      of the cardinalities of all remaining variables in the domain except for $X_i$ and $X_j$:
            \begin{equation*}
          \sum_{i\neq j\in V}{\,\, \prod_{k \in V\setminus \{i,j\}
        }{\left|val(X_k)\right|} } .
          \label{eq:TotalComparisonsFormula}
      \end{equation*}
            For illustration purposes let us consider the particular case where $|val(X_k)|=m$ for all $k \in V$, for which $TN$ results in

        \begin{align*}
         TN = m^{|V|-2} \; \binom{|V|}{2}  - TP - FN - FP.
          \label{eq:TNformula2}
        \end{align*}
        For example, in a domain with 6 binary variables, i.e., $m=2$, the total number of configurations  is $m^{|V|-2} \; \dbinom{|V|}{2} = 2^4 \times \frac{6 \times 5}{2} = 16 \times 15 = 240$.

    \end{enumerate}
\end{proof}

The equivalences in the proofs for $TP_{ij}$, $FN_{ij}$ and $FP_{ij}$ are possible by the following lemmas, which are proven in~\ref{app:lemmasHTPHFN}:

\begin{restatable*}{lemma}{htp}\label{lemma:HTP}
%\begin{lemma} %\label{lemma:HTP}
  Let $F$ and $G$ be two arbitrary log-linear models over $X_V$, and $H^{TP}(F,G)$ be the set of union-features over $F$ and $G$ defined in Eq.~\ref{eq:HTP}, then

  \[
      \Xijof{F} \cap \Xijof{G} = \Xijof{H^{TP}(F,G)}.
  \]
%\end{lemma}
\end{restatable*}

\begin{restatable*}{lemma}{hfn}\label{lemma:HFN}
%\begin{lemma}
  Let $F$ and $G$ be two arbitrary log-linear models over $X_V$, and $H^{FN}(F,G)$ be the set of union-features over $F$ and $G$ defined in Eq.~\ref{eq:HFN}, then

  \[
      \Xijof{F} \setminus \Xijof{G} = \Xijof{H^{FN}(F,G)}
  \]
 % \end{lemma}
\end{restatable*}

In order to produce the cardinalities of Eqs.~\ref{eq:T:method:TP},~\ref{eq:T:method:FN}~and~\ref{eq:T:method:FP} efficiently, \cref{lemma:HTP,lemma:HFN} propose the decomposition of these operations over simpler elements (features) and define sets $H^{TP}$ and $H^{FN}$ which produce the same FC context sets; furthermore, these sets can be converted into partition sets, which allow for an efficient computation of their cardinality.
In the following section, we show the basis for the definitions of these sets, which are constructed by performing syntactic operations over features, thus avoiding the exponential cost of comparing all elements in $\Xij$ for each $(i,j)$.

\subsubsection{Efficient operations over single features}\label{subsubsec:singlefeatures}
The equations in~\cref{lemma:HTP,lemma:HFN} show that, to efficiently compute $TP_{ij}$, $FN_{ij}$, $FP_{ij}$ and $TN_{ij}$, it
is necessary to produce a procedure for computing $\Xij(f) \cap \Xij(g)$ and another for $\Xij(f) \setminus \Xij(g)$, both of which avoid the complexity of $\Xij$.
These efficient procedures are described in
~\cref{lemma:singleintersection,lemma:singlediff} below.
Due to the length and complexity of their proofs, these are presented in~\ref{app:lemmas}.

We start with ~\autoref{lemma:singleintersection}, which provides an efficient computation of the intersection, without the need to count through the exponential number of FC contexts.
Instead, it can arrive at the same set by comparing the scope and assignments of $f$ and $g$ and, in the worst case, producing a new feature whose contexts are equivalent to the intersection of contexts.

\begin{restatable*}{lemma}{singleinter}\label{lemma:singleintersection}

    Let $f$ and $g$ be two arbitrary features over $X_V$, and let $X_i \neq X_j$ be any two different variables in $X_V$.
    Then, the intersection of the FC contexts $\Xijof{f}$ and $\Xijof{g}$ can be efficiently computed as

                                        \begin{equation}
        \Xijof{f} \cap \Xijof{g} = \left\{
               \begin{array}{ll}
                    \emptyset & \quad if~\lnot C_1(f) \lor \lnot C_1(g) \lor \lnot C_2(f,g)\\[5pt]
		    \Xijof{{f} \cupij {g}} & \quad otherwise,            \end{array}
          \right.
          \label{eq:intersection1}
        \end{equation}
    % \tensionFacu{No veo que se haya propagado a esta expresion la nueva definicion 2. ESTA TENSION SE PROPAGA TAMBIEN AL ENUNCIADO EN EL APENDICE}{Entiendo que segun la (nueva) definicion 2, la notacion correcta seria expresar la union como $f \cup g$ o explicitamente su resultado, i.e., $\p{f} \cup \p{g} \cup x_i \cup x_j$}
     where ${f} \cupij {g}$ is a feature union over $(i,j)$ of ${f}$ and ${g}$ according to ~\autoref{def:unionij},  and $C_1(f),\,C_1(g)$ and $C_2(f,g)$ are defined in
     Eqs.~\ref{eq:C1}-\ref{eq:C2} in \autoref{T:method}.
     Note that the ambiguous definition of~\autoref{def:unionij} is valid, since the assignments for $X_i$ and $X_j$ are only needed by the $\Xij$ function to be non-empty (the values that these variables take in the features need not be equal in both $f$ and $g$, since only by their presence in the scope of said features they encode a dependency among the distributions of the variables).

In other words, for there to be a non-empty intersection, both $X_i$ and $X_j$ must be present in both features and these features must have no incompatible assignments.
\end{restatable*}

\begin{proof}
  See~\ref{app:lemmas}.
\end{proof}

Given their importance, we will illustrate the different cases in Eq.~\ref{eq:intersection1} with some examples.

  \begin{example}\label{ex:intersection}

    We consider a domain $X_V, V=\{0,1,2,3,4,5\}$, where each variable can take values in $\{0,1,2\}$. Let $f$ and $g$ be two features defined as

\begin{align*}
& f = < X_0=2, X_1=1, X_2=0, X_5=0 >, \\
& g = < X_0=0, X_1=1, X_2=0, X_4=1 >.
\end{align*}

We have that both $X_0$ and $X_1$ are in the scope of both features ($C_1$ holds for both $f$ and $g$).
The only variable in $S_f \cap S_g \setminus \{X_0,X_1\}$ is $X_2$, and $X_2(f)=X_2(g)=0$; therefore, $C_2(f,g)$ holds.
Since all conditions are satisfied we have the non-empty case, i.e.,

 \[
   f \cup^{01} g = < X_0=\cdot, X_1=\cdot, X_2=0, X_4=1, X_5=0 >.
    \]

    Then, the resulting FC context set is

    \begin{align*}
    \Xijof{f \cup^{01} g} = \left\{ \right. & < X_2=0, X_3=0, X_4=1, X_5=0 >, \\
                                   &  < X_2=0, X_3=1, X_4=1, X_5=0 >, \\
				 &  \left. < X_2=0, X_3=2, X_4=1, X_5=0 > \right\} .
			       \end{align*}

%\begin{align*}
%  & f^{01} = < X_2=0, X_5=0 >, \\
%  & g^{01} = < X_2=0, X_3=0, X_4=1 >.
%\end{align*}

  \end{example}

    \begin{example}\label{ex:intersection2}
Consider the same scenario as~\autoref{ex:intersection} but replace $f$ by

\begin{align*}
& f = < X_1=1, X_2=0, X_5=0 >.
\end{align*}

In this case, $C_1(f)$ does not hold (because $X_0 \not\in \scope{f}$), making
$C_1$ \emph{false}, satisfying the condition for the empty case.
Intuitively, this is because $f$ does not encode any dependencies between variables $X_i$ and $X_j$.
  \end{example}

    \begin{example}\label{ex:intersection3}
Following~\cref{ex:intersection,ex:intersection2}, replace $f$ by

\begin{align*}
& f = < X_0=2, X_1=1, X_2=1, X_5=0 >.
\end{align*}

In this case, $C_2(f,g)$ does not hold, because $X_2(f)=1$ while $X_2(g)=0$, by which we have that $X_2(f)\neq X_2(g)$; thus, the intersection is empty.
This expresses the fact that the two features have no FC contexts in common.
  \end{example}

We present now \autoref{lemma:singlediff}, which provides an efficient computation
of the difference sets of Eqs.~\ref{eq:FNa}~and~\ref{eq:FPa}, without the need to
compare through the exponential number of FC contexts. Similarly to
the intersection case of  \autoref{lemma:singleintersection}, the Lemma results
in several possible values for the difference set by comparing the scope and assignments of the input
features $f$ and $g$.

\begin{restatable*}{lemma}{singledif}\label{lemma:singlediff}

    The difference of FC sets of arbitrary single features $f$ and $g$ over $X_V$ can be efficiently computed as

  \begin{equation}
    \Xij(f) \setminus \Xij(g)  \equiv \left\{
    \begin{array}{ll}
         \emptyset & \quad  if~ g^{ij} \subseteq f^{ij}\ \land C_1(g), \text{ or }\ \lnot  C_1(f) \\[5pt]
	 \Xijof{f} & \quad  \lnot C_2(f,g)   \ \text{or}\ \lnot C_1(g)\\[5pt]
	 \bigcup_{d\in \DE} \Xijof{d} & \quad otherwise,      \end{array}
     \right.\label{eq:diff:D}
  \end{equation}

  while $C_1(f)$, $C_1(g)$ (presence of variables in the scope of features) and
  $C_2$ (existence of mismatched values) are the same conditions defined in
  Eq.~\ref{eq:intersection1}, ~\autoref{lemma:singleintersection}, and the set of features $\DE$ is defined as follows:

         \begin{equation}
      \DE = \bigcup_{X_k \in S_g\setminus S_f} {\DE}_{(k)} ,  \label{eq:DE}
  \end{equation}

  with
  \begin{equation}
      {\DE}_{(k)} = \left\{ \text{\ features }d \  \middle\vert \begin{array}{ll}
       S_d = S_f\cup S_{\p{g}}^{\leq k}; \\[6pt]
      \forall  X_m\in S_{f}, X_m(d)=X_m(f);  \\[6pt]
      \forall m < k,\text{ s.t. } X_m\in S_{\p{g}}^{\leq k}\setminus S_f, X_{m}(d) = X_{m}(g); \\[6pt]
      X_k(d) \neq X_k(g) \end{array}\right\},\label{eq:DEk}
  \end{equation}

    where the notation $X_m(d)$ refers to the assignment to variable $X_m$ in
    feature $d$ (likewise for the remaining indices and features), and $S_{\p{g}}^{\leq k} =
    \{X_m\in S_{\p{g}} | m \leq k\}$.

\end{restatable*}

\begin{proof}
  See~\ref{app:lemmas}.
\end{proof}

 Finally, we provide an example for the most complex case:

 \begin{example} \label{ex:difference}

   The example follows ~\autoref{ex:intersection} to obtain the  difference for the pair $(X_0,X_1)$, were we had
     \begin{align*}
        f = < X_0=2, X_1=1, X_2=0, X_5=0 >,
     \end{align*}

       but we will change $g$, by adding an assignment $X_3=0$ to it in order to simplify the example:

       \begin{align*}
        g = < X_0=0, X_1=1, X_2=0, X_3=0, X_4=1 >.
     \end{align*}

     From inspection, we may deduce the following: feature $f$ has $|val(X_3) \times val(X_4)|=9$ FC contexts for the pair, while $g$ has $|val(X_5)|=3$; however, only one of the contexts in $\Xijof{g}$ is in $\Xijof{f}$, namely, $< X_0=0, X_1=1, X_2=0, X_3=0, X_4=1, X_5=0 >$.
    Therefore, we should arrive at a difference set whose FC context set contains 8 contexts.

     Starting from Eq.~\ref{eq:diff:D}, neither of the first two conditions
     hold, by which the difference set D must be defined as a set of features
     following the third condition.

     The sets ${\DE}_{(k)}$ are firstly determined by $S_g\setminus S_f
     =\{X_3,X_4\}$; then, there will be two difference sets corresponding to $k\in \{3,4\}$.
We will analyze each in turn.

     \begin{itemize}
      			       \item $k=3$: the scope of the features will include all variables in $S_f$ and also, in this case, the
	 variable $X_3$, since $S_f \cup S_{{g^{01}}}^{\leq
	 3}=\{X_0,X_1,X_2,X_5\}\cup \{X_3\}= \{X_0,X_1,X_2,X_3,X_5\}$. $X_2$ and $X_5$ will
	 match their values in $f$ (as stated in the second line in
	 Eq.~\ref{eq:DEk}), while $X_3$ must take values that differ from
	 $X_3(g)=0$ (see the last line in Eq.~\ref{eq:DEk}).
	 This implies that we must generate two different features:
	 \begin{align*}
	   d_1 &= <X_0=2, X_1=1, X_2=0,X_3=1,X_5=0>, and \\
	   d_2 &= <X_0=2, X_1=1, X_2=0,X_3=2,X_5=0>.
	   %\label{eq:difffeaturesexample}
	 \end{align*}
        \item $k=4$: in this case we will also generate two features, but they
	  have a scope containing $\{X_0,X_1,X_2,X_3,X_4,X_5\}$, and there is one
	  variable
	  $X_m\in S_{{g^{01}}}^{\leq 4}$ such that $m < 4$; namely, $X_3$. This variable now takes
	  the same value as $X_3(g)$ (see third line in Eq.~\ref{eq:DEk}), while
	  $X_4$ takes values that differ from $X_4(g)=1$:
	 \begin{align*}
	   d_3 &= <X_0=2, X_1=1, X_2=0,X_3=0,X_4=0,X_5=0>, \\
	   d_4 &= <X_0=2, X_1=1, X_2=0,X_3=0,X_4=2,X_5=0>.
	   %\label{eq:difffeaturesexample2}
	 \end{align*}

     \end{itemize}
	 Lastly, the set $\DE$ is the union of both $k$ sets which is simply

	 \begin{equation}
	   \DE = {\DE}_{(3)} \cup {\DE}_{(4)} = \{ d_1,d_2,d_3,d_4  \}.
	 	 \end{equation}

	 This produces the total of 8 FC contexts in the difference.
	 On the one hand, $d_1$ and $d_2$ produce 6 out of the 9 contexts of $\Xijof{f}$, which correspond to all configurations of $X_4$ for two configurations of $X_3$ ($X_3=1$ and $X_3=2$).
	 On the other hand, $d_3$ and $d_4$ add the remaining configurations for $X_3=0$ (which should be in the difference because they belong to $\Xijof{f}$) but excluding the configuration $X_3=0,X_4=1$, which is precisely the one that is in $\Xijof{g}$ and should be absent in the difference.

 \end{example}

\section{Metric}\label{sec:distance}

In this section we propose a measure based on the FC confusion matrix counts $FP$ and $FN$ for comparing two log-linear model structures $\MP$ and $\MQ$;  and prove it  is a \emph{distance measure} or \emph{metric} by proving it satisfies all four properties [Chapter 3, \cite{aliprantis2006}].

The comparison measure is formally defined as:

\begin{equation}
      \dist{\MP}{\MQ}= FP + FN
  \label{eq:distance}
\end{equation}
where $FP$ and $FN$ correspond to the total count of false positives and false negatives, respectively, of the FC confusion matrix defined in Section~\ref{sec:structureComparisonLL}.
This measure is the complement of the unnormalized FC accuracy $TP+TN$, in that when one equals zero, the other takes its maximum value corresponding to the cardinality of the FC triplet set of Eq.~\ref{eq:tripletesP}.

The proof is formalized in the following theorem:

\begin{theorem}\label{T:distance}

    Given two log-linear model structures  $\MP$ and $\MQ$, the measure
$\dist{\MP}{\MQ}= FP + FN$
is a metric, i.e., it satisfies all four properties: \emph{nonnegativity},
\emph{discrimination} (also known as \emph{identity of the indiscernibles}),
\emph{symmetry}, and \emph{triangle inequality} (also known as
\emph{subadditivity}).

\end{theorem}

\begin{proof}
We prove each property separately:

\begin{enumerate}[i]
  \item \emph{Nonnegativity} holds trivially because the measure is defined as the cardinality of a set, which is always nonnegative.
  \item \emph{Discrimination} can be stated as

    \begin{equation*}
    \dist{\MP}{\MQ}=0 \iff \MP = \MQ.
      \label{eq:identity}
    \end{equation*}
where the equality of the two models in the r.h.s., understood as the equality of their dependencies and independencies, is a shorthand for the equality of their \emph{complete} dependency models, i.e., $\DC{\MP} = \DC{\MQ}$.
We begin by noting that, according to the confusion matrix over the complete dependency models of Eqs.~\ref{eq:TPcomplete}-\ref{eq:TNcomplete}, when the two complete dependency models are equal, both the false positives and false negatives are zero.
Thus, if we define the complete measure as their sum, i.e.,
    \begin{equation*}
            \distcomplete{\MP}{\MQ}= FP_C + FN_C,
        \label{eq:distancecomplete}
    \end{equation*}
then $\MP = \MQ \iff \distcomplete{\MP}{\MQ}=0 $.

Starting by the right-to-left implication, we have that $\MP = \MQ$: this implies that $\distcomplete{\MP}{\MQ} = 0$ and thus $FP_C=FN_C=0$.
If we now consider that, for any model $H$, $\DP{H} \subset \DC{H}$ (in particular for $\MP$ and $\MQ$), then it must always be the case that $FP(\MP, \MQ) \leq FP_C$ and $FN \leq FN_C$.
Then, if $FP_C=FN_C=0$, it follows that $FP=FN=0$ and thus $\dist{\MP}{\MQ}=0$.

The left-to-right implication requires a more detailed analysis.
We provide a proof by transitivity, by showing that

    \begin{equation}
      \dist{\MP}{\MQ} = 0 \implies \distcomplete{\MP}{\MQ} = 0,
      \label{eq:identityimplication}
    \end{equation}

which by in turn implies $\MP = \MQ$.

We proceed by proving the contrapositive of Eq.~\ref{eq:identityimplication}, assuming that $\distcomplete{\MP}{\MQ} > 0$ and showing that this results in $\dist{\MP}{\MQ} > 0$.
Distances higher than zero imply that there is at least one mismatch in the corresponding dependency models of $\MP$ and $\MQ$.
Thus, proving the contrapositive requires proving that any given discrepancy in the complete model produces a discrepancy in the FC model.
This is trivial for discrepancies coming from (in)dependencies with FC conditioning sets; so we will consider an arbitrary discrepancy that comes from a non-FC conditioning set in the complete model.
We assume that the independence $\ci{X_i}{X_j}{\uu,X_W}_{\MP}$ holds for model $\MP$ but does not hold in $\MQ$, i.e., $\cd{X_i}{X_j}{\uu,X_W}_{\MQ}$; and prove that this produces a discrepancy in their corresponding FC dependency models, i.e., there exist some pair of variables $(X_k, X_l), k,l\in V$ and some FC context $\z \in \Xkl $, such that $\ci{X_k}{X_l}{\z}_{\MP}$ but $\cd{X_k}{X_l}{\z}_{\MQ}$.
In short,

    \begin{eqnarray}
        \ci{X_i}{X_j}{\uu,X_W}_{\MP} \land \cd{X_i}{X_j}{\uu,X_W}_{\MQ} \nonumber \\
        \implies \exists \z \in \Xkl,   \ci{X_k}{X_l}{\z}_{\MP} \land \cd{X_k}{X_l}{\z}_{\MQ}. \label{eq:contradiction}
    \end{eqnarray}

The proof involves the concept of paths in instantiated graphs, and a concept proposed in [Theorem 1, \cite{hojsgaard2004statistical}] reproduced below:

\begin{theorem}\label{t:hojsgaard}
    [\cite{hojsgaard2004statistical}]{
  Let $\graph{\uu}{}$ be the graph instantiated by $\uu$. Then $\ci{X_i}{X_j}{\uu,X_W}$ if and only if~$W$ separates $i$ and $j$ in $\graph{\uu}{}$.
}\label{aux:global}
   \end{theorem}

By the dependence $\cd{X_i}{X_j}{\uu,X_W}_{\MQ}$ and \autoref{t:hojsgaard}, we have that, in the instantiated graph \graph{\uu}{\MQ} (Figure~\ref{fig:graph:D2}), there is a path from $i$ to $j$ satisfying that none of its nodes are in $W$.
We denote the sequence corresponding to this path by $\secuencia$.

Now, by the independence $\ci{X_i}{X_j}{\uu,X_W}_{\MQ}$ and \autoref{t:hojsgaard}, in the instantiated graph \graph{\uu}{\MP} (Figure~\ref{fig:graph:D1}) there is no path from $i$ to $j$ that is disconnected from $W$, i.e., a path satisfying that none of its nodes are in $W$.
In particular, the sequence of nodes $\secuencia$ cannot be a path in \MP.
This implies that at least one pair of subsequent nodes in $\secuencia$ has no edge between them in \graph{\uu}{\MP}, while it does have an edge between them in \graph{\uu}{\MQ}. We have denoted this edge by indices $k$ and $l$, as shown in the figures.

By simple inspection of the figure one can infer that $W$ separates $k$ and $l$ in \graph{\uu}{\MP}.
From the right-to-left implication of \autoref{t:hojsgaard}, we have that $\ci{X_k}{X_l}{\uu,X_W}_{\MP}$, and then the axiom of Strong Union\footnote{For a model $H$, the Strong Union axiom is satisfied if: $\ci{X_A}{X_B}{X_Z}_H \implies \ci{X_A}{X_B}{X_{Z\cup W}}_H$}, implies that $\ci{X_k}{X_l}{\uu,X_{W'}}_{\MP}$,  with $W' = V\setminus\left( \{k,l\}\cup U \right)$.

Before concluding, we will prove a similar equivalence for the dependence case.
For that, we note that a direct edge between $k$ and $l$ in the graph \graph{\uu}{\MP} implies that no set of nodes can separate them, not even the set containing all other nodes in the graph, which is precisely $W'$.
Applying the contrapositive of \autoref{t:hojsgaard}, we have that $\cd{X_k}{X_l}{\uu,X_{W'}}_{\MQ}$.
To conclude, we recall the following equivalence of context-specific independencies, presented in Section~\ref{sec:independence}, Eqs.~\ref{eq:implicationindependence} and \ref{eq:implicationdependence},  applied over the above (in)dependencies for $k$ and $l$ over the conditioning set $\{\uu, X_{W'}\}$:

    \begin{eqnarray*}
      \ci{X_i}{X_j}{\uu,X_{W'}} &\equiv \forall x_{W'} \in \val{X_{W'}}, \ci{X_i}{X_j}{\uu,x_{W'}}, \\
      \cd{X_i}{X_j}{\uu,X_{W'}} &\equiv \exists x_{W'} \in \val{X_{W'}}, \cd{X_i}{X_j}{\uu,x_{W'}}.
          \end{eqnarray*}

Let us denote as $\z$ the particular context $\uu, x_{W'}$ for which the second equivalence holds.
Since the first equivalence holds for any such context, it holds for $\z$ as well.
We thus have that $\ci{X_k}{X_l}{\z}_{\MP}$ holds for model $\MP$, and $\cd{X_k}{X_l}{\z}_{\MQ}$ holds for model $\MQ$, matching the r.h.s. of Eq.~\ref{eq:contradiction}, and thus the left-to-right part of the discrimination property.

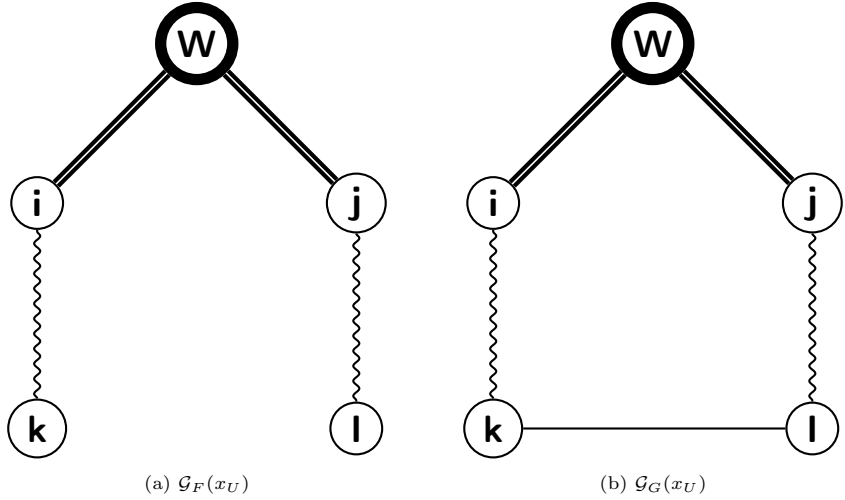
\begin{figure}
\begin{subfigure}{.5\textwidth}
  \centering

\begin{tikzpicture}[node distance=3cm, thick,
                    main node/.style={circle,draw,font=\sffamily\Large\bfseries}]

    \node[main node,line width=1.5mm] (5) {W};
    \node[main node] (1) [below left of=5] {i};
  \node[main node] (2) [below right of=5] {j};
  \node[main node] (3) [below of=1] {k};
  \node[main node] (4) [below of=2] {l};

  \path
    (1) edge [ultra thick, double] node {} (5)
    (2) edge [ultra thick, double, left] node {} (5);

  \tikzset{decoration={snake,amplitude=.4mm,segment length=2mm,
           post length=0mm,pre length=0mm}}
  \draw[decorate] (3) -- (1);
  \draw[decorate] (2) -- (4);

\end{tikzpicture}

\caption{\graph{\uu}{\MP}}
  \label{fig:graph:D1}
\end{subfigure}\begin{subfigure}{.5\textwidth}
  \centering

\begin{tikzpicture}[node distance=3cm, thick,
                    main node/.style={circle,draw,font=\sffamily\Large\bfseries}]

    \node[main node,line width=1.5mm] (5) {W};
    \node[main node] (1) [below left of=5] {i};
  \node[main node] (2) [below right of=5] {j};
  \node[main node] (3) [below of=1] {k};
  \node[main node] (4) [below of=2] {l};

  \path
    (1) edge [ultra thick, double] node {} (5)
    (2) edge [ultra thick,double] node {} (5)
          (4) edge node [left] {} (3);

  \tikzset{decoration={snake,amplitude=.4mm,segment length=2mm,
           post length=0mm,pre length=0mm}}
  \draw[decorate] (3) -- (1);
  \draw[decorate] (2) -- (4);

\end{tikzpicture}

\caption{\graph{\uu}{\MQ}}
  \label{fig:graph:D2}
\end{subfigure}
\caption{Two possible instantiated graphs in $\uu$~for $\ci{X_i}{X_j}{X_U,x_W}_{\MP}$
and $\cd{X_i}{X_j}{X_U,x_W}_{\MQ}$. $W$ represents a set of nodes and the thick
double lines represent edges between nodes $i$ and $j$ to all variables in $W$.
In this example, in $\graph{\uu}{\MQ}$ there is a sequence of nodes $\secuencia$, represented by waved lines, which constitutes a path from $i$ to $j$ that is disconnected from $W$, but for some pair of nodes $k,l \in \secuencia$, there is not an edge between them in $\graph{\uu}{\MP}$.} 

\end{figure}

  \item \emph{Symmetry} follows by the symmetry of $FP$ and
    $FN$, which can be inferred trivially from the commutativity of the logical AND operator in their definitions in Eqs~\ref{eq:FNdepmodel} and \ref{eq:FPdepmodel}, respectively.

  \item \emph{Triangle inequality.} We must prove that, for any three context-specific dependency models~\MA,~\MB~and~\MC,
    \begin{equation}
      \dist{\MA}{\MC} \leq \dist{\MA}{\MB} + \dist{\MB}{\MC}.
      \label{eq:triangle}
    \end{equation}

     To simplify the proof, we rewrite the measure as a sum of terms over the set $\IP$ of all FC triplets as defined in Eq.~\ref{eq:tripletesP},

\begin{equation}
  \dist{\MP}{\MQ}=\sum_{\Ind\in \IP}\disti{\MP}{\MQ},
  \label{eq:trianglesum}
\end{equation}

where

\begin{equation*}
  \disti{\MP}{\MQ}=
     \begin{cases}
       1 &\quad \text{if } \Ind_{\MP} \neq \Ind_{\MQ} \\
       0 &\quad \text{if } \Ind_{\MP} = \Ind_{\MQ}.
     \end{cases}
  \label{eq:distancei}
\end{equation*}

The interpretation is that each of these terms is an indicator of whether the
independence assertion $\Ind$ has the same value in both models or not, contributing $0$ or $1$, respectively. This would clearly correspond to the count of all mistmatches of the confusion matrix, which equals the sum of $FP$ and $FN$.

We can re-express Eq.~\ref{eq:triangle} using Eq.~\ref{eq:trianglesum} as

\begin{equation}
      \sum_{\Ind\in \IP}\disti{\MA}{\MC} \leq \sum_{\Ind\in \IP}\disti{\MA}{\MB} +
      \sum_{\Ind\in \IP}\disti{\MB}{\MC}.
  \label{eq:trianglesumi}
\end{equation}

Considering that the sum of two or more valid inequalities side by side is also
a valid inequality,
it is sufficient to prove that,
for any three models \MA, \MB, \MC~and all triplets $\Ind \in \IP$,

      \begin{equation}
      \disti{\MA}{\MC} \leq \disti{\MA}{\MB} + \disti{\MB}{\MC}.
      \label{eq:trianglei}
      \end{equation}

We consider the two possible cases for the l.h.s.: either $\disti{\MA}{\MC}=0$ or $\disti{\MA}{\MC}=1$.
On the one hand, if $\disti{\MA}{\MC}=0$, then the property is trivially satisfied because the
r.h.s. will always be nonnegative.
On the other hand, given $\disti{\MA}{\MC}=1$, then the only possible combination of
values that violates the property is

\begin{center}
\begin{tabular}[h]{ccc}
  \disti{\MA}{\MC} & \disti{\MA}{\MB} & \disti{\MB}{\MC} \\\hline
  1 & 0 & 0
 \end{tabular}
\end{center}

However, this combination is not possible.
If $\disti{\MA}{\MB}=0$ and $\disti{\MB}{\MC}=0$, then it holds that
$\Ind_{\MA}=\Ind_{\MB}$ and $\Ind_{\MB}=\Ind_{\MC}$.
Then, from transitivity, $\Ind_{\MA}=\Ind_{\MC}$;  which implies $\disti{\MA}{\MB}=0$.
Therefore, the combination above will never occur.

Since all possible cases satisfy Eq.~\ref{eq:trianglei} and the sum
on both sides for all $\Ind \in \IP$ (Eq.~\ref{eq:trianglesumi}) maintains the inequality, we can now conclude
that the property in Eq.~\ref{eq:triangle} is satisfied. \end{enumerate}
\end{proof}

\section{Summary of the development and computation of the proposed method}\label{sec:summary}

In the previous sections, we have presented the complete rationale for the definition and computation of our contribution, including proofs of the correctness of each step.
Due to the length and complexity of the thorough exposition, we now provide a summary of the main steps in Table~\ref{table:summary}, which may serve as a guide for understanding how the different parts of the development of the method are related and integrated from a broader perspective.

%TODO tablita

\begin{table}
\centering
\small
\begin{tabular}{llp{7cm}}
\hline
Section                                                  & \begin{tabular}[c]{@{}l@{}}Equations or\\references\end{tabular}         & Description                                                                                                 \\\hline
\multirow{2}{*}{Section~\ref{sec:structureComparisonLL}} & \ref{eq:TPcomplete}-\ref{eq:TNcomplete}   & Confusion matrix (CM) based on complete dependency models ($\mathcal{D}_C(\cdot)$)                                          \\
                                                         & \ref{eq:TPdepmodel}-\ref{eq:TNdepmodel}   & CM based on reduced dependency models                                                                       \\\hline
\multirow{4}{*}{Section~\ref{sec:setform}}               & \ref{eq:TPij}-\ref{eq:TNij}               & Alternative and equivalent definition of CM based on FC contexts ($\Xij$)                                   \\
                                                         & \ref{eq:TPa}-\ref{eq:TNa}                 & Equivalent definition with set operations                                                                   \\\hline
\multirow{4}{*}{Section~\ref{sec:theorem1}}              & \ref{eq:T:method:TP}-\ref{eq:T:method:TN} & Efficient computation of the CM based on partition models (Theorem~\ref{T:method})                                       \\
                                                         & \ref{eq:HTP}, \ref{eq:HFN}                & Definitions of the partition models $H^{TP}, H^{FP}$ and $H^{FN}$ (\emph{H sets}) (Lemmas~\ref{lemma:HTP} and~\ref{lemma:HFN})                       \\\hline
\ref{app:lemmas}                                         & \ref{eq:intersection1}, \ref{eq:diff:D}                 & Efficient computation for obtaining the \emph{H sets} with intersection and difference (Lemmas ~\ref{lemma:singleintersection} and ~\ref{lemma:singlediff})                  \\\hline
\multirow{2}{*}{Section~\ref{sec:preliminaries}}         & Algorithm~\ref{alg:partition}             & Transformation of the \emph{H sets} into partition models (as per Definition~\ref{def:partition}) \\
                                                         & \ref{eq:cardfeature}                      & Efficient computation of cardinality of FC context sets that allows for the computation of the partition models \\\hline
Section~\ref{sec:distance}                               & \ref{eq:distance}                         & Computation of the distance between structures of two log-linear models \\\hline
&&
\end{tabular}
\caption{Summary of the main equivalences and procedures of the method presented in this work.}
\label{table:summary}
\end{table}

In addition, considering that the main exposition of this work has focused on the theoretical aspects, we now provide a simple list of steps for using the method. 
This guide can be used as the basis for an implementation that obtains the confusion matrix and distance between two log-linear models.
It is written as a high-level procedure, since the details of implementation can vary widely depending on the computational representation of features, among other details, which are strongly dependent on the chosen programming language and libraries.
Such details can be solved in a straightforward manner.
Finally, it is important to note that, although the \emph{H sets} are a necessary step for explaining the method (and they perhaps constitute the most involved part), the actual computation of these sets is simpler than it could appear, and is ultimately based on two operations over single features.

The procedure for obtaining a comparison can thus be summarized in the following steps:
\begin{enumerate}
  \item Iterate over all pairs of variables $(X_i,X_j)$:
    \begin{enumerate}
    \item For each pair, build sets $H^{TP}, H^{FP}$ and $H^{FN}$ by using Eqs. \ref{eq:intersection1} and \ref{eq:diff:D}.
    \item Generate the partition models $P^{TP}$, $P^{FN}$ and $P^{FP}$ with Algorithm~\ref{alg:partition}.
    \item Compute the cardinalities of the partition models with Eq.~\ref{eq:cardfeature} to obtain $TP_{ij}, FP_{ij},$ and $FN_{ij}$.
  \end{enumerate}
  \item Sum the cardinalities of each pair to obtain TP, FP and FN.
  \item Compute TN according to Eq.~\ref{eq:T:method:TN} to complete the confusion matrix.
  \item Sum $FP+FN$ to obtain the distance.
\end{enumerate}

% SECTION Comparison of KL divergence with proposed measure

% Appendix A

\section{Comparison between the proposed metric and
Kullback--Leibler divergence} %
\label{app:kl}

%----------------------------------------------------------------------------------------

In statistics, divergences are functions that measure the similarity of probability distributions, the first of which was introduced by \cite{BHATTA43}.
However, perhaps the most popular of these functions is the \emph{Kullback--Leibler divergence}~\cite{kullback1951information,cover2012elements}.
At present, divergences are still in use, proven useful for statistical comparisons of probabilistic models \cite{gardner2018,venturini2015statistical}.
For model comparisons they have been used mainly in the process of disproving the null-hypothesis that one model differs from the other when the divergence equals zero.

The KL-divergence, also known as \emph{relative entropy}, is a measure of the similarity between two distributions, $p(X)$ and $q(X)$, and it has been used to compare log-linear models~\cite{ederaSchluterBromberg14,edera2014grow}.
It is defined as

\begin{equation*}
  D_{KL}(p||q) = \sum_{x \in \X}{p(x) log \frac{p(x)}{q(x)}},
\end{equation*}

following the conventions $0\ log \frac{0}{q(x)}=0$ and
$p(x)\ log \frac{p(x)}{0}=\infty$.

This divergence can be interpreted as the information lost when using $q(x)$ as an approximation of $p(x)$.
Alternatively, it can be thought of as an approximation of the distance between the two distributions, given that it satisfies the intuition that the cost of approximating $p(x)$ with $q(x)$ is lower when they are similar, being zero only when they are identical and positive in any other case.
Because of this, the KL-divergence satisfies two properties of a metric:
non-negativity and discrimination.\footnote{Note that discrimination is satisfied only in relation to the complete distributions $p(X)$ and $q(X)$, not their structures.}
Nevertheless, it is not a metric, as it does not satisfy symmetry nor the triangle inequality; as a consequence, it provides no sense of scale for the differences.

%TODO incorporar:

In spite of this, the KL-divergence is still useful as a measure of quality of a distribution learned from data sampled from a known distribution.
The biggest disadvantage is that this method is not a direct indicator of the similarity of two structures.
As it involves the parameters of the model, it can obscure false positives when
the parameters cancel spurious interaction terms that are present in the structure.
This translates as an obstacle when evaluating structure learning algorithms for possible tendencies to introduce false positives.

In the next example we have used a synthetic model with a well-defined structure, and randomly generated a great number of structures which possess varying numbers of either false positives or false negatives. We aim to illustrate the shortcomings of a measure that compares the similarity of probability distributions, including log-linear models, in contrast to our proposed method.
For this we learn the parameters for the random structures using data sampled from the original model, and over these models (random structure plus learned parameters) we compute the KL-divergence.
Then, the values obtained in this manner are visualized against the percentage of errors (either FP or FN) computed by our metric.

\subsection{Methodology}

We proceed, first, by using a synthetic model $M_O$ (the ``original'' model), defined over a domain of 6 variables: $\{X_0,\ldots,X_5\}$.
We selected this model due to its presence in related work (see ~\cite{ederaSchluterBromberg14}~and~\cite{edera2014grow}), and its suitability for introducing modifications in the structure that add a considerable number of false positives and false negatives.

The original model is represented as two instantiated graphs in Figure~\ref{fig:klgrafo}.
This representation is useful to show its two local structures: a saturated model (complete subgraph) for one context (given by one variable), and an independent model (empty subgraph).
In this way, the global structure contains a number of context-specific independencies.
The associated dependency model is:
  \begin{align*}
  \mathcal{D}_C(M) = \left\{ \cd{X_i}{X_j}{X_0=0}\right\}\cup \left\{\ci{X_i}{X_j}{X_0=1}\right\}; \forall i\neq j \in \{1,\cdots,5\}.
  \end{align*}
Parameters were generated by using different weights for the features that guarantee strong interactions.
Their design is explained in detail in~[Appendix B, \cite{ederaSchluterBromberg14}].
We have used the models generated for the experiments in the cited work, with the permission of its authors.

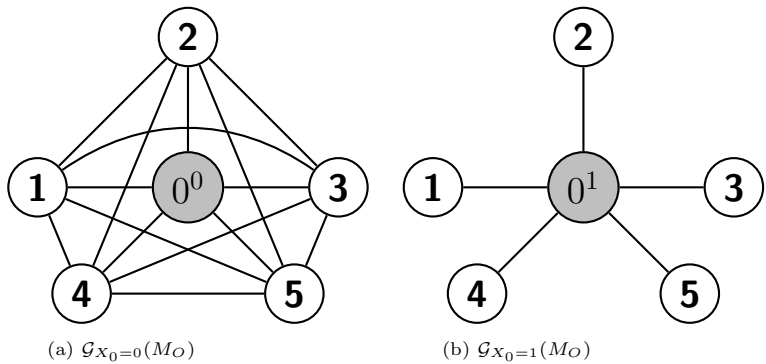
\begin{figure}[H]%{0.6\textwidth}
    \centering
\begin{subfigure}{.25\textwidth}
 % \centering
\begin{tikzpicture}[node distance=2cm, thick,
                    main node/.style={circle,draw,font=\sffamily\Large\bfseries}]

  \node[main node, fill={rgb:black,1;white,3}] (1) {$0^0$};
  \node[main node] (6) [below right of=1] {5};
  \node[main node] (2) [left of=1] {1};
  \node[main node] (3) [above of=1] {2};
  \node[main node] (4) [right of=1] {3};
  \node[main node] (5) [below left of=1] {4};

  \path
    (1) edge node {} (3)
    (1) edge node {} (2)
    (1) edge node {} (4)
    (1) edge node {} (5)
    (1) edge node {} (6)
    (2) edge node {} (3)
    (2) edge [bend right=-35] node {} (4)
    (2) edge node {} (5)
    (2) edge node {} (6)
    (3) edge node {} (4)
    (3) edge node {} (5)
    (3) edge node {} (6)
    (4) edge node {} (5)
    (4) edge node {} (6)
    (5) edge node {} (6);

\end{tikzpicture}
\caption{$\graph{M_O}{X_0=0}$}
  \label{fig:klgrafo:a}
\end{subfigure}
\quad\quad\quad\quad
\quad\quad
\begin{subfigure}{.25\textwidth}
 % \centering

\begin{tikzpicture}[node distance=2cm, thick,
                    main node/.style={circle,draw,font=\sffamily\Large\bfseries}]

  \node[main node, fill={rgb:black,1;white,3}] (1) {$0^1$};
  \node[main node] (2) [left of=1] {1};
  \node[main node] (3) [above of=1] {2};
  \node[main node] (4) [right of=1] {3};
  \node[main node] (5) [below left of=1] {4};
  \node[main node] (6) [below right of=1] {5};

  \path
    (1) edge node {} (3)
    (1) edge node {} (2)
    (1) edge node {} (4)
    (1) edge node {} (5)
    (1) edge node {} (6);

\end{tikzpicture}
\caption{$\graph{M_O}{X_0=1}$}
  \label{fig:klgrafo:b}
\end{subfigure}%

\caption{Instantiated graphs associated to the synthetic model $M_O$. Grey nodes
correspond to the configurations $X_0=0$ (labeled as $0^0$) and $X_0=1$ (labeled as $0^1$), respectively.}\label{fig:klgrafo}
\end{figure}

The comparison is divided in two parts.
On the one hand, we will show the evaluation of both measures over a set of structures $\mathcal{M}_{FP}$ that only have false positives with respect to $M_O$ and, on the other hand, over structures that only have false negatives with respect to $M_O$, $\mathcal{M}_{FN}$.

Once the structures were generated, the next step was to compute our log-linear structure distance measure, $d$, directly between the synthetic structure $M_O$ and each randomly generated structure.
This consists in obtaining, for each structure $M_{FN} \in \mathcal{M}_{FN}$, the value $\dist{M_O}{M_{FN}}$, and for each structure $M_{FP}\in \mathcal{M}_{FP}$, $\dist{M_O}{M_{FP}}$.

The computation of the KL-divergence required three additional steps.
First, it was necessary to generate datasets of varying sizes from the synthetic model $M_O$.
Specifically, the number of datapoints used was $s \in \{ 50,100,1000,10000 \}$, and the sampling method was Gibbs sampling using the open source software package \emph{Libra toolkit}.
Second, we performed parameter learning on these datasets for all the random structures $\mathcal{M}_{FN}$ and $\mathcal{M}_{FP}$, in order to obtain the complete distribution estimated with each dataset.
Lastly, we computed the KL-divergence between $M_O$ and each model obtained in the second step, and averaged the values over the 10 sets of parameters of the model corresponding to each dataset size.

\subsection{Results}

Results are visualized in Figures~\ref{fig:kl1}~to~\ref{fig:kl4}.
An interactive version of these results is available at \url{https://jstrappa.shinyapps.io/llmc}, which provides more visualization options.
For each figure, results for $\mathcal{M}_{FN}$ and $\mathcal{M}_{FP}$ are plotted separately.
The graphs also show a comparison of values obtained for different sample sizes, which were used by the KL-divergence to compute the similarity of the distributions.
In each graph, the x-axis represents the percentage of errors measured by our method (the number of errors relative to the maximum possible number of errors w.r.t. the original structure).
The y-axis shows the value of the KL-divergence for the corresponding structure.
Each dot is a different structure.
The standard deviation is the one obtained with parameter learning, by using the data generated from model $M_O$ with 10 different sets of parameters.

% Figuras KL
\newcommand{\scale}{0.4}

\begin{figure}[!hp]
  \centering
  \includegraphics[scale=\scale]{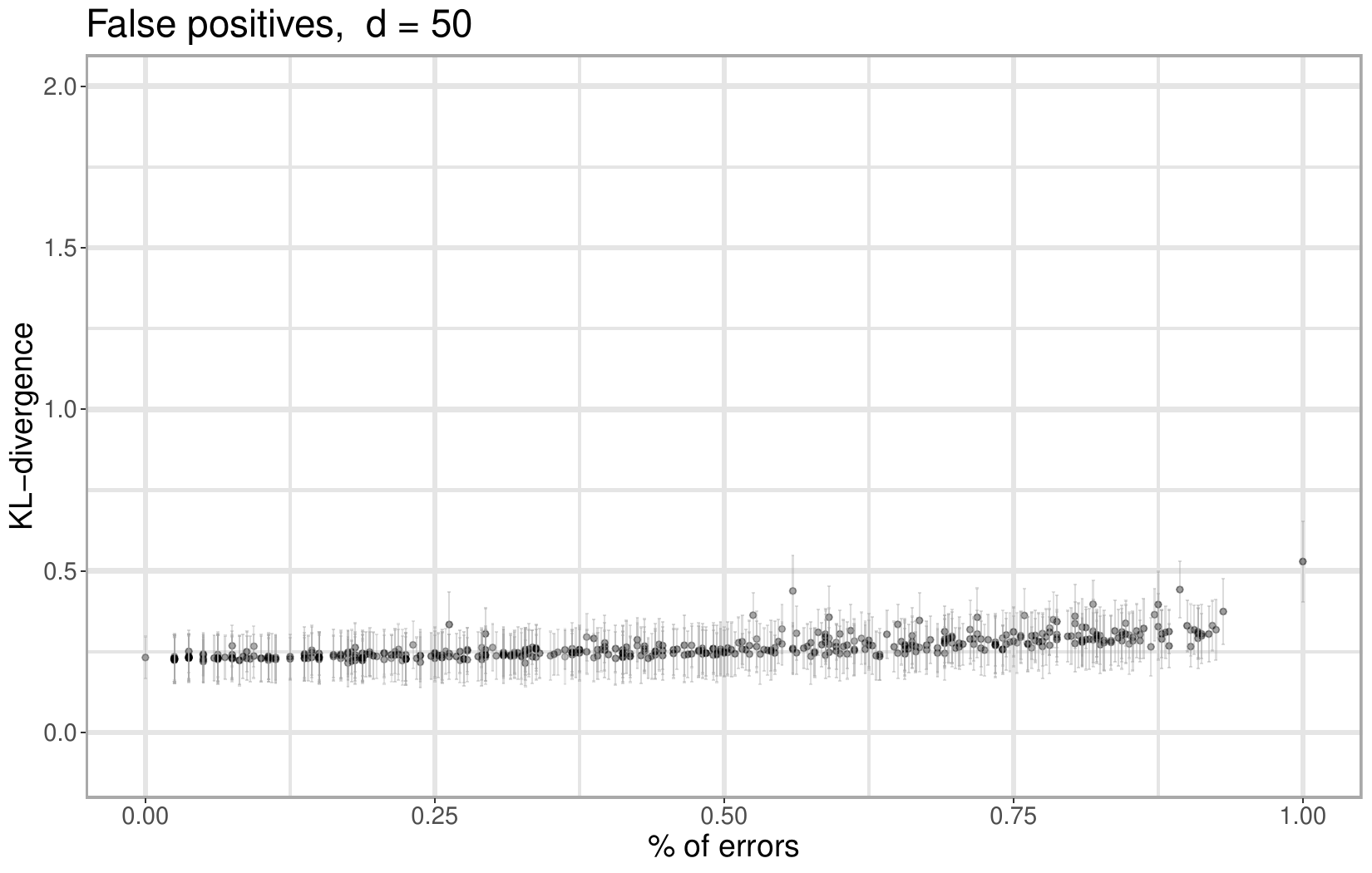}

  \vspace{2cm}

  \includegraphics[scale=\scale]{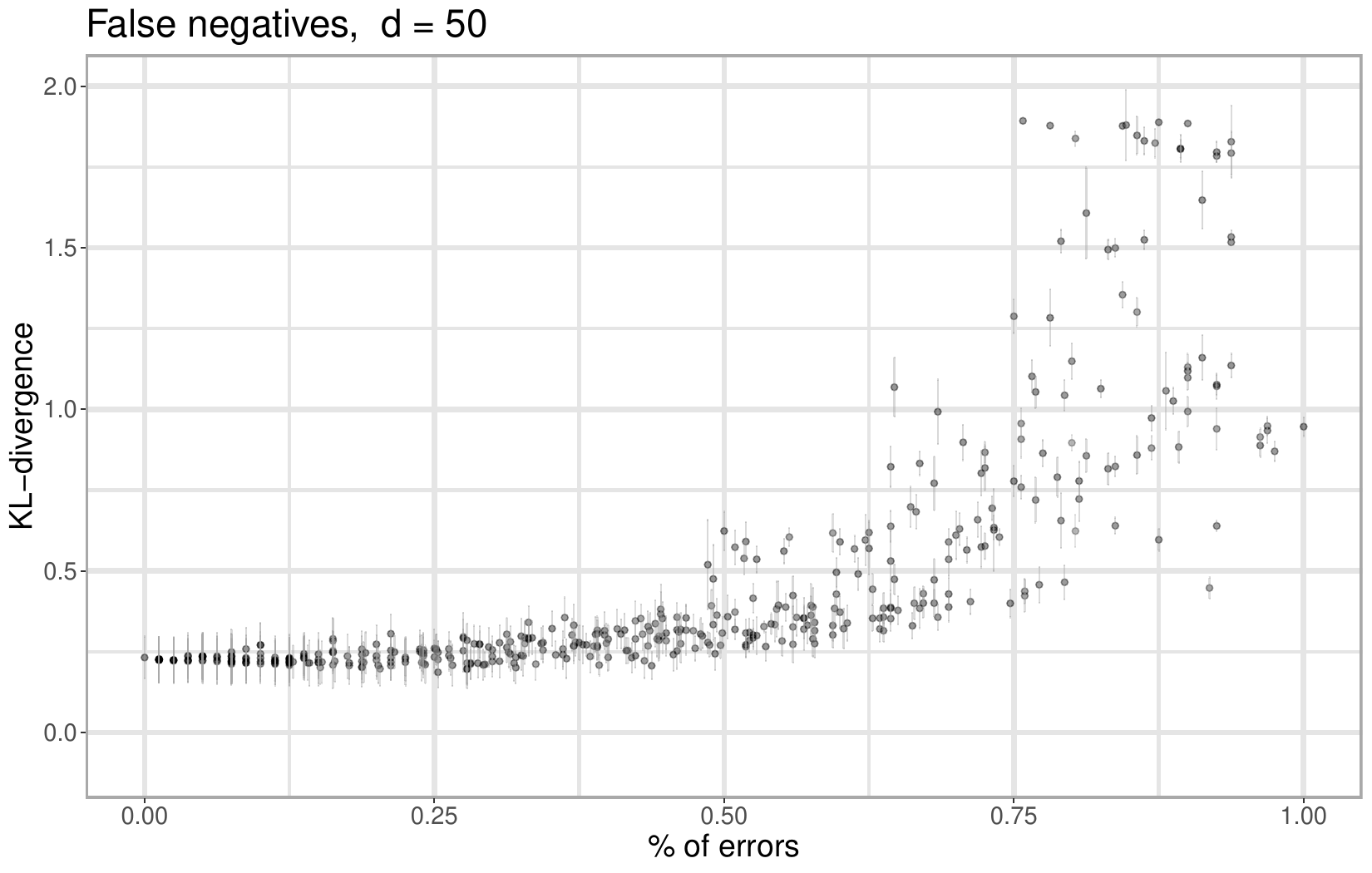}
  \caption{Error comparison for the proposed metric (x-axis) vs KL-divergence (y-axis) for synthetic datasets of 6 variables, and $s=50$.
    Top: \% of false positives.
  Bottom: \% of false negatives.
  Each dot in the graphs is a single structure.
  }
  \label{fig:kl1}
\end{figure}
\begin{figure}[hp]
  \centering

  \includegraphics[scale=\scale]{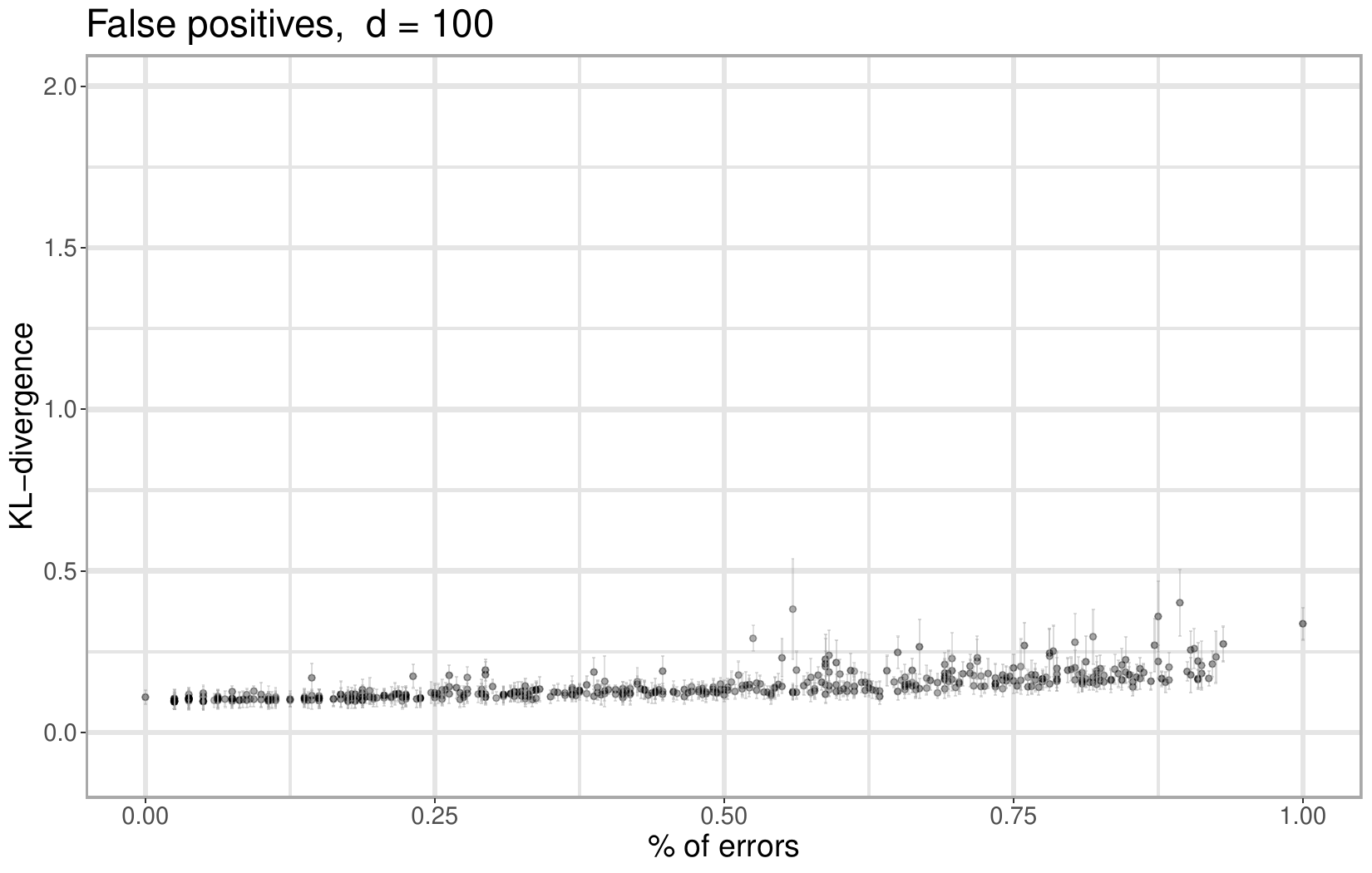}

  \vspace{2cm}

  \includegraphics[scale=\scale]{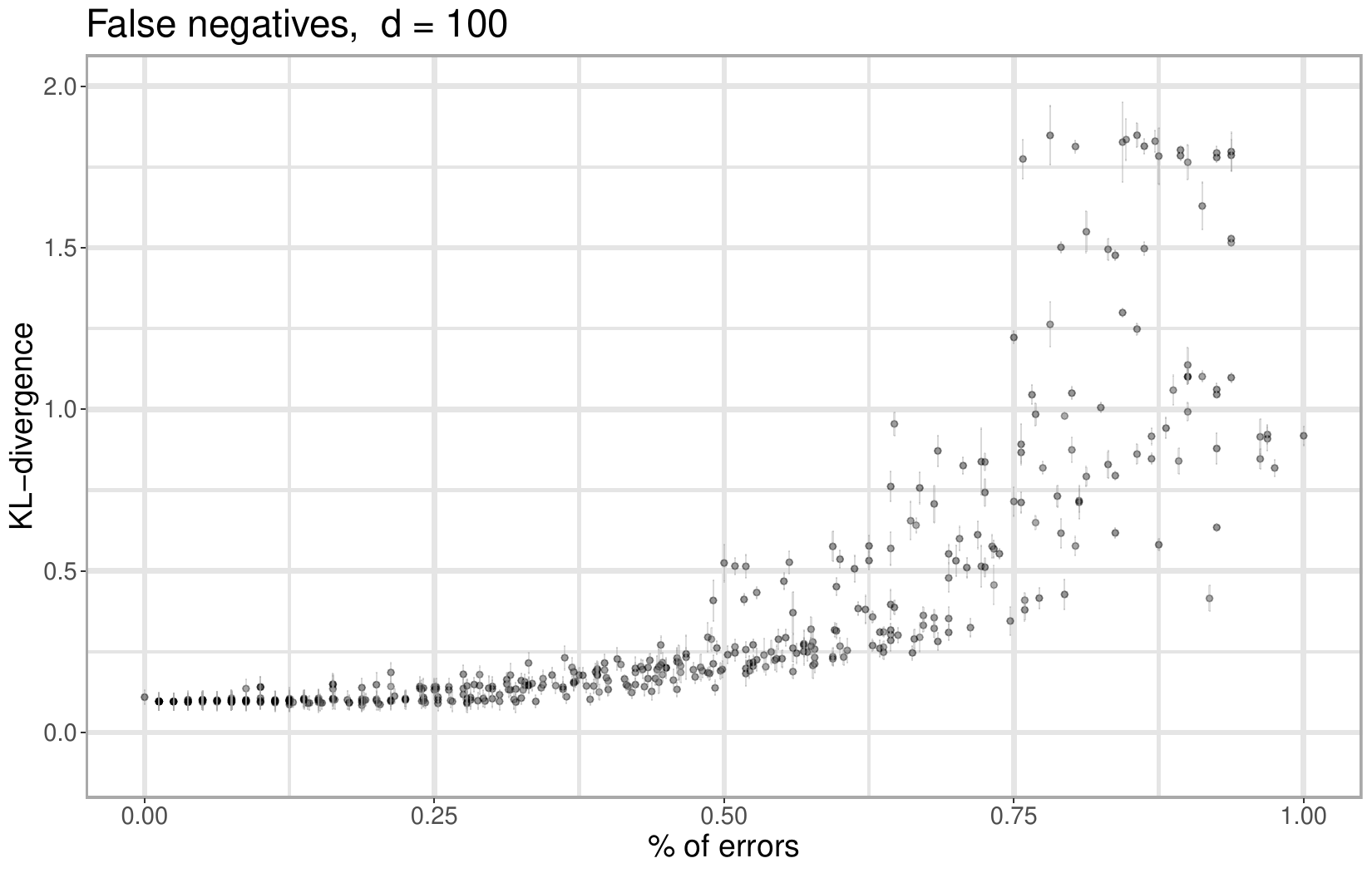}
 \caption{Error comparison for the proposed metric (x-axis) vs KL-divergence (y-axis) for synthetic datasets of 6 variables, and $s=100$.
    Top: \% of false positives.
  Bottom: \% of false negatives.
  Each dot in the graphs is a single structure.
  }
  \label{fig:kl2}
\end{figure}
\begin{figure}[hp]
  \centering
  \includegraphics[scale=\scale]{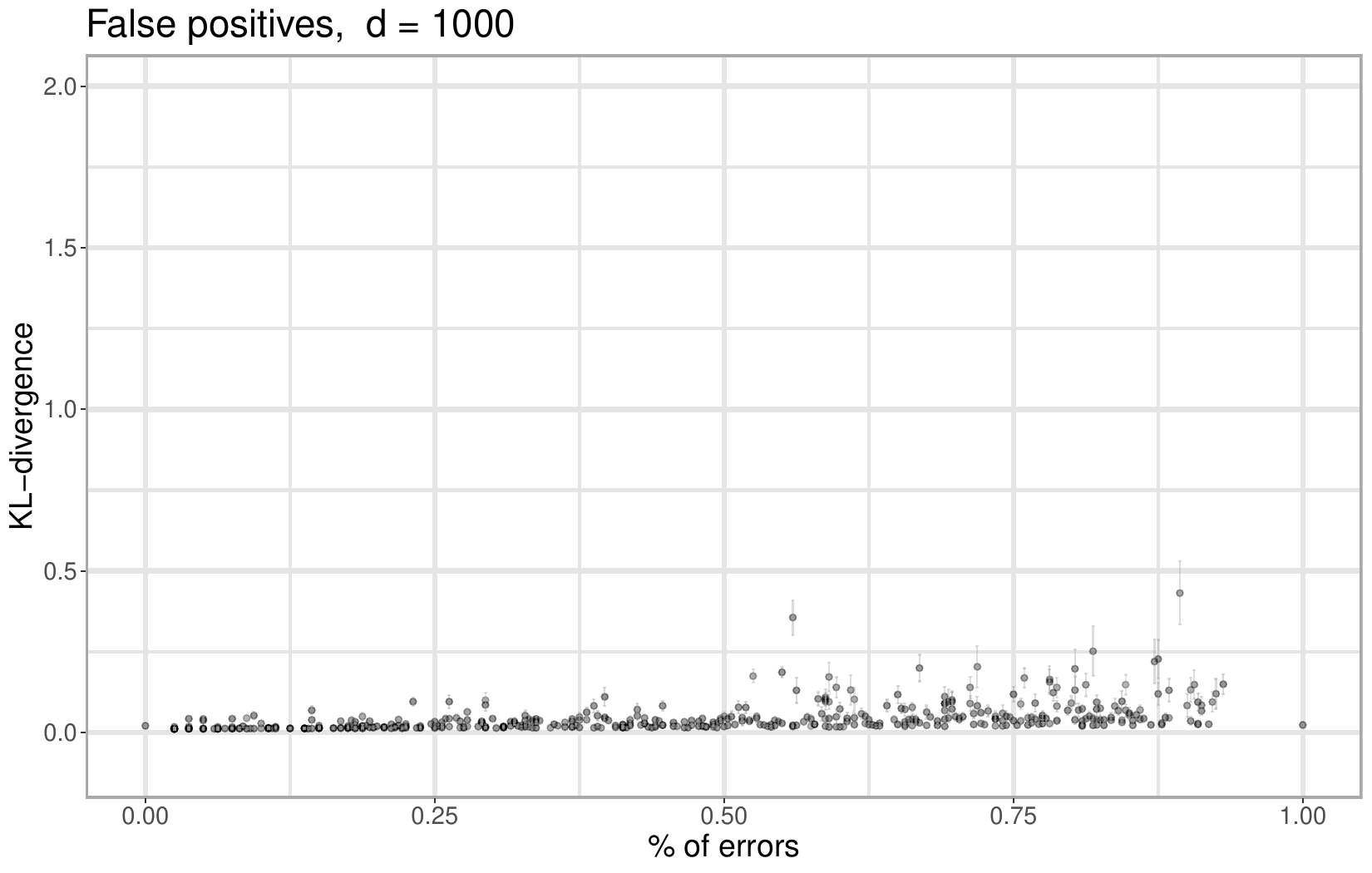}

  \vspace{2cm}

  \includegraphics[scale=\scale]{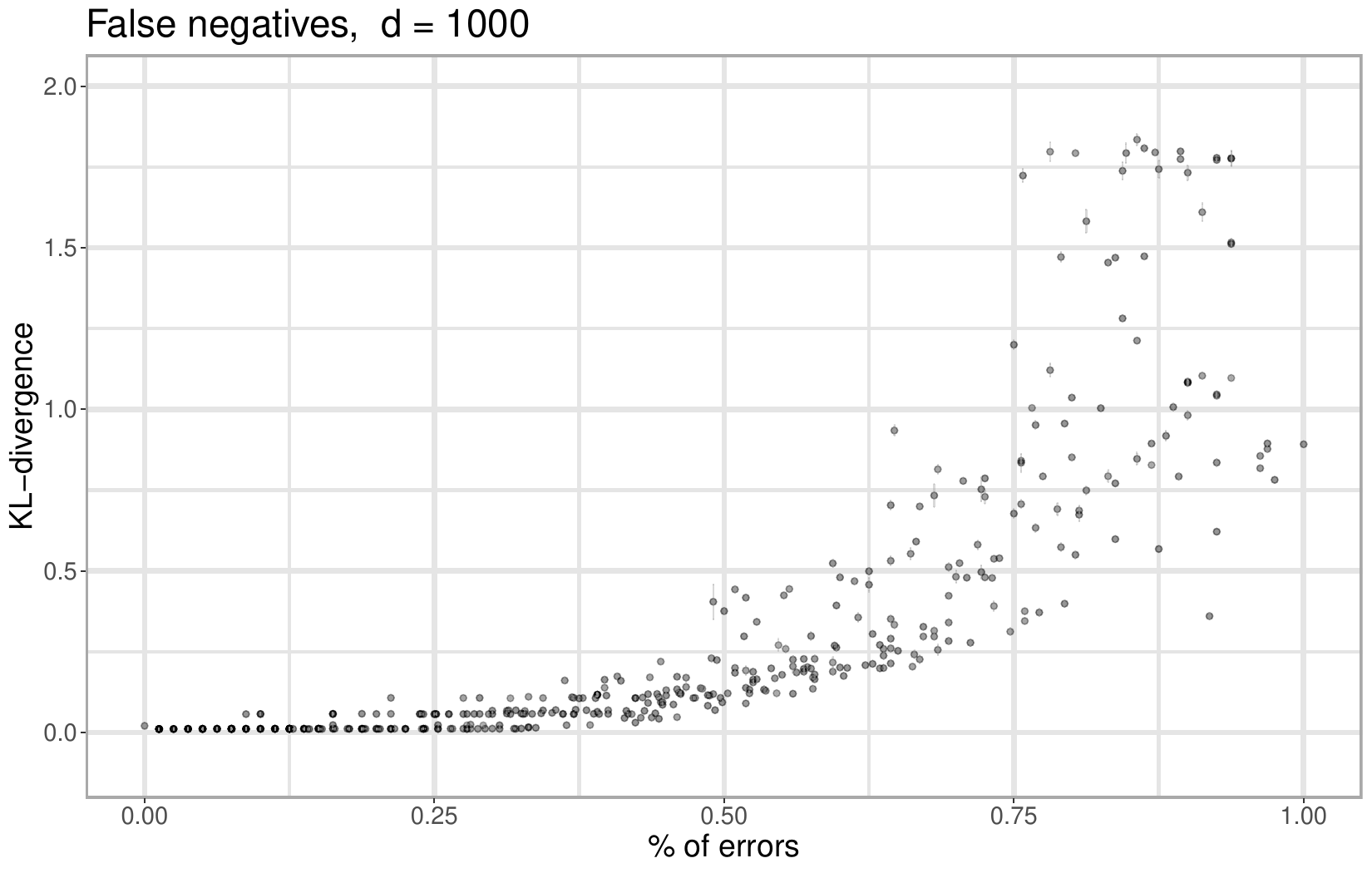}
 \caption{Error comparison for the proposed metric (x-axis) vs KL-divergence (y-axis) for synthetic datasets of 6 variables, and $s=1000$.
    Top: \% of false positives.
  Bottom: \% of false negatives.
  Each dot in the graphs is a single structure.
  }
  \label{fig:kl3}
\end{figure}
\begin{figure}[hp]
  \centering
  \includegraphics[scale=\scale]{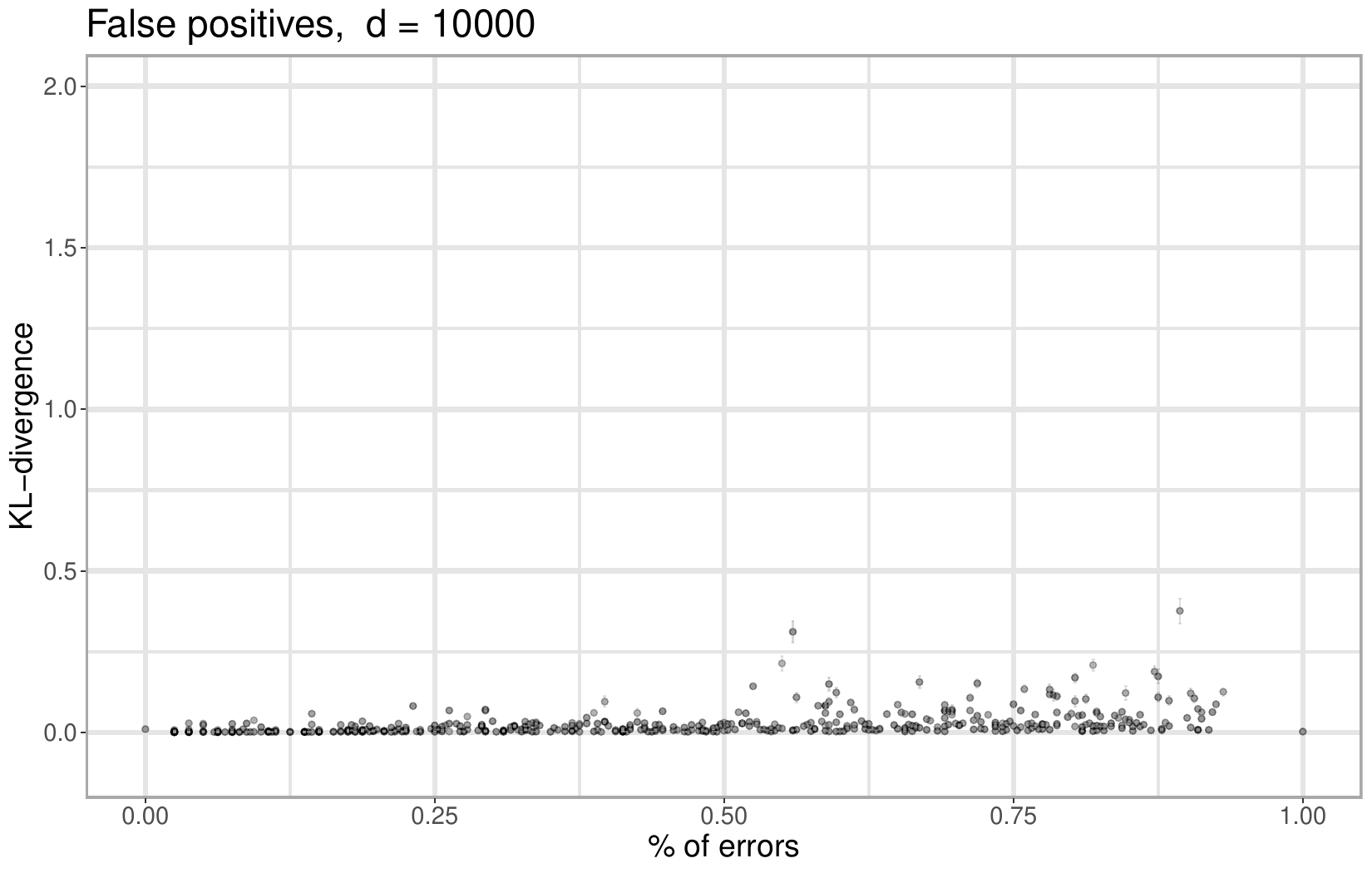}

  \vspace{2cm}

  \includegraphics[scale=\scale]{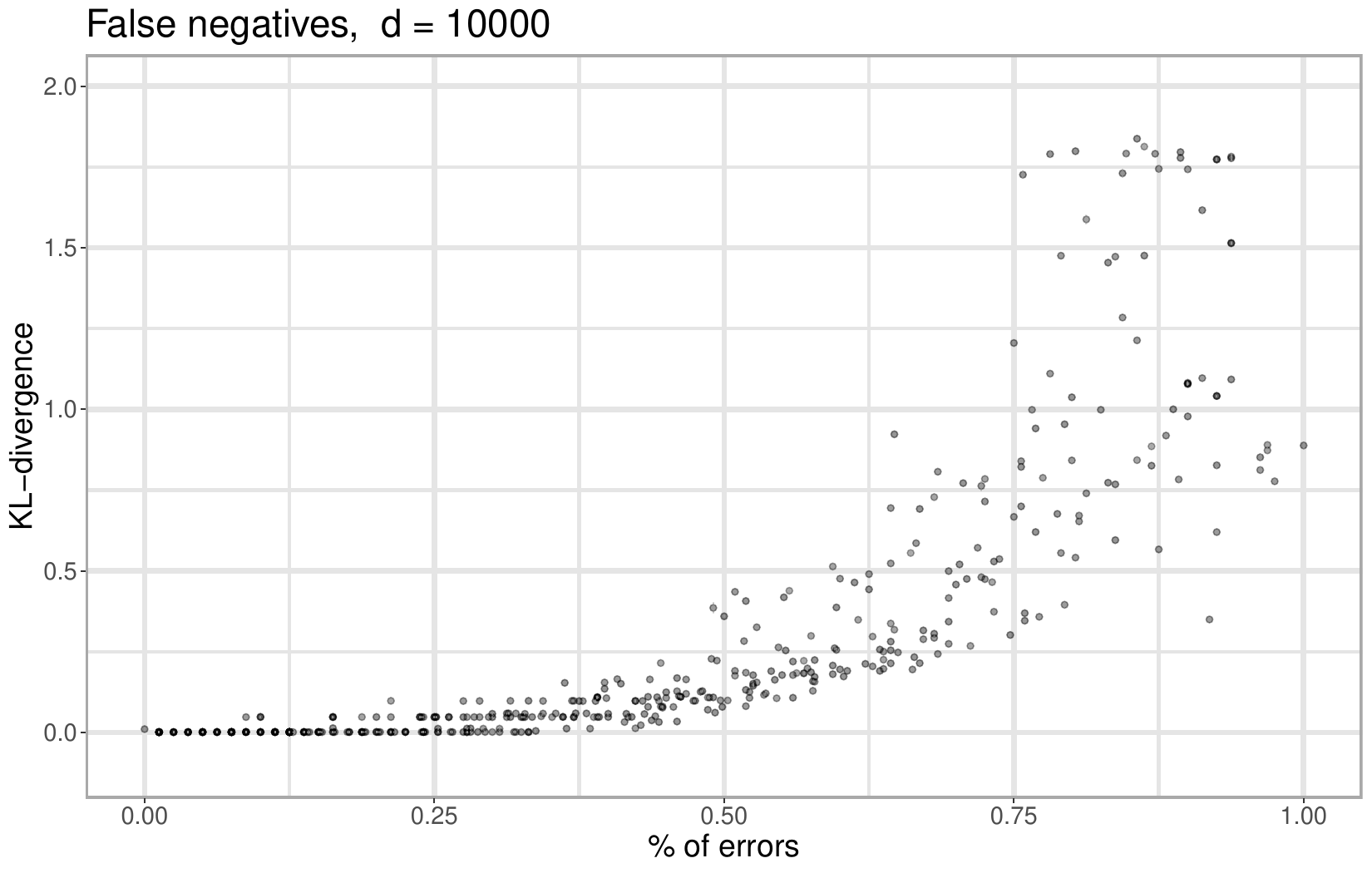}
 \caption{Error comparison for the proposed metric (x-axis) vs KL-divergence (y-axis) for synthetic datasets of 6 variables, and $s=10000$.
    Top: \% of false positives.
  Bottom: \% of false negatives.
  Each dot in the graphs is a single structure.
  }
  \label{fig:kl4}
\end{figure}

\subsection{Conclusions of this comparison}

A positive correlation between the KL-divergence and false negatives (as reported by our metric) can be observed.
This is consistent with our knowledge, since this kind of errors are due to interactions that are missing from the structure, and therefore cannot be quantified by the parameters of the model, regardless of the amount of data.
As a consequence, the KL-divergence shows the dissimilarity caused by the absence of interactions in the second model in relation to the first (original) model.
Nevertheless, the KL-divergence does not correlate with false positives.
In this case the measure over distributions can be said to conceal structural differences between models when the model to be compared possesses this kind of errors.
On a final note, the amount of data serves as a confirmation of the above: as the amount of data used for parameter learning grows, the ability of the model's parameters to mitigate spurious interactions increases.
This is caused by the compensation of the parameters, which becomes more accurate as more data are used to learn them.

\section{Conclusions}\label{sec:discussion}

In this paper we presented a metric for directly and efficiently comparing the structures of two log-linear models.
These models are more expressive than undirected graphs due to their capacity to represent context-specific independencies. 
However, the interpretation of the independence structure of these models is complex, and no sound method for making direct quality comparisons of these structures was known to us prior to the design of our metric.
The importance of a method that compares independence structures is that it can be used not only for enhancing the evaluation of structure learning algorithms, but also for qualitative comparisons in general.
First and foremost, one can analyze differences in the independence structures learned with structure learning algorithms w.r.t.~underlying synthetic structures, or compare the structures learned by different algorithms.
Furthermore, one can draw qualitative insights about structures, either those learned by algorithms or those designed by human experts (or both), which cannot be obtained by mere observation except in simple (low-dimensional) scenarios.

Also, our method provides more guarantees than state-of-the-art techniques for assessing independence structures of log-linear models.
On the one hand, for this representation, learning algorithms are usually evaluated with the KL-divergence measure for complete distributions, or the approximate method of CMLL for high dimensional domains, which require learning the numerical parameters of the models and are therefore indirect.
Besides, they do not have the properties of a metric.
On the other hand, some direct methods have been used, such as the average feature length or number of features, but these only provide very limited information about the structures, and no guarantees of their validity exist. 
In contrast, in this work we have proved that our technique is a metric, thus making it suitable for drawing reliable conclusions about comparisons made with it.
Some possible future lines of research on this method may include the search for an efficient method w.r.t. the number of features in the models, and a reproduction of results from structure learning works, adding measurements with this new metric to the existing KL-divergence or CMLL scores, in order to analyze the impact of using our measure.

\appendix

%\section{Proof of Theorem~\ref{T:method}}

\section{Lemmas for the equivalent feature sets $H^{TP}$ and $H^{FN}$}\label{app:lemmasHTPHFN}

This section contains the proofs for the sets proposed in Eqs.~\ref{eq:HTP} and~\ref{eq:HFN}, in order to show that these sets correctly represent the context sets corresponding to $TP_{ij}$ and $FN_{ij}$, respectively, as defined in~\autoref{T:method}.

\htp

%\begin{lemma} \label{lemma:HTP}
%  Let $F$ and $G$ be two arbitrary log-linear models over $X_V$, and $H^{TP}(F,G)$ be the set of union-features over $F$ and $G$ defined in Eq.~\ref{eq:HTP}, then
%
%  \[
%      \Xijof{F} \cap \Xijof{G} = \Xijof{H^{TP}(F,G)}.
%  \]
%\end{lemma}

\begin{proof}
    From~\autoref{al:Funion} and the definition of $H^{TP}(F,G)$, we have for the r.h.s. that
    \begin{align*}
      \Xijof{H^{TP}(F,G)}  &= \bigcup_{h \in H^{TP}(F,G)} \Xijof{h} \\[5pt]
                          &= \bigcup_{{f}\cupij {g} \in H^{TP}(F,G)} \Xijof{{f} \cupij {g}}   \\[5pt]
                          &= \bigcup_{f \sim C_1(f)}\,\,\,\bigcup_{g \sim C_1(g) \land C_2(f,g)} \Xijof{{f} \cupij {g}},
    \end{align*}
  %  \tensionFacu{La notacion $\cupif$ no ha sido introducida en la (nueva) definicion 2}
    where notation $a \sim c$ is used to denote the set of all elements $a$ satisfying condition $c$. For $f$ and $g$, the universe of elements is assumed to be $F$ and $G$, respectively. Thus,  $f \sim C_1(f)$ denotes all features $f \in F$ satisfying condition $C_1(f)$, and $g \sim C_1(g) \land C_2(f,g)$  denotes all features $g \in G$ satisfying conditions $C_1(g)$  and $C_2(f,g)$.

    It suffices then to prove that
    \begin{equation*}
      \Xijof{F} \cap \Xijof{G} = \bigcup_{f \sim C_1(f)} \,\,\,\bigcup_{g \sim
      C_1(g) \land C_2(f,g)} \Xijof{{f} \cupij {g}}
      \label{eq:nonemptyintersection}
    \end{equation*}

            For this purpose, we use the result of \autoref{al:Funion},
   stated and proven in \ref{app:auxiliary}, and the set property of distribution of intersection over union, in the following manner:
      \begin{align*}
	\Xijof{F} & \cap \Xijof{G} \notag\\[5pt]
	 &= \Xijof{F} \cap \left( \bigcup_{g \in G} \Xijof{g} \right)  & \text{by ~\autoref{al:Funion} over $G$} \\[5pt]
	 &= \bigcup_{g \in G} \Xijof{g} \cap   \Xijof{F}  & \text{by distributive law} \\[5pt]
         &= \bigcup_{g \in G} \left( \Xijof{g} \cap  \bigcup_{f \in F}
   \Xijof{f}\right) & \text{by ~\autoref{al:Funion} over $F$} \\[5pt]
         &= \bigcup_{g \in G} \bigcup_{f \in F} \Xijof{g} \cap \Xijof{f} & \text{by distributive law.} \\
    \end{align*}

    To conclude, we use a result stated in ~\autoref{lemma:singleintersection},
    Section~\ref{subsubsec:singlefeatures}.
    According to this lemma,  the intersections
  over individual features $\Xijof{f} \cap \Xijof{g}$ in the r.h.s.
  can be either empty, or equal to $\Xijof{{f} \cupij {g}}$; with the
   non-empty case occurring only if $f$ and $g$ satisfy
    conditions ${C}_1(f)$, ${C}_1(g)$, and ${C}_2(f,g)$.
    These conditions are precisely the restrictions stated in
    Eq.~\ref{eq:nonemptyintersection}, by which we can conclude that the equivalence is correct.

      \end{proof}

%restate lemma h fn

\hfn

%\begin{lemma} \label{lemma:HFN}
%
%  Let $F$ and $G$ be two arbitrary log-linear models over $X_V$, and $H^{FN}(F,G)$ be the set of union-features over $F$ and $G$ defined in Eq.~\ref{eq:HFN}, then
%
%  \[
%      \Xijof{F} \setminus \Xijof{G} = \Xijof{H^{FN}(F,G)}
%  \]
%  \end{lemma}

 \begin{proof}
  From~\autoref{al:Funion} and the definition of $H^{FN}(F,G)$ we have for the r.h.s. that

\begin{align} \label{eq:HFNsuffices}
    \Xijof{H^{FN}(F,G)}
            =& \bigcup_{h \in H^{FN}(F,G)} \Xijof{h} \nonumber \\[5pt]
            =& \left[ \bigcup_{f \in \Ftwo \setminus \Fone } \,
        \bigcup_{\DEf \in \crossproductDE{f}} \bigXijof{{f} \,
      \cupij \, \bigcupij_{d\in \DEf} d } \right] \nonumber \\[5pt]
       & \cup
             \left[ \bigcup_{f \notin \Fone \cup \Ftwo } \,
             \bigcup_{\DEf\in \crossproductDE{f}} \bigXijof{\bigcupij_{d\in \DEf}  d } \right].
  \end{align}

for $\crossproductDE{f}=\vvarprod_{g \in \Gcomp} \DE$, a cross-product dependent on $f$, whose elements $\DEf$ are  sets of features of cardinality $|\overline{\Gcomp}|$ computed by extracting exactly one feature from each $\DE$, with one $\DE$ defined per $g \notin \Gcomp$

  It suffices then to prove that $\Xijof{F} \setminus \Xijof{G}$ equals the r.h.s. of this last expression.
  We begin by applying a few set equivalences:

      \begin{align}
            \Xijof{F} & \setminus \Xijof{G} \notag\\
         &= \Xijof{F} \cap \overline{\Xijof{G}} & \text{by general sets
   equivalence}\notag \\[5pt]
         &= \overline{\Xijof{G}}  \cap \bigcup_{f\in F} \Xijof{f}
   & \text{by ~\autoref{al:Funion} over $F$} \notag\\[5pt]
         &= \bigcup_{f\in F} \Xijof{f} \cap \overline{\Xijof{G}}
   & \text{by distributive law} \notag\\[5pt]
         &= \bigcup_{f\in F} \Xijof{f} \setminus \Xijof{G}
   & \text{by general sets equivalence} \notag\\[5pt]
         &= \bigcup_{f\in F} \Xijof{f} \setminus \bigcup_{g\in G} \Xijof{g}
     & \text{by ~\autoref{al:Funion} over $G$} \notag\\[5pt]
         &= \bigcup_{f\in F}\quad \bigcap_{g\in G} \Xijof{f} \setminus  \Xijof{g}
   & \text{by \emph{relative complements}} \label{eq:FminusGproof}      \end{align}

      Specifically, the set equivalences used above are the following: the equivalence between
      subtracting a set and intersecting its complement (i.e., $A \setminus B= A
      \cap \overline{B}$), the property of distribution of intersection over union
      (e.g., $(A \cup B) \cap (C \cup D) = (A \cap C) \cup (A\cap D) \cup (B \cap C) \cup (B\cap D)$),
      and, in the last step, a property of sets known as \emph{relative complements} that states that for sets $A$, $B$, and $C$,
      $C\setminus(A\cup B)= (C \setminus A) \cap (C \setminus B)$.

The resulting expression involves differences of FC context sets of individual
features: $\Xijof{f} \setminus \Xijof{g}$.
From~\autoref{lemma:singlediff} (Section~\ref{subsubsec:singlefeatures}), each of these differences can take one of
  three values, depending on the conditions that hold for each combination of $f$ and $g$:

  \begin{equation}
    \Xij(f) \setminus \Xij(g)  \equiv \left\{
    \begin{array}{ll}
      \emptyset & \quad  if~ g^{ij} \subseteq f^{ij}\ \land C_1(g), \text{ or }\ \lnot  C_1(f) \\[5pt]
      \Xijof{f} & \quad  \lnot C_2(f,g)   \ \text{or}\ \lnot C_1(g)\\[5pt]
      \bigcup_{d\in \DE} \Xijof{d} & \quad otherwise.      \end{array}
     \right.\label{eq:diff:proof}
  \end{equation}

  We will now proceed to analyze the impact that the three possible values of
  the difference have in the union of intersections of Eq.~\ref{eq:FminusGproof}.
    For that, we start by decomposing the intersection over $G$ over a partition
  of $G$ consisting of three parts: $g\in G^1(f)$, $g\in G^2(f)$, and the remainder $g\notin G^1(f) \cup G^2(f)$. Note the dependence on $f$ of the partition, indicating it is different for every $f$ of the union. The partition results in

    \begin{equation}
  \begin{aligned}
     \Xijof{F} \setminus \Xijof{G}   \equiv  \bigcup_{f\in F}\quad & \left\{
              \bigcap_{g\in \Gone} \Xijof{f} \setminus  \Xijof{g}\,\,\,
	      \bigcap_{g\in \Gtwo} \Xijof{f} \setminus  \Xijof{g} \,\,\, \right. \\[5pt]
              & \left. \bigcap_{g\notin \Gcomp} \Xijof{f} \setminus  \Xijof{g}
   \right\} \nonumber
  \end{aligned}
  \end{equation}

    We shall start by the first partition over $g\in G^1(f)$. According to the definition of $G^1(f)$ in Eq.~\ref{eq:Fone}, any $g$ in it satisfies either $g^{ij} \subseteq f^{ij}\ \land C_1(g)$ or $\lnot  C_1(f)$, which according to Eq.~\ref{eq:diff:proof} is exactly the condition for the difference to be empty.
    Then, if for some $f\in F$ this condition
  holds for at least one $g \in G$, i.e., if $G^1(f) \neq \emptyset$, then the whole intersection over $G$ is empty, including all three partitions.
  From its definition in Eq.~\ref{eq:Fone}, this occurs for every $f \in  F^1$.
     We can thus omit the partition over $G^1$ by simply restricting the union only
  over features $f \notin F^1$.

  We will now analyze the second partition over $g\in G^2(f)$. 
  According to the definition of $G^2(f)$ in Eq.~\ref{eq:Ftwo}, any $g$ in it satisfies either $\lnot C_2(f,g)$ or $\lnot  C_1(g)$, which according to Eq.~\ref{eq:diff:proof} is exactly the condition for the difference to be $\Xij(f)$. 
  This results in all differences within the intersection over the second partition to be $\Xij(f)$.

  Consequently, the intersection is now equal to $\Xij(f)$. 
  There is an exception to this, when no $g$ satisfies that condition for some $f$. 
  This occurs when $G^2(f) = \emptyset$. 
  When this happens, the second intersection can be ignored.
  From its definition in Eq.~\ref{eq:Ftwo}, this occurs for every feature $f$ that is  not in $F^2$, resulting in a partition over $F$.

  Combining the above conclusions, we have that

  \begin{equation*}
  \begin{aligned}
       \Xijof{F} \setminus \Xijof{G}   \equiv
            &\left[\bigcup_{f\in \Ftwo \setminus \Fone}
                  \Xijof{f} \bigcap_{g\notin \Gcomp} \Xijof{f} \setminus  \Xijof{g} \right] \\[5pt]
            & \bigcup
            \left[\bigcup_{f\notin (\Fone \cup  \Ftwo)}
                   \bigcap_{g\notin \Gcomp}  \Xijof{f} \setminus  \Xijof{g} \right].
    \end{aligned}\label{eq:FminusGproof3}
  \end{equation*}

  This leaves us with only the third partition to analyze. For that, we notice that all $g$ not in $G^1(f)$ nor $G^2(f)$ are exactly those not satisfying neither the first nor second condition but the third condition of Eq.~\ref{eq:diff:proof}. 
  After replacing the difference by the expression corresponding to this third condition we obtain

  \begin{equation*}
  \begin{aligned}
       \Xijof{F} \setminus \Xijof{G}   \equiv
            &\left[ \bigcup_{f\in \Ftwo \setminus \Fone}
              \Xijof{f} \cap \left( \
                \bigcap_{g\notin \Gcomp}\,\,  \bigcup_{d\in \DE} \Xijof{d}  \right)\right] \\[5pt]
            &\bigcup
            \left[\bigcup_{f\notin (\Fone \cup  \Ftwo)} \bigcap_{g\notin \Gcomp} \,\, \bigcup_{d\in \DE} \Xijof{d} \right] .
    \end{aligned}\label{eq:FminusGproof3}
  \end{equation*}

To continue, we further simplify the subexpression

$$\bigcap_{g\notin \Gcomp}
  \bigcup_{d\in \DE} \Xijof{d}  ,$$

  which appears in both unions, by applying the distributive property of intersection over
  union to obtain
      \begin{equation}
  \begin{aligned}
     \Xijof{F} \setminus \Xijof{G}   \equiv
         &\left[\bigcup_{f\in \Ftwo \setminus \Fone}\quad
               \Xijof{f}  \cap \left(
                  \bigcup_{\DEf \in \crossproductDE{f}}\quad
         \bigcap_{d \in \DEf} \Xij(d) \right)\right] \\[5pt]
       & \bigcup
         \left[ \bigcup_{f\notin (\Fone \cup \Ftwo)}\quad  \bigcup_{\DEf \in \crossproductDE{f}}\quad \bigcap_{d \in \DEf} \Xij(d) \right]  . \nonumber
  \end{aligned}
    \label{eq:FminusGproof4}
  \end{equation}

  for $\crossproductDE{f}=\vvarprod_{g \in \Gcomp} \DE$, a cross-product dependent on $f$, whose elements $\DEf$ are  sets of features of cardinality $|\overline{\Gcomp}|$ computed by extracting exactly one feature from each $\DE$, with one $\DE$ defined per $g \notin \Gcomp$.
    To illustrate, if we assume that each $\DE$ contains $2$ features, then the cross-product would produce $2^{|\Gcomp|}$ features $\DEf$, for each $f$.

  One can also apply  the distribution of intersection over union of $\Xijof{f}$ onto the union over the $\DEf$, to obtain

  \begin{equation}
  \begin{aligned}
     \Xijof{F} \setminus \Xijof{G}   \equiv
          & \left[\bigcup_{f\in \Ftwo \setminus \Fone}\quad \bigcup_{\DEf \in \crossproductDE{f}}\quad \left( \Xijof{f}  \cap \bigcap_{d \in \DEf} \Xij(d) \right) \right] \\[5pt]
                  & \bigcup \left[\bigcup_{f\notin (\Fone \cup \Ftwo)}\quad \bigcup_{\DEf \in \crossproductDE{f}}\quad \bigcap_{d \in \DEf} \Xij(d) \right]  . \nonumber
  \end{aligned}
  \end{equation}

  To conclude the proof, we note that the intersection of the FC sets of $f$ and the $d$ in
  $\DEf$ corresponds to the non-empty case from Eq.~\ref{eq:intersection1} in~\autoref{lemma:singleintersection}; therefore, it can
  be replaced by the union of features over $(i,j)$, resulting in

\begin{equation}
  \begin{aligned}
     \Xijof{F} \setminus \Xijof{G}   \equiv
          &\left[\bigcup_{f\in \Ftwo \setminus \Fone}\quad
                 \bigcup_{\DEf \in \crossproductDE{f}} \quad
                 \Xij \left(f \cupij \bigcupij_{d \in \DEf} {d} \right)
          \right] \\[5pt]
                 & \bigcup
          \left[
              \bigcup_{f\notin (\Fone \cup \Ftwo)}\quad
              \bigcup_{\DEf \in \crossproductDE{f}}\quad
              \Xij\left(\bigcupij_{d \in \DEf} {d}\right) \right]  . \nonumber
  \end{aligned}
  \end{equation}

  The above expression matches exactly what Eq.~\ref{eq:HFNsuffices} indicated is sufficient for proving the lemma.
\end{proof}

\section{Lemmas for the efficient computation of $\Xijof{f}\cap \Xijof{g}$ and $\Xijof{f} \setminus \Xijof{g}$}\label{app:lemmas}

This appendix provides the proofs for \autoref{lemma:singleintersection} and
\autoref{lemma:singlediff} of Section~\ref{subsubsec:singlefeatures}, which
propose an approach for the efficient computation of the intersection  $\Xijof{f}\cap \Xijof{g}$ and difference  $\Xijof{f}\setminus \Xijof{g}$ of the FC context sets of single features, respectively.

\singleinter

\begin{proof}
    We will consider each case separately.

    \textbf{Case $\lnot C_1(f)$}: By the definition of this condition, $X_i \notin \scope{f}$ or $X_j \notin \scope{f}$. This contradicts the r.h.s. in ~\autoref{aux:XijSubset} (see ~\ref{app:auxiliary}), which implies that there is no $\z\in \Xijof{f}$, or equivalently, $\Xij(f)=\emptyset$. This in turn implies an empty intersection.

    \textbf{Case $\lnot C_1(g)$}: The case of $C_1(f)$ applies here as well, resulting in $\Xij(g)=\emptyset$, and therefore in an empty intersection.

    \textbf{Case $\lnot C_2(f,g)$}:  For this case, neither $\Xij(f)$ nor $\Xij(g)$ are empty, but their intersection is.
    To prove this, we argue that given any two FC contexts $\z \in \Xij(f)$ and $\z' \in \Xij(g)$, they must differ in the assignment of at least one of its variables.
    By the definition of $C_2(f,g)$ in Eq.~\ref{eq:C2}, its negation implies that there exists at least one variable $X_h$ other than $X_i$ and $X_j$ that is both in $f$ and $g$, such that $X_h(f) \neq X_h(g)$.
        According to Eq.~\ref{eq:XijSubset} of \autoref{aux:XijSubset} (see~\ref{app:auxiliary}), we have that,
                        for all $x_Z \in \Xijof{f}$, $X_h(x_Z)=X_h(f)$, while for all $x'_Z \in \Xijof{g}$, $X_h(x'_Z)=X_h(g)$.
    Therefore, for all $x_Z \in \Xijof{f}$ and all $x'_Z \in \Xijof{g}$, we have $X_h(x_Z) \neq X_h(x'_Z)$ which implies $x_Z \neq x'_Z$, by which we conclude that no FC context belongs to both $\Xijof{f}$ and $\Xijof{g}$ simultaneously.

    \textbf{Case $C_1(f) \land C_1(g) \land C_2(f,g)$ }: We must prove that $\Xij(f) \cap \Xij(g) \equiv \Xij(f \cup \p{g})$.
    In what follows, we denote by $\z$ an arbitrary FC context in $\Xij$.
        From basic set theory, we have that

          \begin{equation*}
            \z \in \Xij(f) \cap \Xij(g)  \iff \z \in \Xijof{f} \land \z \in \Xijof{g}.
        \end{equation*}

        By applying ~\autoref{aux:XijSubset} to the conditions in the r.h.s. we obtain

          \begin{equation*}
              \z \in \Xij(f) \cap \Xij(g) \iff  \p{f} \subseteq \z \,\,\,\land\,\,\, \p{g} \subseteq \z \,\,\,\land\,\,\, X_i,X_j\in\scope{f} \,\,\,\land\,\,\, X_i,X_j \in \scope{g}.
        \end{equation*}

        From set theory, we know that for arbitrary sets $A$, $B$, and $C$, $A\subseteq C \land B \subseteq C$ is equivalent to $A \cup B \subseteq C$.
	This applies in particular to the union of features (\autoref{def:union}), as the union $\p{f} \cup \p{g}$ is a feature that contains all assignments in both features, and by the r.h.s., all of these assignments are in $\z$.
	Also, from the definition of feature union, if $X_i,X_j$ are in the scopes of both $f$ and $g$, then they are in the scope of their union.
	Combining both conclusions, we obtain

         \begin{equation*}
            \z \in \Xij(f) \cap \Xij(g) \iff   \p{f} \cup \p{g} \subseteq \z \,\,\,\land\,\,\, X_i,X_j \in \scope{f \cup g}.
        \end{equation*}
	To conclude the proof we apply the right-to-left direction of Eq.~\ref{eq:XijSubset} (\autoref{aux:XijSubset}), where the feature $\p{f} \cup \p{g}= \p{h}$ should have a corresponding feature $h$ in the l.h.s. of Eq.~\ref{eq:XijSubset}, that is, a feature that contains $X_i$ and $X_j$ in its scope.
	Note that the l.h.s. of the auxiliary lemma holds for any arbitrary assignment to this pair of variables; then,
	by \autoref{def:unionij}, such a feature can be expressed as a union over the pair $(i,j)$, namely, $f\cupij g$, which allows us to obtain

      %   \begin{equation*}
	   %\z \in \Xij(f) \cap \Xij(g) \iff \z \in \Xijof{f \cup \p{g}},
      %  \end{equation*}

%\propuestaF{Asi como estaba escrito no quedaba nada claro porque podia cambiarse $\p{f}$ por $f$. Lo lei y no logre comprender porque es que es valido cambiar $\p{f}$ por $f$. Es mas, asi redactado suena mas a que lo cambias porque lo necesitas para que cierren las cosas, a que porque es una conclusiòn a la que se puede arribar por inferencia logica.}
%{La forma correcta es aplicar una inferencia logica adecuada. Segun entiendo, esto se logra al revisar el A.L. 1 y ver que la feature resultante de aplicar el A.L.1 en direccion right-to-left (ver tension y propuesta en el enunciado de este A.L.) permite agregar cualquier asignaciòn a $X_i$ y $X_j$. Asi que cambie la explicacion que sigue, y agregue una explicacion extra en el A.L.1}

         \begin{equation*}
     \z \in \Xij(f) \cap \Xij(g) \iff \z \in \Xijof{f \cupij {g}}.
        \end{equation*}

	%where we substituted $\p{f}$ by $f$ in order to apply the FC context set function $\Xij$.
%	where $f^{ij} \cup g^{ij}$ has been substituted by $f\cupij g$ to address the fact that according to the \autoref{aux:XijSubset} the resulting feature in the r.h.s. must contain assignments for both $X_i$ and $X_j$, and
%	by noting that the right-to-left implication in~\autoref{aux:XijSubset} holds for any arbitrary assignment, in particular the values assigned to $X_i$ and $X_j$ in feature $f$.

\end{proof}

We will now prove the following lemma for the difference of the FC contexts of single features.

\singledif

\begin{proof}
For the first case of Eq.~\ref{eq:diff:D}, we consider the two cases of the disjunction in the condition separately:

$\lnot C_1(f)$: If this is the case, the right-hand side of ~\autoref{aux:XijSubset} never holds.
This results in $\Xij(f)=\emptyset$, and consecuently the difference is empty.

$\p{g} \subseteq \p{f} \land C_1(g)$: For this case we prove the difference is empty by showing that $\Xijof{f} \subseteq \Xijof{g}$, which means that for every $\z \in \Xij$ it is the case that
\[
   \z \in \Xijof{f}  \implies \z \in \Xijof{g}
\]
From the left-hand side and ~\autoref{aux:XijSubset}, we have that $\p{f} \subseteq \z  \land  {C}_1(f)$, which combined with condition $\p{g} \subseteq \p{f}$ results in $\p{g} \subseteq \z$. The latter combined with $C_1(g)$ can be applied to the left-to-right implication of Eq.~\ref{eq:XijSubset} of the  auxiliary lemma to obtain $\z \in \Xijof{g}$.

For the second case of Eq.~\ref{eq:diff:D}, if $\lnot C_2(f,g)$ holds, some value in $g$ does not match its corresponding value in $f$, and therefore none of the elements of $\Xij(g)$ are in $\Xij(f)$. If $\lnot C_1(g)$ holds, then $\Xij(g)=\emptyset$. In both cases, nothing can be subtracted from $\Xij(f)$.

The proof for the third case of Eq.~\ref{eq:diff:D} consists in proving that, when the conditions of the first two cases are not satisfied, then

    \begin{equation*}
        \Xijof{f} \setminus \Xijof{g} = \bigcup_{d\in \DE}{\Xijof{d}}.\label{eq:minusDE}
    \end{equation*}

For that we proceed in two steps. First, we prove this equality for $\DB$,  an alternative (simpler) version of $\DE$, i.e.,

    \begin{equation}
        \Xijof{f} \setminus \Xijof{g} = \bigcup_{d\in \DB}{\Xijof{d}}.\label{eq:minusDE}
    \end{equation}

where the equivalent set $\DB$ is defined as

   \begin{equation}
                \DB = \left\{ \text{\ features }d \ \middle\vert
                \begin{array}{l}
                S_d = S_f\cup S_g; \forall X_m\in S_{f}, X_m(d)=X_m(f); \\
                \exists X_k\in S_g\setminus S_f, X_k(d) \neq X_k(g)\end{array}\right\}.\label{eq:DB}
            \end{equation}
    that is, one feature $d$ with scope composed of all variables in the scopes of both $f$ and $g$, with matching values with those in $\scope{f}$, and a mistmatch with at least one variable in $\scope{g}$.

Then, we prove that $\DB$ is equivalent to $\DE$:

        \begin{equation*}
            \bigcup_{d\in \DB}{\Xijof{d}} = \bigcup_{d\in \DE}{\Xijof{d}},\label{eq:DBequalsDE}
        \end{equation*}

         The validity of these two steps is proven below in \cref{lemma:L1diff,lemma:L2}, respectively.
 Additionally, we include ~\autoref{lemma:L4}, which shows that the FC contexts of each feature in $\DE$ are mutually exclusive, which guarantees that redundancy is conveniently reduced, and is necessary for Algorithm~\ref{alg:partition} to produce a correct result.

\end{proof}

\begin{lemma}\label{lemma:L1diff}
Let $X_i$ and $X_j$ be two distinct variables in $X_V$,  and let $f$ and $g$ be
two arbitrary features over $X_V$  satisfying neither the first nor second conditions of Eq.~\ref{eq:diff:D}. Then,

  \begin{equation}
        \Xijof{f} \setminus \Xijof{g} = \bigcup_{d\in \DB}{\Xijof{d}}.\tag{\ref{eq:minusDE}}
    \end{equation}
    with $\DB$ defined as in Eq.~\ref{eq:DB}.

                                                  \end{lemma}
\begin{proof}
    By set equivalence, Eq.~\ref{eq:minusDE} can be reformulated, for an arbitrary $\z \in \Xij$, as

    \begin{equation*}
         \z \in \Xijof{f} \land \z \notin \Xijof{g}
        \iff \exists d \in \DB, \z \in \Xijof{d}.
    \end{equation*}

    Since the first and second cases in Eq.~\ref{eq:diff:D} are not satisfied,
    we have that $\p{g} \not\subseteq \p{f}$, $C_2(f,g)$, $C_1(f)$, and
    $C_1(g)$; this is easily demonstrated by the simple application of logical equivalences over the negation of the first two conditions.

    Then, applying  ~\autoref{aux:XijSubset}  of ~\ref{app:auxiliary} to each of the three inclusions $\z\in\Xijof{f}$, $\z\in\Xijof{g}$, and $\z\in\Xijof{d}$ we obtain

     \begin{equation}
      \left(C_1(f) \land \p{f}\subseteq \z \right) \land
       \left( \lnot C_1(g) \lor \p{g} \not\subseteq \z \right)
        \iff
       \exists d \in \DB,\ \left(C_1(d) \land \p{d}\subseteq \z\right).\label{eq:diff:L1implication}
    \end{equation}

    We begin with the right-to-left implication, and prove each term in the l.h.s. separately, i.e.,

    \begin{enumerate}[(i)]
        \item $\exists d\in\DB, \left(C_1(d) \land \p{d} \subseteq \z \right)
              \implies \left(C_1(f) \land \p{f} \subseteq \z \right)$; and        \item $\exists d\in\DB, \left(C_1(d) \land \p{d} \subseteq\z\right)
              \implies \left(\lnot C_1(g) \lor \p{g} \not\subseteq \z \right)$.     \end{enumerate}

        \begin{enumerate}[(i)]
        \item
    Since $\scope{d}=\scope{f}\cup\scope{g}$ and both $C_1(f)$ and $C_1(g)$ hold, then $C_1(d)$ holds as well. It then suffices to prove
    $\exists d\in\DB, \p{d} \subseteq \z \implies \p{f} \subseteq \z$.

    For that, we use ~\autoref{aux:subset}.
    We consider $a=f^{ij}$ and $b=d^{ij}$, where $d$ is the feature satisfying the l.h.s. of (i), i.e., the feature $d$ for which $d^{ij} \subseteq \z$.
    We thus have that $b\subseteq \z$.
    To conclude that $a=f^{ij} \subseteq \z $ it suffices then to prove that $f^{ij} \subseteq d^{ij}$ and that $\scope{f} \subseteq \scope{d}$.
         By definition, $S_{d}=S_f\cup S_g$, so $S_f \subseteq S_{d}$ holds.
    Also, by definition, for every $k \in \DB$, it is the case that
    $\forall X_m\in S_{f^{ij}}, X_m(k)=X_m(f^{ij})$, which is equivalent to saying that for every $k \in \DB$, $\p{f} \subseteq \p{k}$.
    In particular, this must then hold for $d$.
    \qed

            \item Again, $C_1(d)$ holds. Also, since we already know that $C_1(g)$, the r.h.s. reduces to $\p{g} \not\subseteq \z$. It thus suffices to prove that $\exists d\in\DB, \p{d} \subseteq\z
              \implies \p{g} \not\subseteq \z$.

    To prove this  we first note that

    \begin{equation}
        d^{ij} \subseteq \z \implies \forall X_m\in S_{d^{ij}}, X_m(d^{ij})=X_m(\z).\label{eq:diff:L1:I2proof}
    \end{equation}

    Then, from the definition of every $d \in \DB$, $\exists X_m\in S_g\setminus S_f$ s.t. $X_m(d) \neq X_m(g)$. Also, since both $X_i$ and $X_j$ are in $S_f$, then $S_g\setminus S_f =S_{g^{ij}}\setminus S_{f^{ij}}$, so
    $\exists X_m\in S_{g^{ij}}\setminus S_{f^{ij}}$ s.t. $X_m(d^{ij})\neq X_m(g^{ij})$.
    Combining with Eq.~\ref{eq:diff:L1:I2proof}, this results in
    $\exists X_m\in S_{g^{ij}}\setminus S_{f^{ij}}$ s.t. $X_m(g^{ij})\neq X_m(\z)$,
    from which we can conclude that $g^{ij} \not\subseteq \z$.

    \qed
  \end{enumerate}

    We proceed now to prove the left-to-right implication of
    Eq.~\ref{eq:diff:L1implication}. For that, it suffices to consider the
    l.h.s.~without terms $C_1(f)$ and $\lnot C_1(g)$.
    The former can be omitted because it is a condition of the lemma that
    $C_1(f)$, while the latter can be omitted because again, it is a condition
    of the lemma that $C_1(g)$, i.e., $\lnot C_1(g)$ is false, so for the disjunction to be true it must be that $\p{g} \not\subseteq \z$. Finally, we already showed above that $C_1(d)$.
        It thus suffices to prove that:
    \[
      \p{f}\subseteq \z \land \p{g} \not\subseteq \z \implies
        \exists d \in \DB, \p{d} \subseteq \z,
    \]

    that is, given the condition on the l.h.s. for any arbitrary $\z$, there is some $d\in \DB$ for which the r.h.s. is satisfied, i.e., $\p{d} \subseteq \z$.

    We start by reinterpreting the l.h.s.:

    $(f^{ij}\subseteq\z):$
    \begin{equation}
    (f^{ij}\subseteq\z): \forall X_m\in S_{f^{ij}}, X_m(f^{ij})=X_m(\z).\label{eq:diff:L1l2r1}    \end{equation}

                    $(g^{ij}\not\subseteq\z):$

    Since $S_{\z}=V\setminus \{X_i,X_j\}$ and $S_g^{ij}\subseteq S_{\z}$,
    then  $g^{ij}\not\subseteq \z$ can only occur because some assignment in
    $g^{ij}$ takes a value that is different from $\z$, i.e.,

    \begin{equation}
        \exists X_m\in S_{g^{ij}}\text{ s.t. } X_m(g^{ij})\neq X_m(\z). \label{eq:diff:L1l2r2b}     \end{equation}

    Nevertheless, by Eq.~\ref{eq:diff:L1l2r1} and $C_2(f,g)$, this cannot be the case for
    all $X_m\in S_{g^{ij}}\cap S_{f^{ij}}$ so Eq.~\ref{eq:diff:L1l2r2b} can be
    re-expressed as

    \begin{equation}
        \exists X_m\in S_{g^{ij}}\setminus S_{f^{ij}}\text{ s.t. }  X_m(g^{ij})\neq X_m(\z). \label{eq:diff:L1l2r2c}     \end{equation}

    Finally, since $C_1(f)$ holds, both $X_i$ and $X_j$ are in $S_f$,
    so $S_{g^{ij}}\setminus S_{f^{ij}}= Sg \setminus S_f$, i.e.,
    Eq.~\ref{eq:diff:L1l2r2c} becomes

    \begin{equation}
        \exists X_m\in S_{g}\setminus S_{f}, X_m(g^{ij})\neq X_m(\z). \label{eq:diff:L1l2r2d}     \end{equation}

    Let $M$ denote those $X_m$ that satisfy Eq.~\ref{eq:diff:L1l2r2d}, i.e.,

    \begin{equation}
        \forall X_m\in M, X_m(g^{ij})\neq X_m(\z). \label{eq:diff:L1l2r2e}     \end{equation}

    We proceed by proposing some feature $d$ defined over $S_f\cup S_g$ that
    satisfies $d^{ij}\subseteq \z$, .i.e,

    \begin{equation}
        \forall X_m\in S_d^{ij}=S_f\cup S_g, ~X_m(d)= X_m(\z),\label{eq:diff:L1l2r2f}     \end{equation}

    and prove that $d\in \DB$.
    For that, we prove that $d$ satisfies all three conditions in the
    definition of any $d\in \DB$:
    \begin{enumerate}
        \item The first condition, $S_d=S_f\cup S_g$ is satisfied by the definition of $d$.
        \item From Eqs.~\ref{eq:diff:L1l2r1}~and~\ref{eq:diff:L1l2r2f}, and the fact that $S_{f^{ij}}\subseteq S_d$ (by definition of $d$), we have that
        $\forall X_m \in S_{f^{ij}}, X_m(f^{ij})=X_m(\z)=X_m(d)$,
        satisfying the second condition of $\DB$ for $X_m(f^{ij})$ and
        $X_m(d)$.
        \item From Eqs.~\ref{eq:diff:L1l2r2e}~and~\ref{eq:diff:L1l2r2f},
        and the fact that $M\subseteq S_g\setminus S_f$ (by definition of $M$),
        and $S_g\setminus S_f \subseteq S_d$ (by definition of $d$),
        and consequently $M\subseteq S_d$, we have that

    \begin{equation*}
        \forall X_m\in M, X_m(g^{ij})\neq X_m(\z)=X_m(d),\label{eq:diff:L1l2r2g}     \end{equation*}

        then $\forall X_m\in M, X_m(d)\neq X_m(g^{ij})$, satisfying the third condition of $\DB$.

    \end{enumerate}

\end{proof}

\begin{lemma}\label{lemma:L2}
    \begin{equation}
        \bigcup_{d\in \DB}{\Xijof{d}} = \bigcup_{d\in \DE}{\Xijof{d}}.\label{eq:diff:l2}
    \end{equation}
    for $\DE$ defined by Eq.~\ref{eq:DE} and $\DB$ defined by Eq.~\ref{eq:DB}.
\end{lemma}
\begin{proof}
    For arbitrary $\z$, Eq.~\ref{eq:diff:l2} is equivalent to

    \begin{equation*}
         \exists d\in\DB, ~\z \in \Xijof{d} \iff \exists d'\in\DE, \z \in \Xijof{d'}.  \\
     \end{equation*}

    Since $C_1(f)$,     $S_f \subseteq S_d$ and $S_f \subseteq S_{d'}$,
    we have that both $X_i$ and $X_j$ are in both $\scope{d}$ and $\scope{d'}$, and therefore it holds that $C_1(d)$ and $C_1(d')$. We can then apply  Eq.~\ref{eq:XijSubset} to the above to obtain
    \begin{equation}
         \exists d\in\DB, ~d\subseteq\z \iff d\in\DE, d\subseteq\z . \label{eq:diff:L2}
    \end{equation}

        For the left-to-right implication, by ~\autoref{aux:XijSubset},
    it suffices to prove that

    \begin{equation}
        \exists d\in\DB \land \exists d'\in\DE\text{ s.t. }d'\subseteq d,\label{eq:diff:L2:dsubset}
    \end{equation}

    for $a=d', b=d$.
    Any $d'\in\DE$ and $d\in\DB$ satisfy $S_f\subseteq S_{d'}$ and $S_f\subseteq S_{d}$, respectively, and match the values of $f^{ij}$,
    so to prove $d' \subseteq d$ we can focus on the values for $g$.
        For that, we start noticing that every feature $d' \in \DE$ is defined over some subset of $S_g\setminus S_f$ (dependent on $k$), over which it is guaranteed to have an assignment different from that of $g$ at $k$, i.e.,  $X_k(d')\neq X_k(g)$. Then, any $d$ satisfying these assignments for these variables in
    $S_{d'}\setminus S_f\ = S_{\p{d}}^{\leq k} \setminus S_f$, and any value for the
    remaining assignments in $S_g\setminus S_f$ would satisfy that $\exists X_m \in S_g\setminus S_f, X_m(d)\neq X_m(g)$, and thus is in $\DB$.

        For the right-to-left implication of Eq.~\ref{eq:diff:L2},
    we have that, by Eq.~\ref{eq:diff:L2:dsubset}, for every $d'\in \DE$ there
    exists a $d\in \DB$ such that $d'\subseteq d$.  It suffices then to complete $d'$ with assignments for the remaining
    variables with values matching $\z$, i.e.,

    \begin{equation}
      \forall X_k \in S_{\p{g}}^{>k}\setminus S_f, X_k(d)=X_k(\z).\label{eq:diff:L2:complete}
    \end{equation}

    Then, by Eq.~\ref{eq:diff:L2:complete} and the fact that $d'\subseteq \z$, we conclude that $d\subseteq \z$.

\end{proof}

\begin{lemma}\label{lemma:L4}
    The FC contexts for each feature in $\DE$ are mutually exclusive:
    \begin{equation}
        \forall d,d' \in \DE, \Xijof{d} \cap \Xijof{d'} = \emptyset
    \end{equation}
\end{lemma}
\begin{proof}
    Given the definition of $\DE$ in Eq.~\ref{eq:DEk}, all features $d$ in some
    ${\DE}_{(k)}$ have different values at $X_k$ among each other, by which they cannot have FC contexts in common.
    Additionally, for another $k'$ such that $k' > k$, not only are the FC
    contexts of the features mutually exclusive among each other, but they are
    also different from all features for the previous index $k$, since these
    features take values different from $X_k(g)$ at $X_k$, while features at $k'$ have the value $X_k(g)$ at $X_k$.

\end{proof}

\section{Auxiliary lemmas}\label{app:auxiliary}

 \begin{auxlemma}  \label{aux:XijSubset}
   Let $F$ be the log-linear model of some distribution over $X_V$, $f \in F$ be
   some feature in $F$, let $X_i, X_j \in S_f$ be two different variables in the scope of $f$,
   $\z\in \Xij$ be some FC context, and $\p{f}$ denote the feature composed of the same assignments in $f$ except for those of $X_i$ and $X_j$. Then,

  \begin{equation}  \label{eq:XijSubset}
      \z \in \Xij(f) \iff  \p{f} \subseteq \z  \land X_i,X_j \in \scope{f},
  \end{equation}
  where the subset operation $\p{f} \subseteq \z$ runs over assignments, that is, it reads that every variable in feature $f$ other than $X_i$ and $X_j$ is assigned to the same value in both $\p{f}$ and $\z$.
   \end{auxlemma}
    %\propuestaFacu{La implicacion right-to-left es ambigua, ya que no restringe los valores que deben tomar ni $X_i$ ni $X_j$ en $f$. O dicho de otra manera, el r.h.s. se cumple para cualquier $h$ que tenga a $X_i$ y $X_j$ en su scope y $\p{h}=\p{f}$ }{Aclarar esto para facilitar luego la demostraciones que usen este A.L., en especial la de la single feature union. Para ello agreguè el siguiente parrafo antes de la proof:}
%
  %  It should be noted that the right-to-left implication is ambiguous, in that  it is valid for any feature $f$ in the l.h.s. as long as it contains both $X_i$ and $X_j$ in its scope, and whose other assignments match those of $\z$, \emph{regardless of its assignments for $X_i$ and $X_j$}.

%   Let us see why is this true.
\begin{proof}
    This auxiliary lemma is a straight rewrite of known facts from the theory of log-linear models. From the definition of $\Xij(f)$ in Eq.~\ref{eq:XijF}, $\z \in \Xij(f)$ whenever $\cd{X_i}{X_j}{\z}_F$. According to [Chapter 4, \cite{koller09}], any distribution structured by some log-linear model $F$, holds a direct interaction (a dependency) among any pair of variables $X_i$ and $X_j$, whenever they appear together in at least one feature $f \in F$.
        Moreover, the interaction still holds when conditioned on some partial assignment of variables, when there is at least one feature $f \in F$ satisfying that all its assignments match the conditioning set.
        In particular, for the case of FC conditioning sets,
    $\assertion{X_i}{X_j}{\z}$ is false (dependency) whenever each variable in
    $\z$ (which is also in the scope  $\scope{f}$) has a matching value in both.
    The auxiliary lemma is proven after noticing that $\z$ contains neither $X_i$ nor $X_j$.

\end{proof}

\begin{auxlemma} \label{al:Funion}
    The FC set $\Xijof{F}$ of a log-linear model $F$ is equivalent to the union of the FC sets $\Xijof{f}$ of each of its features $f \in F$; formally,

    \begin{equation*}
        \Xijof{F} = \bigcup_{f\in F}  \Xijof{f}.
    \end{equation*}
\end{auxlemma}

Before proceeding to the proof, we will illustrate this definition with an example:

  \begin{example}
    Let~$V=\{ 1,\dots,4\}$, $\forall X_k, val(X_k)=\{0,1\}$,
 and let $ f, f', f'' \in F$ where

 \begin{align*}
    f = & <X_1=0, X_3=0, X_4=1>,\\
    f' = & <X_2=1, X_3=0, X_4=0>,\ and\\
    f'' = & <X_1=0, X_2=0>.
  \end{align*}

  Let $i=3$ and $j=4$. Then,

  \begin{align*}
    \mathcal{X}^{34}(f) =&\ \{<X_1=0,X_2=0>,<X_1=0,X_2=1>\}\\
    \mathcal{X}^{34}(f') =&\ \{<X_1=0,X_2=1>,<X_1=1,X_2=1>\}\\
    \mathcal{X}^{34}(f'') =&\ \emptyset.
  \end{align*}

  Then, $\mathcal{X}^{34}(F)$ is the union of the sets of FC contexts $\mathcal{X}^{34}(f)$ and $\mathcal{X}^{34}(f')$, assuming $f$ and $f'$ are the only features in $F$ containing $X_3$ and $X_4$ in their scope, resulting in

  \begin{align*}
    \mathcal{X}^{34}(F) =&  \mathcal{X}^{34}(f) \cup  \mathcal{X}^{34}(f')  \\
    =&\{<X_1=0,X_2=0>,<X_1=0,X_2=1>,<X_1=1,X_2=1>\}.
  \end{align*}

\end{example}\label{ex:XF2union}

\begin{proof}
        From Eq~\ref{eq:XijF}, $ \Xij(F) \equiv \{ \z \in \Xij ~\mid ~ \cd{X_i}{X_j}{\z}_F \}$; also,
        by the same reasoning followed in ~\autoref{aux:XijSubset}, if an assertion $\assertion{X_i}{X_j}{\z}$ is false (dependency) according to some feature $f \in F$, it is false according to the complete log-linear model:
        \begin{equation*}
           \cd{X_i}{X_j}{x_u,X_W}_f  \implies \cd{X_i}{X_j}{x_u,X_W}_F.
        \end{equation*}

        Therefore,
        \begin{eqnarray*}
              \Xij(F) &\equiv& \{ \z \in \Xij ~\mid ~ \cd{X_i}{X_j}{\z}_F \} \\
                      & \equiv& \{ \z \in \Xij ~\mid ~ \vee_{f\in F} \cd{X_i}{X_j}{\z}_f \}  \\
                      & \equiv & \bigcup_{f\in F} \left\{ \z \in \Xij ~\mid ~ \cd{X_i}{X_j}{\z}_f \right\} \\
                      & \equiv & \bigcup_{f \in F} \Xijof{f}
        \end{eqnarray*}
\end{proof}

\begin{auxlemma}\label{aux:subset}
    Let $a$ and $b$ be two features such that $S_a \subseteq S_b$,
    and $a \subseteq b$, then $b \subseteq \z \implies a \subseteq \z$.
\end{auxlemma}

\begin{proof}
\begin{equation}
    a \subseteq b \implies \forall k\in S_a, X_k(a)=X_k(b),\label{eq:alsubset}
\end{equation}

then,

\begin{equation*}
\begin{array}{ll}
    b \subseteq \z&\ \implies \forall k\in S_b X_k(b)=X_k(\z) \\
    &\ \implies \forall k\in S_a X_k(b)=X_k(\z)\text{, by }S_a\subseteq S_b \\
    &\ \implies \forall k\in S_a X_k(a)=X_k(\z)\text{, by Eq.~\ref{eq:alsubset}} \\
    &\ \implies a \subseteq \z.
    \end{array}
\end{equation*}
\end{proof}

\section*{Funding}
This work was supported by CONICET (Argentinean Council for Scientific and Technological Research) [full doctoral scholarship for Jan Strappa]; and Universidad Tecnológica Nacional [grant EIUTIME0004481TC].

\section*{Acknowledgements}
We thank the reviewers for their useful comments. 
We would like to thank Dr. Federico Schl\"uter for his assistance in writing the introduction section and for having participated in the design of the method.
%We also wish to thank Carolina Rabbat for proofreading the article and providing valuable suggestions. 

\section*{Data availability}

The source code used for computing the metric for the simulation presented in Section~\ref{app:kl} is available in \emph{Figshare} at \url{https://dx.doi.org/10.6084/m9.figshare.14666163}.
The implementation used for computing the KL-divergence is also available in \emph{Figshare} at \url{https://dx.doi.org/10.6084/m9.figshare.14668473}.

The log-linear models' files are the same from \cite{ederaSchluterBromberg14} and were provided by its corresponding author with permission.
These data may be shared on request to the corresponding author of the cited work.

The \emph{Libra Toolkit} is available at \href{http://libra.cs.uoregon.edu/}{http://libra.cs.uoregon.edu/}.

Other data and source code related to this simulation will be shared on reasonable request to the corresponding author.

\bibliography{escllm}
%\printbibliography

\end{document}